\RequirePackage{fix-cm} 
\documentclass{article} 
\PassOptionsToPackage{table}{xcolor}
\usepackage{iclr2025_conference,times}
\iclrfinalcopy 

\usepackage{amsmath,amsfonts,bm}

\def\eqref#1{equation~\ref{#1}}

\def\1{\bm{1}}

\DeclareMathAlphabet{\mathsfit}{\encodingdefault}{\sfdefault}{m}{sl}
\SetMathAlphabet{\mathsfit}{bold}{\encodingdefault}{\sfdefault}{bx}{n}

\newcommand{\KL}{D_{\mathrm{KL}}}

\usepackage{hyperref}
\hypersetup{hidelinks}
\usepackage{url}

\usepackage{graphicx}
\usepackage{booktabs}
\usepackage{array}
\usepackage{multirow}
\usepackage{longtable}
\usepackage{placeins}
\usepackage{float}
\usepackage{capt-of}   
\usepackage{wrapfig}
\usepackage{makecell}
\usepackage{subcaption}
\usepackage{adjustbox}
\usepackage[table,xcdraw]{xcolor}
\definecolor{CoCurveShade}{HTML}{2F6DB5}
\definecolor{collapseband}{HTML}{FCE6CF}
\definecolor{collapsehatch}{HTML}{D79658}
\newcommand{\cellc}[2]{\cellcolor{CoCurveShade!\the\numexpr#1/2\relax}\strut #2}
\newcommand{\cellavg}[2]{\cellcolor{edgeorange!\the\numexpr#1/3\relax}\strut #2}
\newcommand{\collapsecell}[1]{\cellcolor{collapseband}\makebox[2.52em][c]{\tikz[baseline=-.55ex,inner sep=0pt,outer sep=0pt]\node[minimum width=2.52em,inner sep=0pt,outer sep=0pt,fill=collapseband,pattern={Lines[angle=45,distance=5.5pt,line width=.28pt]},pattern color=collapsehatch]{\strut #1};}}
\usepackage{amsmath}
\usepackage{amssymb}
\usepackage{amsthm}
\usepackage{pifont}
\usepackage{algorithm}
\usepackage{algorithmic}
\usepackage{tikz}
\usetikzlibrary{positioning,arrows.meta,fit,backgrounds,calc,shapes.geometric,shadows,patterns,patterns.meta}
\usepackage[capitalize,noabbrev]{cleveref}
\crefname{section}{Sec.}{Secs.}
\Crefname{section}{Sec.}{Secs.}
\crefname{subsection}{Sec.}{Secs.}
\Crefname{subsection}{Sec.}{Secs.}
\crefname{subsubsection}{Sec.}{Secs.}
\Crefname{subsubsection}{Sec.}{Secs.}
\newcommand{\appref}[1]{Appendix~\ref{#1}}
\newcommand{\Appref}[1]{Appendix~\ref{#1}}
\usepackage{enumitem}
\usepackage{needspace}
\usepackage{xspace}

\definecolor{oursblue}{RGB}{232, 240, 254}   
\newif\ifDRAFTMARKS \DRAFTMARKSfalse
\definecolor{pendingorange}{RGB}{198, 90, 0}
\definecolor{stalered}{RGB}{170, 0, 0}
\ifDRAFTMARKS

  \newcommand{\pendnote}[1]{\textcolor{pendingorange}{\footnotesize[\,TO UPDATE: #1\,]}}
\else

  \newcommand{\pendnote}[1]{}
\fi
\definecolor{headgray}{gray}{0.95}
\definecolor{accentblue}{RGB}{26, 86, 173}
\definecolor{edgeorange}{RGB}{201, 96, 20}   
\definecolor{attnblue}{RGB}{76, 114, 176}    
\definecolor{ffnorange}{RGB}{221, 132, 82}   

\newcommand{\pendingpanel}[2]{%
  \begingroup\setlength{\fboxsep}{7pt}%
  \fcolorbox{pendingorange}{headgray}{%
    \parbox[c][#1][c]{\dimexpr\linewidth-2\fboxsep-2\fboxrule\relax}{%
      \centering\small\color{pendingorange}\textbf{Figure slot: not yet measured}\par
      \smallskip\normalcolor #2}}%
  \endgroup}
\newcommand{\rowours}{}
\newcommand{\best}[1]{\textbf{#1}}
\newcommand{\snd}[1]{\underline{#1}}
\newcommand{\se}[1]{\mbox{{\fontsize{4.1}{4.3}\selectfont$\pm#1$}}}
\newcommand{\rankchip}[1]{\colorbox{CoCurveShade!#1}{\rule{0pt}{.72ex}\hspace{.70em}}}
\newcommand{\avgchip}[1]{\colorbox{edgeorange!#1}{\rule{0pt}{.72ex}\hspace{.70em}}}
\newcommand{\collapsechip}{\colorbox{collapseband}{\rule{0pt}{.72ex}\makebox[.70em][c]{\smash{\raisebox{-.22ex}{\textcolor{collapsehatch!70}{\tiny///}}}}}}
\newcommand{\dn}{$\downarrow$}
\newcommand{\up}{$\uparrow$}
\newcommand{\xmark}{\ding{55}}
\newcommand{\cmark}{\ding{51}}
\ifDRAFTMARKS
  
\else
  
\fi

\definecolor{linkcol}{RGB}{46,111,176}
\newcommand{\codeicon}{\raisebox{-0.22em}{\includegraphics[height=1.05em]{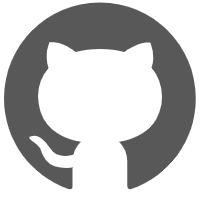}}}
\newcommand{\globeicon}{\raisebox{-0.22em}{\includegraphics[height=1.05em]{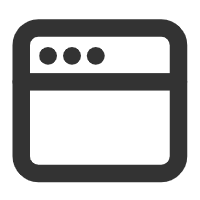}}}
\newcommand{\videoicon}{\raisebox{-0.22em}{\includegraphics[height=1.05em]{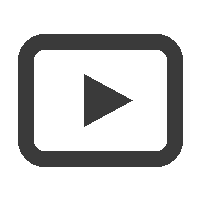}}}
\newcommand{\iconlink}[3]{\href{#1}{#2\,\,{\color{linkcol}\textbf{#3}}}}

\makeatletter
\def\ps@plain{%
  \let\@mkboth\@gobbletwo
  \let\@oddhead\@empty
  \let\@evenhead\@empty
  \def\@oddfoot{\hfil\thepage\hfil}
  \def\@evenfoot{\hfil\thepage\hfil}
}
\makeatother
\newcommand{\method}{CoCurve\xspace}
\newcommand{\methodfull}{CoCurve\xspace}

\newcommand{\Hmat}{\mathbf{H}}
\newcommand{\Uset}{\mathcal{U}}

\theoremstyle{plain}
\newtheorem{proposition}{Proposition}
\newtheorem{lemma}[proposition]{Lemma}
\newtheorem{corollary}[proposition]{Corollary}
\theoremstyle{remark}
\newtheorem{remark}[proposition]{Remark}

\title{CoCurve: Cross-Module Co-Pruning\\ Curvature for Structured LLM Pruning}

\author{%
Zhiren Gong$^{1,2}$ \quad
Zeng Zihao$^{1}$ \quad
Tiantong Wang$^{1}$ \quad
Yixin Wang$^{1}$ \quad
Honoka Anada$^{1,4}$ \\
\bf
Zijie Wang$^{1}$ \quad
Ming Xiao$^{5}$ \quad
Chau Yuen$^{3}$ \quad
Wei Yang Bryan Lim$^{1}$ \\
$^{1}$College of Computing and Data Science, Nanyang Technological University, Singapore\\
$^{2}$Interdisciplinary Graduate Programme, Nanyang Technological University, Singapore\\
$^{3}$School of Electrical and Electronic Engineering, Nanyang Technological University, Singapore\\
$^{4}$Department of Computer Science, The University of Tokyo, Japan\\
$^{5}$Information Science \& Engineering, KTH Royal Institute of Technology, Sweden\\
\texttt{zhiren001@e.ntu.edu.sg} \quad
\texttt{bryan.limwy@ntu.edu.sg}
}

\begin{document}

\raggedbottom

\maketitle
\thispagestyle{plain}

\vspace{-6mm}
\begin{center}\vspace{-2pt}
{\hypersetup{hidelinks}%
\iconlink{https://github.com/GongZhiren/CoCurve}{\codeicon}{Code}\hspace{20pt}%
\iconlink{https://gongzhiren.github.io/CoCurve-website/}{\globeicon}{Project Page}\hspace{20pt}%
\iconlink{https://gongzhiren.github.io/CoCurve-website/tutorial.html}{\videoicon}{Tutorial}}
\end{center}
\vspace{2mm}

\begin{abstract}
Resource-constrained deployment requires sustaining large language model (LLM) capabilities as models scale under fixed memory and computation budgets.
Structured pruning advances this deployment frontier with smaller dense checkpoints, yet deciding what to prune remains bottlenecked by interactions among joint removals.
We introduce Cross-Module Co-Pruning Curvature (CoCurve), which formulates structured pruning as set-dependent predictive risk over a unified inventory of attention heads and feed-forward groups.
A co-pruning graph built from single-unit forward ablations assigns individual risk to nodes and reinforcement or cancellation to interaction edges, conditioning each decision on the units already removed.
We evaluate $6$ LLMs ($3$B--$70$B) and $3$ vision--language models (VLMs) across $3$ perplexity corpora, $12$ language tasks, and $7$ multimodal benchmarks.
Across the five-model $20$--$40\%$ grid and the $70$B $10$--$50\%$ sweep, \method ranks first in $53/60$ corpus comparisons ($15/15$ at $70$B); matched-quality interpolation permits $2.2$--$6.6$ more pruning points in $9/10$ cases, while \method leads all $6$ VLM Avg$_7$ blocks.
After the same lightweight recovery, its retained structures remain strongest through $50\%$ pruning, where the $8$B checkpoint regains $10.8$ Avg$_{12}$ points; physical slicing delivers $1.58\times$ dense prefill throughput with $41\%$ lower peak memory.
Mechanism analysis across $10$ LLMs and $7$ VLMs finds organized within- and cross-module edge structure; matched low-saliency, high-coupling removals degrade $19/20$ capability groups by up to $23.7$ points.
\end{abstract}

\section{Introduction}
\label{sec:intro}

\begin{wrapfigure}{r}{.51\textwidth}
\vspace{-1.05\baselineskip}
\centering
\includegraphics[width=\linewidth]{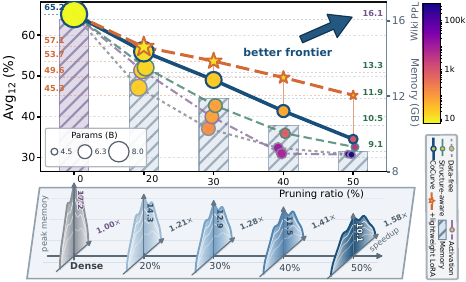}
\vspace{-10pt}
\caption{\textbf{CoCurve advances the quality--deployment frontier} on Llama-3.1-8B-Instruct.}
\label{fig:headline}
\vspace{-.85\baselineskip}
\end{wrapfigure}

Across architectures and model scales, large language models (LLMs) and multimodal extensions support increasingly broad capabilities in language modeling and downstream tasks, yet their memory, latency, and computation demands complicate deployment on heterogeneous hardware~\citep{dubey2024llama3,bai2025qwen25vl}.
The practical target is therefore the \emph{quality--deployment frontier}: how well these capabilities are retained within a resource budget.
Moving this frontier is central to democratizing LLM deployment.

\begin{figure}[t!]
\centering
\includegraphics[width=\textwidth]{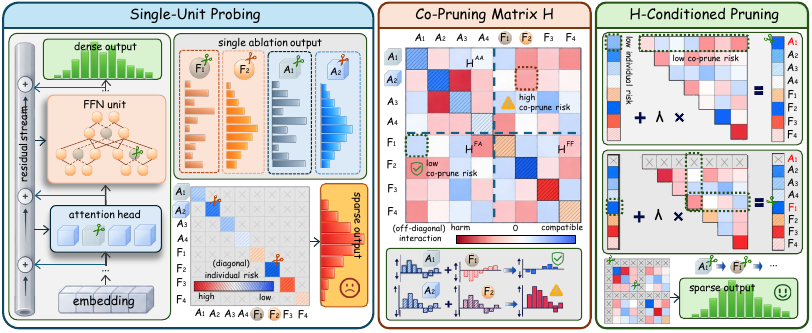}
\caption{\textbf{CoCurve: from single-unit responses to graph-conditioned pruning.}}
\label{fig:overview}
\vspace{-0.85\baselineskip}
\end{figure}

Model compression advances this frontier through complementary routes.
Quantization reduces memory and arithmetic but realizes speedups only through suitable low-precision kernels~\citep{frantar2023gptq}, while distillation incurs extra training to produce a smaller model~\citep{hinton2015distilling,muralidharan2024minitron}.
Pruning instead removes pretrained computation directly.
Unstructured pruning offers fine-grained removal but relies on sparse kernels or hardware to translate sparsity into speedups~\citep{frantar2023sparsegpt,sun2024wanda}, whereas structured pruning yields smaller dense checkpoints for standard hardware~\citep{ma2023llmpruner,an2024flap,ashkboos2024slicegpt}.
Existing structured pruners exploit weights, activations, gradients, reconstruction, or learned allocation~\citep{ling2024slimgpt,tang2025slimllm,li2025tyr,wang2026grasprune}, yet mostly assign unit-wise scores, leaving interactions among prospective removals unmodeled and constraining further progress along the frontier.

These interactions matter because structured pruning is inherently a joint-removal problem.
Attention heads and feed-forward network (FFN) groups write to a shared residual stream~\citep{vaswani2017attention,elhage2021framework}, through which their pruning effects can reinforce or partially cancel one another~\citep{mcgrath2023hydra,gong2026coablation}.
In a mechanism analysis across five LLMs, we find that mean normalized attention--FFN edge magnitude is $0.74$--$0.98\times$ the within-module mean (\cref{fig:anatomy}a).
We therefore ask: \emph{How can interactions among structured-unit pruning effects guide joint removal that better preserves model quality?}

We introduce Cross-Module Co-Pruning Curvature (CoCurve), casting structured pruning as set-dependent predictive risk over a unified inventory of attention heads and FFN groups (\cref{fig:overview}).
A local expansion of teacher Kullback--Leibler (KL) risk separates \emph{individual} risk on the diagonal from signed \emph{interaction} edges off the diagonal, covering pairs within and across modules.
A Fisher--Gram estimator builds this reusable graph from one forward ablation per unit, without gradients or weight updates.
At each step, a conditional pruner updates each candidate's individual risk using its edges to the current pruned set, and held-out predictive risk selects among the resulting masks.
CoCurve thus revises risk as the pruned set grows instead of applying a fixed unit ranking or module allocation.

We evaluate \method on $6$ LLMs ($3$B--$70$B) and $3$ vision--language models (VLMs) over $3$ perplexity corpora, $12$ language tasks, and $7$ multimodal benchmarks.
Across the LLMs, \method leads $53/60$ corpus comparisons on the five-model $20$--$40\%$ grid and the $70$B $10$--$50\%$ sweep ($15/15$ at $70$B), permits $2.2$--$6.6$ points more pruning in $9/10$ matched-quality interpolations at $30$--$40\%$, and leads all $8$ Avg$_{12}$ blocks on the $8$B/$24$B models.
Across the VLMs, it leads every Avg$_7$ block and gains $1.1$--$10.7$ points at $30\%$.
On the $8$B LLM, its structure remains best under identical lightweight low-rank adaptation (LoRA) through $50\%$ pruning, and under weight quantization.
Physical slicing reaches $1.58\times$ dense prefill throughput with $41\%$ lower peak memory, moving the quality--deployment frontier in \cref{fig:headline}.
Across $10$ LLMs and $7$ VLMs, mechanism analysis finds organized within- and cross-module edges that reverse an actual pruning decision; matched low-saliency, high-coupling removals degrade $19/20$ capability groups by up to $23.7$ points.
Thus modeling interactions among joint removals provides a direct route to advancing the quality--deployment frontier.

\section{Method}
\label{sec:method}

Attention and FFN units meet in the shared residual stream, motivating one structured-pruning inventory $\Uset=\{u_1,\ldots,u_M\}$.
With $s_u{=}1$ denoting pruning unit $u$, the goal is to meet a target parameter reduction while preserving quality and lowering deployment cost.

\method minimizes \emph{set-dependent predictive risk}.
It derives individual and pairwise terms (\cref{sec:method-risk}), estimates their graph from single-unit ablations (\cref{sec:method-estimate}), and conditionally prunes with held-out-risk calibration (\cref{sec:method-solver}).
Details and theory appear in Appendices~\ref{app:method} and~\ref{app:theory}.

\subsection{Predictive risk defines the co-pruning graph}
\label{sec:method-risk}

Building on Optimal Brain Damage (OBD), Optimal Brain Surgeon (OBS)~\citep{lecun1989obd,hassibi1993obs}, and structured objectives for interacting removals~\citep{yu2022cbs,nonnenmacher2022sosp,kwon2022fastpruning}, CoCurve evaluates all units through one predictive risk rather than separate module scores.
For a calibration sequence $x$ and token position $t$, let $p_0(\cdot\mid x_{\le t})$ and $p_s(\cdot\mid x_{\le t})$ denote the dense and gated next-token distributions, with gates relaxed continuously near the dense point $s=\mathbf0$.
Forward teacher KL divergence uses the dense next-token distribution as unlabeled supervision, following the distribution-matching principle of knowledge distillation~\citep{hinton2015distilling}.
Its token-level predictive risk is
\begin{equation}
\mathcal R(s)=\mathbb E_{x,t}\!\left[\KL\!\left(p_0(\cdot\mid x_{\le t})\,\|\,p_s(\cdot\mid x_{\le t})\right)\right].
\label{eq:risk}
\end{equation}
At the dense model, the distributions coincide, so $\mathcal R(\mathbf0)=0$, $\nabla_s\mathcal R(\mathbf0)=0$, and curvature is the first informative local term without assuming training-loss stationarity (\appref{app:th-geom}).
Since binary removal gates satisfy $s_u^2=s_u$, the second-order Taylor approximation separates into
\begin{equation}
\mathcal R(s)\approx \tfrac12 s^\top\Hmat^\star s
=\tfrac12\sum_u\Hmat^\star_{uu}s_u
+\sum_{u<v}\Hmat^\star_{uv}s_us_v,
\qquad \Hmat^\star=\nabla_s^2\mathcal R(\mathbf0).
\label{eq:decomp}
\end{equation}
This decomposition naturally defines the \emph{co-pruning graph}: diagonals encode individual pruning risk (nodes), while signed off-diagonals link removal pairs (edges).
Positive edges capture reinforcement between Fisher-weighted logit responses and raise joint risk above isolated risks; negative edges capture partial cancellation and lower it (\cref{prop:offdiag}).
Edges may link units within or across modules and layers; they measure alignment between pruning-induced predictive changes rather than only direct architectural adjacency.
Thus the graph models set-dependent pruning risk beyond any static ranking.

\subsection{Single-unit ablations construct the full graph}
\label{sec:method-estimate}

The exact mask Hessian factors through unit-wise logit derivatives.
Let $z_0$ and $z_s$ denote the dense and gated next-token logits, write $\Delta z_s=z_0-z_s$, and define $G_u=\partial\Delta z_s/\partial s_u|_{\mathbf0}$.
At the dense model, the exact mask Hessian is $\Hmat^\star_{uv}=\mathbb E_{x,t}[G_u^\top F_{x,t}G_v]$, where $F_{x,t}=\operatorname{diag}(p_0)-p_0p_0^\top$ is the categorical Fisher matrix in logit space~\citep{amari1998natural}.

CoCurve replaces each derivative column with the finite next-token-logit response $\delta z_u=z_0-z_{-u}$, where $z_{-u}$ denotes the logits with only $u$ pruned.
Because every response occupies the same next-token-logit space, units from different layers and modules are directly comparable.
Let $\bar D_u$ stack these responses after teacher-probability weighting and mean centering over $P$ sampled positions.
Their Gram matrix estimates every node and edge:
\begin{equation}
\Hmat_{uv}:=\tfrac1P\bar D_u^\top\bar D_v.
\label{eq:gram}
\end{equation}
Thus $\Hmat$ is positive semidefinite, and $M$ single-unit forward ablations recover all $O(M^2)$ edges of the finite-ablation Gram surrogate without pairwise interventions, gradients, or weight updates (\cref{prop:gram}).
On the $8$B inventory, this is $768$ rather than $294{,}528$ model interventions.

\begin{figure}[t!]
\centering
\captionsetup{skip=1pt}
\includegraphics[width=\textwidth]{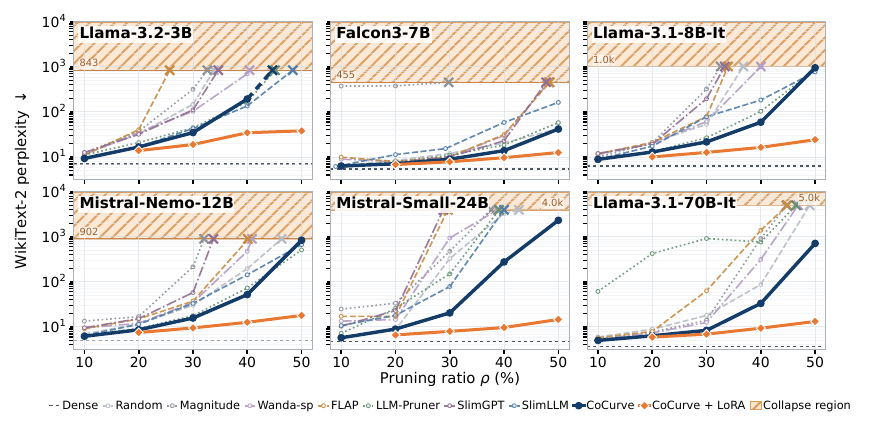}
\caption{\textbf{WikiText-2 perplexity $\downarrow$ across pruning ratios.}
Full sweep: \appref{app:ratio}.}
\label{fig:pareto}
\vspace{-.65\baselineskip}
\end{figure}

\subsection{Accumulated edges condition each pruning decision}
\label{sec:method-solver}

Edges matter operationally when they change which unit is pruned next.
To bridge local curvature and finite removal, CoCurve traces a positive-semidefinite path $\lambda\in[0,1]$ from isolated risk to the full co-pruning graph, and held-out predictive risk selects its mask:
\begin{equation}
\widehat{\mathcal R}_\lambda(s)
=\tfrac12\sum_u\Hmat_{uu}s_u+\lambda\sum_{u<v}\Hmat_{uv}s_us_v
=\tfrac12s^\top\Hmat_\lambda s,
\quad
\Hmat_\lambda=(1-\lambda)\operatorname{diag}(\Hmat)+\lambda\Hmat.
\label{eq:calibrated-risk}
\end{equation}
At $\lambda=0$ this recovers isolated saliency and at $\lambda=1$ the full Gram surrogate; every intermediate matrix remains positive semidefinite.
Here $\lambda$ indexes candidate masks rather than interaction magnitude: candidate $u$ receives the accumulated signed correction $\lambda\sum_{v\in S}\Hmat_{uv}$.

Let $c_u>0$ be unit $u$'s removable parameters, $C=\sum_u c_u$, $c_{\max}=\max_u c_u$, and $\mathcal G\subseteq\{0,1\}^M$ the masks satisfying fixed guards.
For target fraction $\rho$, we solve
\begin{equation}
\min_{s\in\mathcal G}\widehat{\mathcal R}_\lambda(s)
\quad\text{subject to}\quad
\rho C\le \sum_uc_us_u<\rho C+c_{\max}.
\label{eq:objective}
\end{equation}
Attention and FFN units compete under this single objective, with no module-wise quota searched or prescribed.
Exact minimization would search exponentially many heterogeneous-cost subsets, so \method uses a cost-normalized conditional greedy rule.
If $S$ is the current removal set, the incremental surrogate risk of adding an eligible unit $u$ is
\begin{equation}
\Delta\widehat{\mathcal R}_\lambda(u\mid S)
=\tfrac12\Hmat_{uu}+\lambda\sum_{v\in S}\Hmat_{uv}.
\label{eq:marginal}
\end{equation}
The first term is fixed individual saliency, while the second accumulates signed edges to the current set.
We minimize $\Delta\widehat{\mathcal R}_\lambda(u\mid S)/c_u$ and stop at the target cost.
Because the incident-edge sum changes with $S$, the ranking can change after every pruning step.

\noindent\textbf{Closing the loop with held-out risk.}
Varying $\lambda$ produces a finite family $\mathcal F$ of distinct removal sets (\cref{prop:finite-path}).
On a disjoint split $\mathcal C'$, we evaluate each mask by its teacher-KL risk $\mathcal R(S;\mathcal C')$ and select $\widehat S=\arg\min_{S\in\mathcal F}\mathcal R(S;\mathcal C')$.
The graph proposes candidate masks, while measured risk chooses the actual pruned model without downstream labels or weight updates.

\begin{table}[t!]
\centering\fontsize{6.05}{6.55}\selectfont
\caption{\textbf{Pruning-only quality across two model scales.}
Family means and Avg$_{12}$ are percentages; full task breakdowns appear in \appref{app:full-results}.}
\label{tab:flagship}
\adjustbox{width=\textwidth}{
\setlength{\tabcolsep}{1.2pt}\renewcommand{\arraystretch}{0.91}
\begin{tabular}{@{}>{\centering\arraybackslash}m{2.2em}l>{\centering\arraybackslash}m{2.65em}!{\color{black!35}\vrule width .5pt}>{\centering\arraybackslash}m{2.65em}>{\centering\arraybackslash}m{2.65em}>{\centering\arraybackslash}m{2.65em}>{\centering\arraybackslash}m{2.65em}>{\centering\arraybackslash}m{4.35em}!{\color{black!35}\vrule width .5pt}>{\centering\arraybackslash}m{2.65em}!{\color{black!35}\vrule width .5pt}>{\centering\arraybackslash}m{2.65em}>{\centering\arraybackslash}m{2.65em}>{\centering\arraybackslash}m{2.65em}>{\centering\arraybackslash}m{2.65em}>{\centering\arraybackslash}m{4.35em}@{}}\toprule
 & & \multicolumn{6}{c}{\textbf{Llama-3.1-8B-Instruct}} & \multicolumn{6}{c}{\textbf{Mistral-Small-24B}} \\
\cmidrule(lr){3-8}\cmidrule(lr){9-14}
\textbf{$\rho$} & \textbf{Method} & \textbf{Wiki}$\downarrow$ & \textbf{Cmn$_4$} & \textbf{Sci$_3$} & \textbf{Read$_4$} & \textbf{Know$_1$} & \textbf{Avg$_{12}$}$\uparrow$ & \textbf{Wiki}$\downarrow$ & \textbf{Cmn$_4$} & \textbf{Sci$_3$} & \textbf{Read$_4$} & \textbf{Know$_1$} & \textbf{Avg$_{12}$}$\uparrow$ \\
\midrule
\rowcolor{headgray} -- & \textbf{Dense} & 6.4 & 70.9 & 60.1 & 67.6 & 47.8 & 65.2\se{0.34} & 4.7 & 73.5 & 63.4 & 65.9 & 52.3 & 66.7\se{0.33} \\
\midrule
\multirow{8}{*}{20\%} & Random & \cellcolor{CoCurveShade!7}21.2 & \cellcolor{CoCurveShade!7}53.8 & \cellcolor{CoCurveShade!7}37.0 & \cellcolor{CoCurveShade!11}49.1 & \cellcolor{CoCurveShade!7}31.8 & \cellcolor{edgeorange!7}46.2\se{0.34} & \cellcolor{CoCurveShade!32}\textbf{8.7} & \cellcolor{CoCurveShade!21}64.6 & \cellcolor{CoCurveShade!18}53.0 & \cellcolor{CoCurveShade!18}58.0 & \cellcolor{CoCurveShade!21}43.2 & \cellcolor{edgeorange!21}57.7\se{0.34} \\
 & Magnitude & \cellcolor{CoCurveShade!14}20.5 & \cellcolor{CoCurveShade!11}55.5 & \cellcolor{CoCurveShade!11}39.5 & \cellcolor{CoCurveShade!14}51.0 & \cellcolor{CoCurveShade!11}32.8 & \cellcolor{edgeorange!11}48.1\se{0.34} & \cellcolor{CoCurveShade!7}33.6 & \cellcolor{CoCurveShade!14}63.5 & \cellcolor{CoCurveShade!21}53.1 & \cellcolor{CoCurveShade!7}56.6 & \cellcolor{CoCurveShade!18}42.4 & \cellcolor{edgeorange!14}56.8\se{0.35} \\
 & Wanda-sp & \cellcolor{CoCurveShade!18}20.0 & \cellcolor{CoCurveShade!21}59.3 & \cellcolor{CoCurveShade!21}45.9 & \cellcolor{CoCurveShade!25}52.6 & \cellcolor{CoCurveShade!21}35.6 & \cellcolor{edgeorange!21}51.7\se{0.35} & \cellcolor{CoCurveShade!25}14.6 & \cellcolor{CoCurveShade!11}62.7 & \cellcolor{CoCurveShade!11}51.2 & \cellcolor{CoCurveShade!21}58.6 & \cellcolor{CoCurveShade!11}41.0 & \cellcolor{edgeorange!11}56.6\se{0.35} \\
 & FLAP & \cellcolor{CoCurveShade!11}20.9 & \cellcolor{CoCurveShade!18}58.4 & \cellcolor{CoCurveShade!25}46.1 & \cellcolor{CoCurveShade!21}51.4 & \cellcolor{CoCurveShade!18}35.3 & \cellcolor{edgeorange!18}51.1\se{0.35} & \cellcolor{CoCurveShade!21}17.0 & \cellcolor{CoCurveShade!7}61.6 & \cellcolor{CoCurveShade!7}50.1 & \cellcolor{CoCurveShade!11}57.1 & \cellcolor{CoCurveShade!7}40.6 & \cellcolor{edgeorange!7}55.5\se{0.35} \\
 & LLM-Pruner & \cellcolor{CoCurveShade!28}\underline{13.2} & \cellcolor{CoCurveShade!28}\underline{62.2} & \cellcolor{CoCurveShade!14}44.8 & \cellcolor{CoCurveShade!28}\underline{58.5} & \cellcolor{CoCurveShade!25}36.8 & \cellcolor{edgeorange!28}\underline{54.5}\se{0.34} & \cellcolor{CoCurveShade!11}25.8 & \cellcolor{CoCurveShade!25}65.5 & \cellcolor{CoCurveShade!28}\underline{55.2} & \cellcolor{CoCurveShade!28}\underline{59.8} & \cellcolor{CoCurveShade!32}\textbf{43.7} & \cellcolor{edgeorange!28}\underline{59.2}\se{0.34} \\
 & SlimGPT & \cellcolor{CoCurveShade!21}19.2 & \cellcolor{CoCurveShade!14}58.3 & \cellcolor{CoCurveShade!18}44.9 & \cellcolor{CoCurveShade!7}48.3 & \cellcolor{CoCurveShade!14}35.1 & \cellcolor{edgeorange!14}49.7\se{0.35} & \cellcolor{CoCurveShade!14}23.1 & \cellcolor{CoCurveShade!18}64.3 & \cellcolor{CoCurveShade!14}52.0 & \cellcolor{CoCurveShade!14}57.5 & \cellcolor{CoCurveShade!14}41.4 & \cellcolor{edgeorange!18}57.0\se{0.35} \\
 & SlimLLM & \cellcolor{CoCurveShade!25}17.4 & \cellcolor{CoCurveShade!25}59.8 & \cellcolor{CoCurveShade!32}\textbf{48.2} & \cellcolor{CoCurveShade!18}51.1 & \cellcolor{CoCurveShade!28}\underline{37.2} & \cellcolor{edgeorange!25}52.1\se{0.35} & \cellcolor{CoCurveShade!18}18.1 & \cellcolor{CoCurveShade!28}\underline{67.0} & \cellcolor{CoCurveShade!25}54.6 & \cellcolor{CoCurveShade!25}58.8 & \cellcolor{CoCurveShade!28}\underline{43.5} & \cellcolor{edgeorange!25}59.2\se{0.34} \\
 & \textbf{CoCurve} & \cellcolor{CoCurveShade!32}\textbf{12.9} & \cellcolor{CoCurveShade!32}\textbf{63.7} & \cellcolor{CoCurveShade!28}\underline{48.0} & \cellcolor{CoCurveShade!32}\textbf{58.6} & \cellcolor{CoCurveShade!32}\textbf{38.5} & \cellcolor{edgeorange!32}\textbf{56.0}\se{0.35} & \cellcolor{CoCurveShade!28}\underline{9.0} & \cellcolor{CoCurveShade!32}\textbf{67.8} & \cellcolor{CoCurveShade!32}\textbf{56.3} & \cellcolor{CoCurveShade!32}\textbf{60.9} & \cellcolor{CoCurveShade!25}43.5 & \cellcolor{edgeorange!32}\textbf{60.6}\se{0.34} \\
\midrule
\multirow{8}{*}{30\%} & Random & \cellcolor{CoCurveShade!25}52.9 & \cellcolor{CoCurveShade!14}46.0 & \cellcolor{CoCurveShade!11}30.6 & \cellcolor{CoCurveShade!21}40.9 & \cellcolor{CoCurveShade!14}28.6 & \cellcolor{edgeorange!14}39.0\se{0.33} & \cellcolor{CoCurveShade!21}328.2 & \cellcolor{CoCurveShade!25}56.7 & \cellcolor{CoCurveShade!25}43.4 & \cellcolor{CoCurveShade!18}45.3 & \cellcolor{CoCurveShade!24}35.5 & \cellcolor{edgeorange!21}47.8\se{0.35} \\
 & Magnitude & \cellcolor{CoCurveShade!7}324.5 & \cellcolor{CoCurveShade!7}41.2 & \cellcolor{CoCurveShade!7}27.4 & \cellcolor{CoCurveShade!7}37.0 & \cellcolor{CoCurveShade!7}27.3 & \cellcolor{edgeorange!7}35.2\se{0.33} & \cellcolor{CoCurveShade!18}477.5 & \cellcolor{CoCurveShade!14}50.1 & \cellcolor{CoCurveShade!14}39.0 & \cellcolor{CoCurveShade!21}46.8 & \cellcolor{CoCurveShade!15}31.7 & \cellcolor{edgeorange!18}44.7\se{0.34} \\
 & Wanda-sp & \cellcolor{CoCurveShade!21}61.2 & \cellcolor{CoCurveShade!18}47.8 & \cellcolor{CoCurveShade!21}33.9 & \cellcolor{CoCurveShade!18}40.2 & \cellcolor{CoCurveShade!21}29.6 & \cellcolor{edgeorange!21}40.3\se{0.34} & \cellcolor{CoCurveShade!14}945.4 & \cellcolor{CoCurveShade!18}50.6 & \cellcolor{CoCurveShade!18}41.7 & \cellcolor{CoCurveShade!14}42.9 & \cellcolor{CoCurveShade!20}33.7 & \cellcolor{edgeorange!14}44.4\se{0.35} \\
 & FLAP & \cellcolor{CoCurveShade!14}79.4 & \cellcolor{CoCurveShade!21}48.5 & \cellcolor{CoCurveShade!14}31.8 & \cellcolor{CoCurveShade!14}39.6 & \cellcolor{CoCurveShade!18}29.0 & \cellcolor{edgeorange!18}39.7\se{0.33} & \collapsecell{4.8e3} & \cellcolor{CoCurveShade!7}37.1 & \cellcolor{CoCurveShade!7}24.3 & \cellcolor{CoCurveShade!7}28.0 & \cellcolor{CoCurveShade!7}26.5 & \cellcolor{edgeorange!7}30.0\se{0.32} \\
 & LLM-Pruner & \cellcolor{CoCurveShade!28}\underline{26.6} & \cellcolor{CoCurveShade!28}\underline{55.2} & \cellcolor{CoCurveShade!28}\underline{38.2} & \cellcolor{CoCurveShade!32}\textbf{52.1} & \cellcolor{CoCurveShade!28}\underline{31.9} & \cellcolor{edgeorange!28}\underline{48.0}\se{0.34} & \cellcolor{CoCurveShade!25}150.6 & \cellcolor{CoCurveShade!21}54.6 & \cellcolor{CoCurveShade!21}41.8 & \cellcolor{CoCurveShade!28}\underline{51.5} & \cellcolor{CoCurveShade!20}33.7 & \cellcolor{edgeorange!25}48.6\se{0.34} \\
 & SlimGPT & \cellcolor{CoCurveShade!11}196.8 & \cellcolor{CoCurveShade!11}45.7 & \cellcolor{CoCurveShade!18}31.9 & \cellcolor{CoCurveShade!11}38.6 & \cellcolor{CoCurveShade!11}27.9 & \cellcolor{edgeorange!11}38.4\se{0.33} & \collapsecell{7.9e3} & \cellcolor{CoCurveShade!11}49.4 & \cellcolor{CoCurveShade!11}37.4 & \cellcolor{CoCurveShade!11}42.5 & \cellcolor{CoCurveShade!11}31.4 & \cellcolor{edgeorange!11}42.6\se{0.34} \\
 & SlimLLM & \cellcolor{CoCurveShade!18}78.2 & \cellcolor{CoCurveShade!25}49.1 & \cellcolor{CoCurveShade!25}36.7 & \cellcolor{CoCurveShade!25}41.3 & \cellcolor{CoCurveShade!25}30.4 & \cellcolor{edgeorange!25}41.8\se{0.34} & \cellcolor{CoCurveShade!28}\underline{78.8} & \cellcolor{CoCurveShade!28}\underline{58.5} & \cellcolor{CoCurveShade!32}\textbf{49.8} & \cellcolor{CoCurveShade!25}50.2 & \cellcolor{CoCurveShade!28}\underline{36.5} & \cellcolor{edgeorange!28}\underline{51.7}\se{0.35} \\
 & \textbf{CoCurve} & \cellcolor{CoCurveShade!32}\textbf{21.6} & \cellcolor{CoCurveShade!32}\textbf{55.7} & \cellcolor{CoCurveShade!32}\textbf{41.7} & \cellcolor{CoCurveShade!28}\underline{51.8} & \cellcolor{CoCurveShade!32}\textbf{33.4} & \cellcolor{edgeorange!32}\textbf{49.0}\se{0.35} & \cellcolor{CoCurveShade!32}\textbf{20.5} & \cellcolor{CoCurveShade!32}\textbf{59.3} & \cellcolor{CoCurveShade!28}\underline{48.2} & \cellcolor{CoCurveShade!32}\textbf{54.1} & \cellcolor{CoCurveShade!32}\textbf{36.5} & \cellcolor{edgeorange!32}\textbf{52.9}\se{0.35} \\
\midrule
\multirow{8}{*}{40\%} & Random & \collapsecell{4.2e3} & \cellcolor{CoCurveShade!21}39.4 & \cellcolor{CoCurveShade!14}26.2 & \cellcolor{CoCurveShade!21}34.4 & \cellcolor{CoCurveShade!18}26.7 & \cellcolor{edgeorange!21}33.4\se{0.32} & \cellcolor{CoCurveShade!28}\underline{3.3e3} & \cellcolor{CoCurveShade!25}42.7 & \cellcolor{CoCurveShade!21}29.8 & \cellcolor{CoCurveShade!21}32.3 & \cellcolor{CoCurveShade!21}27.9 & \cellcolor{edgeorange!25}34.8\se{0.33} \\
 & Magnitude & \collapsecell{2.7e4} & \cellcolor{CoCurveShade!7}36.6 & \cellcolor{CoCurveShade!18}26.5 & \cellcolor{CoCurveShade!18}31.0 & \cellcolor{CoCurveShade!11}26.2 & \cellcolor{edgeorange!14}31.3\se{0.32} & \collapsecell{7.8e3} & \cellcolor{CoCurveShade!18}39.4 & \cellcolor{CoCurveShade!25}29.9 & \cellcolor{CoCurveShade!25}33.4 & \cellcolor{CoCurveShade!18}27.2 & \cellcolor{edgeorange!18}34.0\se{0.33} \\
 & Wanda-sp & \cellcolor{CoCurveShade!21}1.0e3 & \cellcolor{CoCurveShade!18}38.8 & \cellcolor{CoCurveShade!11}26.1 & \cellcolor{CoCurveShade!14}29.3 & \cellcolor{CoCurveShade!14}26.6 & \cellcolor{edgeorange!18}31.4\se{0.32} & \collapsecell{4.8e3} & \cellcolor{CoCurveShade!14}38.0 & \cellcolor{CoCurveShade!11}25.8 & \cellcolor{CoCurveShade!14}29.7 & \cellcolor{CoCurveShade!11}26.8 & \cellcolor{edgeorange!11}31.3\se{0.32} \\
 & FLAP & \collapsecell{5.0e4} & \cellcolor{CoCurveShade!14}37.3 & \cellcolor{CoCurveShade!7}24.7 & \cellcolor{CoCurveShade!11}29.2 & \cellcolor{CoCurveShade!7}25.6 & \cellcolor{edgeorange!7}30.5\se{0.32} & \collapsecell{5.3e3} & \cellcolor{CoCurveShade!11}38.0 & \cellcolor{CoCurveShade!14}26.3 & \cellcolor{CoCurveShade!11}29.4 & \cellcolor{CoCurveShade!14}26.8 & \cellcolor{edgeorange!14}31.3\se{0.32} \\
 & LLM-Pruner & \cellcolor{CoCurveShade!28}\underline{103.3} & \cellcolor{CoCurveShade!32}\textbf{48.4} & \cellcolor{CoCurveShade!28}\underline{31.6} & \cellcolor{CoCurveShade!28}\underline{43.4} & \cellcolor{CoCurveShade!32}\textbf{29.7} & \cellcolor{edgeorange!28}\underline{41.0}\se{0.33} & \collapsecell{5.5e3} & \cellcolor{CoCurveShade!21}41.4 & \cellcolor{CoCurveShade!18}29.7 & \cellcolor{CoCurveShade!18}31.2 & \cellcolor{CoCurveShade!25}28.5 & \cellcolor{edgeorange!21}34.0\se{0.33} \\
 & SlimGPT & \collapsecell{2.6e4} & \cellcolor{CoCurveShade!11}37.1 & \cellcolor{CoCurveShade!21}27.0 & \cellcolor{CoCurveShade!7}27.8 & \cellcolor{CoCurveShade!21}26.9 & \cellcolor{edgeorange!11}30.6\se{0.32} & \collapsecell{9.1e3} & \cellcolor{CoCurveShade!7}37.3 & \cellcolor{CoCurveShade!7}25.1 & \cellcolor{CoCurveShade!7}28.6 & \cellcolor{CoCurveShade!7}26.8 & \cellcolor{edgeorange!7}30.5\se{0.32} \\
 & SlimLLM & \cellcolor{CoCurveShade!25}183.5 & \cellcolor{CoCurveShade!25}43.3 & \cellcolor{CoCurveShade!25}30.2 & \cellcolor{CoCurveShade!25}35.5 & \cellcolor{CoCurveShade!25}27.6 & \cellcolor{edgeorange!25}36.1\se{0.33} & \collapsecell{4.0e3} & \cellcolor{CoCurveShade!32}\textbf{45.1} & \cellcolor{CoCurveShade!32}\textbf{37.6} & \cellcolor{CoCurveShade!28}\underline{34.1} & \cellcolor{CoCurveShade!28}\underline{30.0} & \cellcolor{edgeorange!28}\underline{38.3}\se{0.34} \\
 & \textbf{CoCurve} & \cellcolor{CoCurveShade!32}\textbf{59.1} & \cellcolor{CoCurveShade!28}\underline{47.1} & \cellcolor{CoCurveShade!32}\textbf{33.3} & \cellcolor{CoCurveShade!32}\textbf{44.8} & \cellcolor{CoCurveShade!28}\underline{29.3} & \cellcolor{edgeorange!32}\textbf{41.4}\se{0.34} & \cellcolor{CoCurveShade!32}\textbf{277.1} & \cellcolor{CoCurveShade!28}\underline{44.3} & \cellcolor{CoCurveShade!28}\underline{33.1} & \cellcolor{CoCurveShade!32}\textbf{41.8} & \cellcolor{CoCurveShade!32}\textbf{30.0} & \cellcolor{edgeorange!32}\textbf{39.5}\se{0.34} \\
\midrule
\multirow{8}{*}{50\%} & Random & \collapsecell{1.3e5} & \cellcolor{CoCurveShade!14}36.1 & \cellcolor{CoCurveShade!14}24.6 & \cellcolor{CoCurveShade!7}29.7 & \cellcolor{CoCurveShade!25}26.6 & \cellcolor{edgeorange!11}30.3\se{0.32} & \collapsecell{6.5e3} & \cellcolor{CoCurveShade!28}\underline{37.6} & \cellcolor{CoCurveShade!28}\underline{27.1} & \cellcolor{CoCurveShade!18}28.7 & \cellcolor{CoCurveShade!28}\underline{27.2} & \cellcolor{edgeorange!25}31.1\se{0.32} \\
 & Magnitude & \collapsecell{1.6e5} & \cellcolor{CoCurveShade!7}35.6 & \cellcolor{CoCurveShade!25}26.6 & \cellcolor{CoCurveShade!21}32.4 & \cellcolor{CoCurveShade!7}24.7 & \cellcolor{edgeorange!18}31.4\se{0.32} & \collapsecell{9.4e3} & \cellcolor{CoCurveShade!21}36.8 & \cellcolor{CoCurveShade!21}26.5 & \cellcolor{CoCurveShade!11}28.1 & \cellcolor{CoCurveShade!11}26.0 & \cellcolor{edgeorange!18}30.4\se{0.32} \\
 & Wanda-sp & \collapsecell{3.7e5} & \cellcolor{CoCurveShade!18}36.6 & \cellcolor{CoCurveShade!7}22.5 & \cellcolor{CoCurveShade!11}30.4 & \cellcolor{CoCurveShade!14}25.5 & \cellcolor{edgeorange!7}30.1\se{0.31} & \collapsecell{6.2e3} & \cellcolor{CoCurveShade!7}35.7 & \cellcolor{CoCurveShade!11}25.2 & \cellcolor{CoCurveShade!7}27.6 & \cellcolor{CoCurveShade!25}27.0 & \cellcolor{edgeorange!7}29.7\se{0.32} \\
 & FLAP & \collapsecell{6.3e5} & \cellcolor{CoCurveShade!11}36.1 & \cellcolor{CoCurveShade!11}24.0 & \cellcolor{CoCurveShade!25}33.3 & \cellcolor{CoCurveShade!18}25.8 & \cellcolor{edgeorange!14}31.3\se{0.31} & \collapsecell{6.5e3} & \cellcolor{CoCurveShade!11}36.0 & \cellcolor{CoCurveShade!18}26.3 & \cellcolor{CoCurveShade!14}28.5 & \cellcolor{CoCurveShade!18}26.9 & \cellcolor{edgeorange!14}30.3\se{0.32} \\
 & LLM-Pruner & \cellcolor{CoCurveShade!28}\underline{865.1} & \cellcolor{CoCurveShade!32}\textbf{42.0} & \cellcolor{CoCurveShade!32}\textbf{27.7} & \cellcolor{CoCurveShade!18}31.3 & \cellcolor{CoCurveShade!32}\textbf{27.9} & \cellcolor{edgeorange!28}\underline{33.7}\se{0.32} & \collapsecell{9.3e3} & \cellcolor{CoCurveShade!25}37.5 & \cellcolor{CoCurveShade!25}26.7 & \cellcolor{CoCurveShade!21}29.2 & \cellcolor{CoCurveShade!21}27.0 & \cellcolor{edgeorange!28}\underline{31.2}\se{0.32} \\
 & SlimGPT & \collapsecell{1.1e5} & \cellcolor{CoCurveShade!21}36.7 & \cellcolor{CoCurveShade!18}25.7 & \cellcolor{CoCurveShade!28}\underline{33.5} & \cellcolor{CoCurveShade!11}25.3 & \cellcolor{edgeorange!21}31.9\se{0.32} & \collapsecell{\underline{6.1e3}} & \cellcolor{CoCurveShade!18}36.1 & \cellcolor{CoCurveShade!7}24.6 & \cellcolor{CoCurveShade!28}\underline{29.8} & \cellcolor{CoCurveShade!7}25.9 & \cellcolor{edgeorange!11}30.3\se{0.32} \\
 & SlimLLM & \cellcolor{CoCurveShade!32}\textbf{773.3} & \cellcolor{CoCurveShade!25}39.0 & \cellcolor{CoCurveShade!21}26.0 & \cellcolor{CoCurveShade!14}31.0 & \cellcolor{CoCurveShade!21}26.5 & \cellcolor{edgeorange!25}32.0\se{0.32} & \collapsecell{9.7e3} & \cellcolor{CoCurveShade!14}36.1 & \cellcolor{CoCurveShade!14}26.1 & \cellcolor{CoCurveShade!25}29.8 & \cellcolor{CoCurveShade!14}26.4 & \cellcolor{edgeorange!21}30.7\se{0.32} \\
 & \textbf{CoCurve} & \cellcolor{CoCurveShade!25}956.4 & \cellcolor{CoCurveShade!28}\underline{40.1} & \cellcolor{CoCurveShade!28}\underline{27.3} & \cellcolor{CoCurveShade!32}\textbf{36.1} & \cellcolor{CoCurveShade!28}\underline{27.5} & \cellcolor{edgeorange!32}\textbf{34.5}\se{0.32} & \cellcolor{CoCurveShade!32}\textbf{2.3e3} & \cellcolor{CoCurveShade!32}\textbf{38.4} & \cellcolor{CoCurveShade!32}\textbf{27.4} & \cellcolor{CoCurveShade!32}\textbf{33.5} & \cellcolor{CoCurveShade!32}\textbf{28.2} & \cellcolor{edgeorange!32}\textbf{33.2}\se{0.32} \\
\bottomrule\end{tabular}
}
\vspace{-1pt}
\parbox{\textwidth}{\fontsize{5.5}{6.1}\selectfont\emph{Note.}
Avg$_{12}$ includes stratified SE; bold/underline: first/second;
\rankchip{32}\,/\rankchip{20}\,/\rankchip{7} show within-block rank,
\avgchip{24} Avg$_{12}$ rank, and \collapsechip collapse.}
\end{table}

\section{Experiments}
\label{sec:experiments}

\subsection{Evaluation setup}
\label{sec:exp-setup}

\textbf{Models.}
We evaluate $6$ \emph{LLMs} from $3$B to $70$B: Llama-3.2-3B and Llama-3.1-Instruct at 8B/70B~\citep{dubey2024llama3}; Falcon3-7B~\citep{falcon3}; Mistral-Nemo-12B~\citep{mistralnemo2407}; and Mistral-Small-24B~\citep{mistralsmall2501}.
We assess transfer to $3$ \emph{VLMs}: Qwen2.5-VL-7B~\citep{bai2025qwen25vl}, Qwen3-VL-8B~\citep{qwen3vl2025}, and InternVL3-8B~\citep{zhu2025internvl3}.
Mechanism analyses cover $10$ LLMs ($3$B--$14$B) and $7$ VLMs ($2$B--$8$B; \appref{app:analysis}).

\textbf{Evaluation.}
LLM evaluation covers $3$ perplexity corpora and $12$ zero-shot tasks grouped into commonsense reasoning, science QA, reading comprehension, and knowledge; VLM transfer covers $7$ benchmarks across perception, reasoning, hallucination, and knowledge.

\textbf{Baselines.}
The primary grid covers $3$ signal classes: \emph{data-free heuristics}, Random and Magnitude~\citep{han2015learning}; \emph{activation statistics}, Wanda-sp~\citep{sun2024wanda} and FLAP~\citep{an2024flap}; and \emph{structure-aware optimization}, LLM-Pruner~\citep{ma2023llmpruner}, SlimGPT~\citep{ling2024slimgpt}, and SlimLLM~\citep{tang2025slimllm}.
\Appref{app:baselines} adds \emph{allocation-aware} AMP~\citep{sultan2025amp} and 2SSP~\citep{sandri2025twossp}.

\textbf{Protocol.}
All methods prune the same units under a shared parameter budget, layer guard, and harness, without weight updates or fine-tuning.
\method builds its graph on unlabeled data and selects $\hat\lambda$ on a disjoint split by teacher KL; downstream tasks never affect pruning.
All efficiency metrics use one isolated A100-80GB.
\Appref{app:setup} details data, prompts, matching, and implementation.

\subsection{Language modeling and downstream capabilities}
\label{sec:exp-main}

\textbf{CoCurve preserves quality as compression becomes harder.}
Across $6$ LLMs, \method is best in $53/60$ corpus-level comparisons: $38/45$ on the five-model shared grid and $15/15$ at $70$B (\cref{fig:pareto}).
The separation grows on distinct architectures: the $8$B WikiText advantage rises from $2.6\%$ at $20\%$ to $19.0\%$ and $42.8\%$ at $30/40\%$, while the $24$B model reaches $20.5$ versus $78.8$ perplexity at $30\%$ and $277.1$ versus $3299$ at $40\%$.
This scale-consistent widening supports interaction-aware selection: as redundancy shrinks, compatibility among jointly pruned units becomes more consequential.
CoCurve therefore advances the quality--deployment frontier: at $30\%$, it wins $94\%$ of $18$ model--corpus comparisons; on the five-model $30/40\%$ grid, it permits $2.2$--$6.6$ points more pruning at matched quality in $9/10$ cases (\appref{app:frontier-analysis}).

\textbf{The advantage transfers from likelihood to downstream decisions.}
Across all eight $8$B/$24$B model--ratio blocks from $20$--$50\%$, \method leads Avg$_{12}$ by $0.4$--$2.0$ points over the strongest baseline and ranks first or second in $31/32$ capability-group comparisons (\cref{tab:flagship}).
At the $50\%$ boundary on the $8$B model, it remains first or second on eight of twelve individual tasks.
It also leads all $12$ dense-agreement group comparisons across $20$--$40\%$ (\appref{app:fidelity}), linking lower predictive risk to preserved decisions rather than an isolated perplexity gain.
\FloatBarrier
\Needspace{4\baselineskip}
\textbf{The advantage persists with in-context demonstrations.}
Across five model--ratio blocks, \method leads every subject-matched five-shot MMLU comparison by $0.2$--$2.1$ points (\cref{tab:fewshot-main}).
The ordering persists when demonstrations change the inference context, indicating that the selected structures preserve usable knowledge beyond zero-shot behavior.

\begin{wraptable}{r}{.625\textwidth}
\vspace{-.25\baselineskip}
\centering\fontsize{5.45}{6.0}\selectfont
\captionsetup{hypcap=false}
\caption{\textbf{Pruned large versus native small.}
}
\label{tab:large-pruned-small-dense}
\adjustbox{width=\linewidth}{\setlength{\tabcolsep}{1.75pt}\renewcommand{\arraystretch}{.92}
\begin{tabular}{@{}lcrrrrrr@{}}
\toprule
\textbf{Checkpoint} & \textbf{Prec.} & \textbf{Peak}$\downarrow$ & \textbf{Cmn$_4$}$\uparrow$ & \textbf{Sci$_3$}$\uparrow$ & \textbf{Read$_4$}$\uparrow$ & \textbf{Know$_1$}$\uparrow$ & \textbf{Avg$_{12}$}$\uparrow$ \\
\midrule
\rowcolor{headgray}\multicolumn{8}{c}{\textbf{Mistral-Small-24B versus Mistral-7B}} \\
\rowcolor{headgray}Mistral-Small-24B dense & bf16 & 45.02 & 73.5 & 63.4 & 65.9 & 52.3 & 66.7 \\
Mistral-7B dense & bf16 & \underline{14.16} & \underline{69.5} & \underline{57.4} & \underline{61.9} & \underline{44.1} & \underline{61.8} \\
\textbf{Mistral-Small-24B (20\%)} & NF4 & \textcolor{accentblue}{\textbf{12.14}} & \textbf{69.6} & \textbf{59.5} & \textbf{64.9} & \textbf{45.5} & \textbf{63.5} \\
\bottomrule
\end{tabular}
}
\vspace{-1pt}
\parbox{\linewidth}{\fontsize{5.15}{5.6}\selectfont\emph{Note.}
Peak is forward-workload GB. Gray marks the dense parent; blue marks the favorable matched point.}
\vspace{-1.05\baselineskip}
\end{wraptable}

\textbf{Inherited capacity yields a better memory--capability point.}
At $12.14$ GB peak memory, the recovered and quantized $24$B checkpoint uses $14\%$ less memory than dense bf16 Mistral-7B yet improves all four capability groups and Avg$_{12}$ by $1.7$ points.
The result directly shows that pruning can retain useful capacity from a larger pretrained parent, yielding higher quality than a native smaller checkpoint under a tighter memory budget (\appref{app:deployment-context}).
\par

\FloatBarrier
\subsection{Recovery and quantization}
\label{sec:exp-recovery}
\vspace{-5pt}

\Needspace{18\baselineskip}
\begin{wraptable}{r}{.595\textwidth}
\vspace{-0.7\baselineskip}
\centering\fontsize{6.1}{6.8}\selectfont
\captionsetup{hypcap=false}
\caption{\textbf{Common-LoRA recovery across model scales.}
Full results: \appref{app:recovery}.}
\label{tab:recovery-main}
\adjustbox{width=\linewidth}{
\setlength{\tabcolsep}{1.45pt}\renewcommand{\arraystretch}{.93}
\begin{tabular}{@{}l>{\centering\arraybackslash}m{2.75em}>{\centering\arraybackslash}m{3.15em}>{\centering\arraybackslash}m{2.75em}>{\centering\arraybackslash}m{3.15em}>{\centering\arraybackslash}m{2.75em}>{\centering\arraybackslash}m{3.15em}>{\centering\arraybackslash}m{2.75em}>{\centering\arraybackslash}m{3.15em}@{}}\toprule
 & \multicolumn{2}{c}{\textbf{20\%}} & \multicolumn{2}{c}{\textbf{30\%}} & \multicolumn{2}{c}{\textbf{40\%}} & \multicolumn{2}{c}{\textbf{50\%}} \\
\cmidrule(lr){2-3}\cmidrule(lr){4-5}\cmidrule(lr){6-7}\cmidrule(lr){8-9}
\textbf{Method} & \textbf{Wiki}$\downarrow$ & \textbf{Avg$_{12}$}$\uparrow$ & \textbf{Wiki}$\downarrow$ & \textbf{Avg$_{12}$}$\uparrow$ & \textbf{Wiki}$\downarrow$ & \textbf{Avg$_{12}$}$\uparrow$ & \textbf{Wiki}$\downarrow$ & \textbf{Avg$_{12}$}$\uparrow$ \\
\midrule
\rowcolor{headgray}\multicolumn{9}{l}{\textbf{Llama-3.1-8B-Instruct}\quad \textbf{Dense:} Wiki 6.4; Avg$_{12}$ 65.2}  \\
Random & \cellcolor{CoCurveShade!7}11.3 & \cellcolor{edgeorange!20}56.8 & \cellcolor{CoCurveShade!7}15.3 & \cellcolor{edgeorange!26}\underline{51.9} & \cellcolor{CoCurveShade!13}20.0 & \cellcolor{edgeorange!20}47.5 & \cellcolor{CoCurveShade!13}32.1 & \cellcolor{edgeorange!20}43.2 \\
LLM-Pruner & \cellcolor{CoCurveShade!32}\textbf{9.9} & \cellcolor{edgeorange!13}56.0 & \cellcolor{CoCurveShade!26}\underline{13.0} & \cellcolor{edgeorange!7}51.3 & \cellcolor{CoCurveShade!26}\underline{17.4} & \cellcolor{edgeorange!26}\underline{47.9} & \cellcolor{CoCurveShade!26}\underline{28.0} & \cellcolor{edgeorange!26}\underline{44.6} \\
SlimGPT & \cellcolor{CoCurveShade!13}10.9 & \cellcolor{edgeorange!7}55.6 & \cellcolor{CoCurveShade!13}14.7 & \cellcolor{edgeorange!20}51.8 & \cellcolor{CoCurveShade!7}24.3 & \cellcolor{edgeorange!7}45.3 & \cellcolor{CoCurveShade!7}41.8 & \cellcolor{edgeorange!7}40.7 \\
SlimLLM & \cellcolor{CoCurveShade!20}10.2 & \cellcolor{edgeorange!26}\underline{56.8} & \cellcolor{CoCurveShade!20}13.9 & \cellcolor{edgeorange!13}51.7 & \cellcolor{CoCurveShade!20}19.6 & \cellcolor{edgeorange!13}47.4 & \cellcolor{CoCurveShade!20}31.1 & \cellcolor{edgeorange!13}42.0 \\
\rowcolor{black!6}\textbf{No LoRA} & 12.9 & 56.0 & 21.6 & 49.0 & 59.1 & 41.4 & 956.4 & 34.5 \\
\textbf{CoCurve} & \cellcolor{CoCurveShade!26}\underline{10.0} & \cellcolor{edgeorange!32}\textbf{57.1} & \cellcolor{CoCurveShade!32}\textbf{12.6} & \cellcolor{edgeorange!32}\textbf{53.7} & \cellcolor{CoCurveShade!32}\textbf{16.3} & \cellcolor{edgeorange!32}\textbf{49.6} & \cellcolor{CoCurveShade!32}\textbf{24.2} & \cellcolor{edgeorange!32}\textbf{45.3} \\
\midrule
\rowcolor{headgray}\multicolumn{9}{l}{\textbf{Mistral-Small-24B}\quad \textbf{Dense:} Wiki 4.7; Avg$_{12}$ 66.7}  \\
Random & \cellcolor{CoCurveShade!13}6.7 & \cellcolor{edgeorange!13}64.3 & \cellcolor{CoCurveShade!20}8.2 & \cellcolor{edgeorange!13}61.0 & \cellcolor{CoCurveShade!7}11.5 & \cellcolor{edgeorange!20}57.3 & \cellcolor{CoCurveShade!7}16.4 & \cellcolor{edgeorange!20}53.5 \\
LLM-Pruner & \cellcolor{CoCurveShade!7}6.8 & \cellcolor{edgeorange!32}\textbf{65.5} & \cellcolor{CoCurveShade!13}8.3 & \cellcolor{edgeorange!32}\textbf{62.3} & \cellcolor{CoCurveShade!20}10.4 & \cellcolor{edgeorange!26}\underline{58.5} & \cellcolor{CoCurveShade!26}\underline{14.4} & \cellcolor{edgeorange!26}\underline{53.9} \\
SlimGPT & \cellcolor{CoCurveShade!20}6.7 & \cellcolor{edgeorange!7}64.0 & \cellcolor{CoCurveShade!26}\underline{8.0} & \cellcolor{edgeorange!7}60.5 & \cellcolor{CoCurveShade!26}\underline{10.1} & \cellcolor{edgeorange!7}55.5 & \cellcolor{CoCurveShade!32}\textbf{13.4} & \cellcolor{edgeorange!7}50.3 \\
SlimLLM & \cellcolor{CoCurveShade!32}\textbf{6.4} & \cellcolor{edgeorange!26}\underline{65.2} & \cellcolor{CoCurveShade!7}8.5 & \cellcolor{edgeorange!20}61.6 & \cellcolor{CoCurveShade!13}11.4 & \cellcolor{edgeorange!13}56.9 & \cellcolor{CoCurveShade!13}15.6 & \cellcolor{edgeorange!13}51.6 \\
\rowcolor{black!6}\textbf{No LoRA} & 9.0 & 60.6 & 20.5 & 52.9 & 277.1 & 39.5 & 2.3e3 & 33.2 \\
\textbf{CoCurve} & \cellcolor{CoCurveShade!26}\underline{6.6} & \cellcolor{edgeorange!20}64.4 & \cellcolor{CoCurveShade!32}\textbf{8.0} & \cellcolor{edgeorange!26}\underline{61.7} & \cellcolor{CoCurveShade!32}\textbf{9.7} & \cellcolor{edgeorange!32}\textbf{58.9} & \cellcolor{CoCurveShade!20}14.7 & \cellcolor{edgeorange!32}\textbf{54.3} \\
\bottomrule\end{tabular}
}
\vspace{-0.65\baselineskip}
\end{wraptable}

\textbf{Structure sets the recovery frontier.}
One common lightweight adaptation should recover a useful mask.
We apply one rank-$16$ recipe to five masks on an $8$B Llama and a $24$B Mistral, using a single lightweight pass over cleaned Alpaca for every mask.
The $8$B training stage takes only $1.2$--$1.9$ GPU-hours per mask, excluding evaluation.
Across all eight blocks, \method ranks top-two in $14/16$ WikiText-2/Avg$_{12}$ columns and leads both $8$B metrics from $30$--$50\%$ (\cref{tab:recovery-main}).
Across six scales, the recipe lowers perplexity at all $24$ measured points; at $50\%$, it raises $8$B Avg$_{12}$ by $10.8$ points and cuts selection-only perplexity from $41.7$--$4616.9$ to $12.6$--$38.0$ (\cref{fig:pareto}).
Thus \method preserves an adaptable subnetwork, not merely a favorable unrecovered checkpoint.
\par
\FloatBarrier

\begin{table}[!t]
\centering\fontsize{6.0}{6.5}\selectfont
\caption{\textbf{Structural pruning and weight quantization on Llama-3.1-8B-Instruct.}
Quality and measured deployment cost at the same $30/40\%$ structures; task breakdown: \appref{app:quant}.}
\label{tab:quant-main}
\adjustbox{width=\textwidth}{
\setlength{\tabcolsep}{1.15pt}\renewcommand{\arraystretch}{0.89}
\begin{tabular}{@{}>{\centering\arraybackslash}m{2.15em}l>{\centering\arraybackslash}m{2.2em}>{\centering\arraybackslash}m{3.55em}>{\centering\arraybackslash}m{3.55em}>{\centering\arraybackslash}m{3.55em}>{\centering\arraybackslash}m{3.55em}!{\color{black!55}\vrule width .65pt}>{\centering\arraybackslash}m{3.15em}>{\centering\arraybackslash}m{3.15em}>{\centering\arraybackslash}m{3.15em}!{\color{black!55}\vrule width .65pt}>{\centering\arraybackslash}m{3.15em}>{\centering\arraybackslash}m{3.15em}>{\centering\arraybackslash}m{3.15em}>{\centering\arraybackslash}m{3.15em}!{\color{black!55}\vrule width .65pt}>{\centering\arraybackslash}m{4.8em}@{}}\toprule
 & & & \multicolumn{4}{c}{\textbf{Efficiency}} & \multicolumn{3}{c}{\textbf{Perplexity} $\downarrow$} & \multicolumn{4}{c}{\textbf{Downstream groups} $\uparrow$} &  \\
\cmidrule(lr){4-7}\cmidrule(lr){8-10}\cmidrule(lr){11-14}
\textbf{$\rho$} & \textbf{Method} & \textbf{Prec.} & \textbf{Param.}$\downarrow$ & \textbf{Mem.}$\downarrow$ & \textbf{Pref.}$\uparrow$ & \textbf{Lat.}$\downarrow$ & \textbf{Wiki} & \textbf{PTB} & \textbf{C4} & \textbf{Cmn$_4$} & \textbf{Sci$_3$} & \textbf{Read$_4$} & \textbf{Know$_1$} & {\fontsize{5.5}{5.9}\selectfont\textbf{Avg$_{12}$}$\,\pm$SE} \\
\midrule
-- & \textbf{Dense} & bf16 & \cellcolor{headgray}8.03 & \cellcolor{headgray}15.99 & \cellcolor{headgray}4112 & \cellcolor{headgray}31.74 & \cellcolor{headgray}6.4 & \cellcolor{headgray}10.4 & \cellcolor{headgray}10.7 & \cellcolor{headgray}70.9 & \cellcolor{headgray}60.1 & \cellcolor{headgray}67.6 & \cellcolor{headgray}47.8 & \cellcolor{headgray}65.2\se{0.34} \\
 &  & INT8 & \cellcolor{headgray}8.03 & \cellcolor{headgray}9.50 & \cellcolor{headgray}3035 & \cellcolor{headgray}133.06 & \cellcolor{headgray}6.5 & \cellcolor{headgray}10.4 & \cellcolor{headgray}10.9 & \cellcolor{headgray}70.9 & \cellcolor{headgray}60.2 & \cellcolor{headgray}67.4 & \cellcolor{headgray}47.7 & \cellcolor{headgray}65.1\se{0.34} \\
 &  & NF4 & \cellcolor{headgray}8.03 & \cellcolor{headgray}6.34 & \cellcolor{headgray}3934 & \cellcolor{headgray}47.90 & \cellcolor{headgray}6.8 & \cellcolor{headgray}10.8 & \cellcolor{headgray}11.5 & \cellcolor{headgray}70.7 & \cellcolor{headgray}58.4 & \cellcolor{headgray}66.8 & \cellcolor{headgray}46.7 & \cellcolor{headgray}64.3\se{0.34} \\
\midrule
\multirow{7}{*}{30\%} & \multirow{2}{*}{LLM-Pruner} & INT8 & 5.93 & 7.54 & 3659 & 132.76 & \cellcolor{CoCurveShade!20}\underline{26.8} & \cellcolor{CoCurveShade!20}\underline{39.6} & \cellcolor{CoCurveShade!20}\underline{40.8} & \cellcolor{CoCurveShade!20}\underline{54.6} & \cellcolor{CoCurveShade!20}\underline{38.1} & \cellcolor{CoCurveShade!32}\textbf{52.1} & \cellcolor{CoCurveShade!20}\underline{31.8} & \cellcolor{edgeorange!20}\underline{47.8}\se{0.34} \\
 &  & NF4 & 5.93 & 5.34 & 5065 & 44.65 & \cellcolor{CoCurveShade!20}\underline{29.2} & \cellcolor{CoCurveShade!20}\underline{42.2} & \cellcolor{CoCurveShade!20}\underline{44.3} & \cellcolor{CoCurveShade!32}\textbf{54.4} & \cellcolor{CoCurveShade!20}\underline{37.2} & \cellcolor{CoCurveShade!32}\textbf{51.1} & \cellcolor{CoCurveShade!20}\underline{31.4} & \cellcolor{edgeorange!20}\underline{47.1}\se{0.34} \\
 & \multirow{2}{*}{SlimLLM} & INT8 & 5.93 & 7.54 & 3726 & 130.36 & \cellcolor{CoCurveShade!7}80.0 & \cellcolor{CoCurveShade!7}102.1 & \cellcolor{CoCurveShade!7}115.3 & \cellcolor{CoCurveShade!7}48.6 & \cellcolor{CoCurveShade!7}36.5 & \cellcolor{CoCurveShade!7}41.0 & \cellcolor{CoCurveShade!7}30.5 & \cellcolor{edgeorange!7}41.5\se{0.34} \\
 &  & NF4 & 5.93 & 5.34 & 5266 & 44.10 & \cellcolor{CoCurveShade!7}80.0 & \cellcolor{CoCurveShade!7}101.5 & \cellcolor{CoCurveShade!7}114.4 & \cellcolor{CoCurveShade!7}48.2 & \cellcolor{CoCurveShade!7}36.7 & \cellcolor{CoCurveShade!7}41.3 & \cellcolor{CoCurveShade!7}30.4 & \cellcolor{edgeorange!7}41.6\se{0.34} \\
 & \multirow{3}{*}{\textbf{CoCurve}} & bf16 & \cellcolor{headgray}5.93 & \cellcolor{headgray}12.08 & \cellcolor{headgray}5232 & \cellcolor{headgray}30.64 & \cellcolor{headgray}21.6 & \cellcolor{headgray}36.3 & \cellcolor{headgray}32.0 & \cellcolor{headgray}55.7 & \cellcolor{headgray}41.7 & \cellcolor{headgray}51.8 & \cellcolor{headgray}33.4 & \cellcolor{headgray}49.0\se{0.35} \\
 &  & INT8 & \textcolor{accentblue}{\textbf{5.93}} & \textcolor{accentblue}{\textbf{7.54}} & \textcolor{accentblue}{\textbf{3606}} & 134.61 & \cellcolor{CoCurveShade!32}\textbf{21.7} & \cellcolor{CoCurveShade!32}\textbf{36.9} & \cellcolor{CoCurveShade!32}\textbf{32.3} & \cellcolor{CoCurveShade!32}\textbf{55.6} & \cellcolor{CoCurveShade!32}\textbf{41.7} & \cellcolor{CoCurveShade!20}\underline{51.7} & \cellcolor{CoCurveShade!32}\textbf{33.3} & \cellcolor{edgeorange!32}\textbf{48.9}\se{0.35} \\
 &  & NF4 & \textcolor{accentblue}{\textbf{5.93}} & \textcolor{accentblue}{\textbf{5.34}} & \textcolor{accentblue}{\textbf{5013}} & \textcolor{accentblue}{\textbf{46.14}} & \cellcolor{CoCurveShade!32}\textbf{23.1} & \cellcolor{CoCurveShade!32}\textbf{38.5} & \cellcolor{CoCurveShade!32}\textbf{34.8} & \cellcolor{CoCurveShade!20}\underline{54.1} & \cellcolor{CoCurveShade!32}\textbf{38.4} & \cellcolor{CoCurveShade!20}\underline{50.9} & \cellcolor{CoCurveShade!32}\textbf{32.4} & \cellcolor{edgeorange!32}\textbf{47.3}\se{0.34} \\
\midrule
\multirow{7}{*}{40\%} & \multirow{2}{*}{LLM-Pruner} & INT8 & 5.23 & 6.89 & 3931 & 130.22 & \cellcolor{CoCurveShade!20}\underline{101.3} & \cellcolor{CoCurveShade!20}\underline{162.9} & \cellcolor{CoCurveShade!20}\underline{124.9} & \cellcolor{CoCurveShade!32}\textbf{48.4} & \cellcolor{CoCurveShade!20}\underline{32.1} & \cellcolor{CoCurveShade!20}\underline{43.7} & \cellcolor{CoCurveShade!32}\textbf{29.6} & \cellcolor{edgeorange!32}\textbf{41.2}\se{0.34} \\
 &  & NF4 & 5.23 & 5.00 & 5464 & 44.56 & \cellcolor{CoCurveShade!20}\underline{116.1} & \cellcolor{CoCurveShade!20}\underline{194.2} & \cellcolor{CoCurveShade!20}\underline{138.8} & \cellcolor{CoCurveShade!32}\textbf{48.3} & \cellcolor{CoCurveShade!20}\underline{30.8} & \cellcolor{CoCurveShade!20}\underline{42.2} & \cellcolor{CoCurveShade!32}\textbf{29.1} & \cellcolor{edgeorange!32}\textbf{40.3}\se{0.33} \\
 & \multirow{2}{*}{SlimLLM} & INT8 & 5.23 & 6.89 & 4025 & 130.34 & \cellcolor{CoCurveShade!7}187.8 & \cellcolor{CoCurveShade!7}219.6 & \cellcolor{CoCurveShade!7}222.5 & \cellcolor{CoCurveShade!7}42.9 & \cellcolor{CoCurveShade!7}29.5 & \cellcolor{CoCurveShade!7}36.4 & \cellcolor{CoCurveShade!7}27.7 & \cellcolor{edgeorange!7}36.1\se{0.33} \\
 &  & NF4 & 5.23 & 5.00 & 5942 & 42.90 & \cellcolor{CoCurveShade!7}192.8 & \cellcolor{CoCurveShade!7}272.0 & \cellcolor{CoCurveShade!7}242.6 & \cellcolor{CoCurveShade!7}43.6 & \cellcolor{CoCurveShade!7}29.9 & \cellcolor{CoCurveShade!7}35.4 & \cellcolor{CoCurveShade!7}27.7 & \cellcolor{edgeorange!7}36.1\se{0.33} \\
 & \multirow{3}{*}{\textbf{CoCurve}} & bf16 & \cellcolor{headgray}5.24 & \cellcolor{headgray}10.79 & \cellcolor{headgray}5654 & \cellcolor{headgray}31.45 & \cellcolor{headgray}59.1 & \cellcolor{headgray}108.4 & \cellcolor{headgray}76.1 & \cellcolor{headgray}47.1 & \cellcolor{headgray}33.3 & \cellcolor{headgray}44.8 & \cellcolor{headgray}29.3 & \cellcolor{headgray}41.4\se{0.34} \\
 &  & INT8 & \textcolor{accentblue}{\textbf{5.24}} & \textcolor{accentblue}{\textbf{6.89}} & \textcolor{accentblue}{\textbf{3889}} & \textcolor{accentblue}{\textbf{131.11}} & \cellcolor{CoCurveShade!32}\textbf{57.6} & \cellcolor{CoCurveShade!32}\textbf{105.4} & \cellcolor{CoCurveShade!32}\textbf{75.0} & \cellcolor{CoCurveShade!20}\underline{46.2} & \cellcolor{CoCurveShade!32}\textbf{33.0} & \cellcolor{CoCurveShade!32}\textbf{44.7} & \cellcolor{CoCurveShade!20}\underline{29.2} & \cellcolor{edgeorange!20}\underline{41.0}\se{0.34} \\
 &  & NF4 & \textcolor{accentblue}{\textbf{5.24}} & \textcolor{accentblue}{\textbf{5.00}} & \textcolor{accentblue}{\textbf{5413}} & \textcolor{accentblue}{\textbf{44.87}} & \cellcolor{CoCurveShade!32}\textbf{66.9} & \cellcolor{CoCurveShade!32}\textbf{140.8} & \cellcolor{CoCurveShade!32}\textbf{86.7} & \cellcolor{CoCurveShade!20}\underline{45.8} & \cellcolor{CoCurveShade!32}\textbf{31.3} & \cellcolor{CoCurveShade!32}\textbf{43.8} & \cellcolor{CoCurveShade!20}\underline{28.6} & \cellcolor{edgeorange!20}\underline{40.1}\se{0.34} \\
\bottomrule\end{tabular}
}
\vspace{-1pt}
\parbox{\textwidth}{\fontsize{5.5}{6.1}\selectfont\emph{Note.}
Params: B; peak: GB; prefill: tokens/s; latency: ms/step. Blue marks CoCurve gains over matched-precision Dense; shading ranks quality within $(\rho,\text{precision})$.}
\end{table}
\vspace{-8pt}

\textbf{Low-bit weights preserve the selection advantage.}
We apply INT8 and NF4 to fixed masks from three distinct pruners, so within-precision comparisons isolate the selected structure.
Across both ratios and precisions, \method leads all $12$ corpus-level perplexity comparisons and is first or within $0.2$ points of first on every Avg$_{12}$ column, showing that the selected structure retains its quality advantage across both low-bit formats (\cref{tab:quant-main}).

\textbf{Structural and numerical compression compound.}
Within the same CoCurve structure, NF4 cuts peak memory by $54$--$56\%$ relative to bf16 while retaining $96\%$ of prefill throughput; Avg$_{12}$ changes by only $1.3$--$1.7$ points.
Relative to dense NF4, physically slicing the $30/40\%$ structures further lowers memory by $16/21\%$ and raises prefill throughput by $27/38\%$.
Thus low-bit weights compress numerical storage, while CoCurve removes dense computation itself; their composition captures both gains at one deployment point.

\subsection{Inference efficiency}
\label{sec:exp-efficiency}
\vspace{-7pt}

\begin{center}
\vspace{-4pt}
\begin{minipage}[t]{.605\textwidth}
\centering\fontsize{5.5}{6.15}\selectfont
\captionsetup{hypcap=false}
\captionof{table}{\textbf{Inference gains from physically sliced checkpoints.}
Absolute measurements and matched-dense gains.}
\label{tab:efficiency}
\adjustbox{width=\linewidth}{
\setlength{\tabcolsep}{1.35pt}\renewcommand{\arraystretch}{.94}
\begin{tabular}{@{}l>{\centering\arraybackslash}m{2.1em}>{\centering\arraybackslash}m{3.5em}rrr@{}}\toprule
\textbf{Model} & \textbf{$\rho$} & \textbf{Param.}$\downarrow$ & \textbf{Prefill}$\uparrow$ & \textbf{Decode}$\uparrow$ & \textbf{Peak}$\downarrow$ \\
\midrule
\rowcolor{headgray}\multicolumn{6}{c}{\textbf{Scale at a fixed $20\%$ pruning budget}} \\
Llama-3.2-3B & 20\% & 2.64 & 10027 (\textcolor{accentblue}{1.17}$\times$) & 548 (\textcolor{accentblue}{0.95}$\times$) & 6.11 (\textcolor{accentblue}{0.83}$\times$) \\
Falcon3-7B & 20\% & 6.12 & 6059 (\textcolor{accentblue}{1.18}$\times$) & 561 (\textcolor{accentblue}{1.03}$\times$) & 13.59 (\textcolor{accentblue}{0.84}$\times$) \\
Llama-3.1-8B-It & 20\% & 6.63 & 5200 (\textcolor{accentblue}{1.21}$\times$) & 503 (\textcolor{accentblue}{1.01}$\times$) & 14.29 (\textcolor{accentblue}{0.83}$\times$) \\
Mistral-Nemo-12B & 20\% & 10.06 & 3485 (\textcolor{accentblue}{1.20}$\times$) & 404 (\textcolor{accentblue}{1.09}$\times$) & 21.36 (\textcolor{accentblue}{0.83}$\times$) \\
\midrule
\rowcolor{headgray}\multicolumn{6}{c}{\textbf{Compression levels on Llama-3.1-8B-Instruct}} \\
Llama-3.1-8B-It & 30\% & 5.93 & 5524 (\textcolor{accentblue}{1.28}$\times$) & 504 (\textcolor{accentblue}{1.02}$\times$) & 12.94 (\textcolor{accentblue}{0.75}$\times$) \\
Llama-3.1-8B-It & 40\% & 5.24 & 6098 (\textcolor{accentblue}{1.41}$\times$) & 512 (\textcolor{accentblue}{1.04}$\times$) & 11.53 (\textcolor{accentblue}{0.67}$\times$) \\
Llama-3.1-8B-It & 50\% & 4.53 & 6801 (\textcolor{accentblue}{1.58}$\times$) & 511 (\textcolor{accentblue}{1.06}$\times$) & 10.10 (\textcolor{accentblue}{0.59}$\times$) \\
\bottomrule\end{tabular}
}
\vspace{-1pt}
\parbox{\linewidth}{\fontsize{5.05}{5.55}\selectfont\emph{Note.}
bf16, eager attention. Params: B; prefill/decode: tokens/s; generation peak: GB.
Blue numbers in parentheses are ratios to dense; $\rho$ is the removable-projection budget.}
\end{minipage}
\hfill
\begin{minipage}[t]{.375\textwidth}
\vspace{0pt}
\centering
\captionsetup{font=small,hypcap=false}
\IfFileExists{figures/fig_quality_efficiency_8b.pdf}{%
\includegraphics[width=\linewidth,trim=0 4pt 0 0,clip]{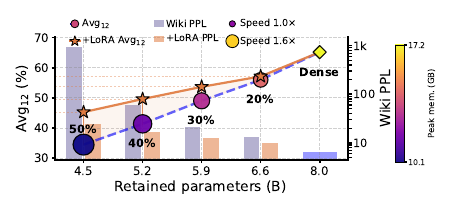}}{%
\pendingpanel{31mm}{8B quality--efficiency frontier at $20$--$50\%$.}}\par
\vspace{-2pt}
\nointerlineskip
\IfFileExists{figures/fig_quant_efficiency_8b.pdf}{%
\includegraphics[width=\linewidth,trim=0 0 0 4pt,clip]{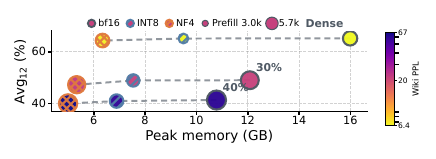}}{%
\pendingpanel{22mm}{Dense and $30/40\%$ CoCurve trajectories from bf16 to INT8 and NF4.}}\par
\vspace{-1pt}
\captionof{figure}{\textbf{Quality--efficiency frontiers.}
Cross-model view: \appref{app:efficiency}.}
\label{fig:ratio-performance}
\end{minipage}
\end{center}

\textbf{Physical slicing realizes deployment gains across scale and compression.}
At $20\%$ from $3$B to $12$B, \method raises prefill throughput by $17$--$21\%$, lowers peak memory by $16$--$17\%$, and keeps decode within $5\%$ of dense or improves it by up to $9\%$.
Along the $8$B path, prefill rises from $1.28\times$ to $1.58\times$ and peak memory falls from $0.75\times$ to $0.59\times$ dense, while decode remains at $1.02$--$1.06\times$ (\cref{tab:efficiency}).
The masks therefore compile into smaller dense workloads rather than logical sparsity.

\textbf{Selection, recovery, and precision move complementary frontier axes.}
At matched $20/40\%$ budgets, \method shares the narrow throughput--memory band of other sliced selectors yet retains higher quality (\cref{tab:efficiency-band}).
In \cref{fig:ratio-performance}, LoRA moves each unchanged checkpoint toward dense quality, whereas NF4 moves it toward lower memory.
One graph therefore supports independently controlled structure, recovered quality, and numerical precision.

\subsection{Transfer to vision--language models}
\label{sec:exp-vlm}

\textbf{Co-pruning transfers without tower-specific design.}
One teacher-KL graph and conditional pruner jointly score and prune units across the vision and language towers, without separate objectives or manually assigned tower budgets.
\method leads all $6$ pruning-only Avg$_7$ blocks across the three VLMs at $20/30\%$.
At $30\%$, it improves over the strongest external pruner by $4.9$, $10.7$, and $1.1$ points; removing only the edge term loses $2.5$--$7.6$ points under identical features, guards, and budgets.
The consistent separation at this harder point shows that one graph-conditioned objective transfers across VLM architectures without bespoke tower allocations.

\vspace{-7pt}
\begin{center}
\vspace{-4pt}
\begin{minipage}[t]{.535\textwidth}
\vspace{0pt}
\centering\fontsize{4.9}{5.45}\selectfont
\captionsetup{hypcap=false,skip=2pt}
\captionof{table}{\textbf{VLM transfer/recovery.}
Full: \cref{tab:vlm-full}.}
\label{tab:vlm}
\adjustbox{width=\linewidth}{
\setlength{\tabcolsep}{2.7pt}\renewcommand{\arraystretch}{1.04}
\begin{tabular}{@{}l l>{\centering\arraybackslash}m{3.75em}>{\centering\arraybackslash}m{3.75em}>{\centering\arraybackslash}m{3.75em}>{\centering\arraybackslash}m{3.75em}>{\centering\arraybackslash}m{3.75em}>{\centering\arraybackslash}m{3.75em}@{}}\toprule
\textbf{Model} & \textbf{Method} & \multicolumn{2}{c}{\textbf{$20\%$}} & \multicolumn{2}{c}{\textbf{$30\%$}} & \textbf{$40\%$} & \textbf{$50\%$} \\
\cmidrule(lr){3-4}\cmidrule(lr){5-6}\cmidrule(lr){7-7}\cmidrule(lr){8-8}
 & & \textbf{Sel.} & \textbf{$+$LoRA} & \textbf{Sel.} & \textbf{$+$LoRA} & \textbf{$+$LoRA} & \textbf{$+$LoRA} \\
\midrule
\rowcolor{headgray}\multicolumn{8}{l}{\textbf{Qwen2.5-VL-7B}\quad \textbf{Dense Avg$_7$:} 81.6}  \\
 & Magnitude & \cellcolor{CoCurveShade!20}72.9 & \cellcolor{CoCurveShade!13}75.1 & \cellcolor{CoCurveShade!26}\underline{58.7} & \cellcolor{CoCurveShade!20}69.2 & \cellcolor{CoCurveShade!26}\underline{61.7} & \cellcolor{CoCurveShade!20}48.3 \\
 & Wanda-sp & \cellcolor{CoCurveShade!7}62.9 & \cellcolor{CoCurveShade!7}69.2 & \cellcolor{CoCurveShade!7}52.2 & \cellcolor{CoCurveShade!7}61.4 & \cellcolor{CoCurveShade!7}50.3 & \cellcolor{CoCurveShade!7}35.8 \\
 & FLAP & \cellcolor{CoCurveShade!26}\underline{73.3} & \cellcolor{CoCurveShade!26}\underline{75.3} & \cellcolor{CoCurveShade!20}58.5 & \cellcolor{CoCurveShade!13}68.7 & \cellcolor{CoCurveShade!20}61.5 & \cellcolor{CoCurveShade!13}42.6 \\
 & LLM-Pruner & \cellcolor{CoCurveShade!13}72.2 & \cellcolor{CoCurveShade!20}75.1 & \cellcolor{CoCurveShade!13}57.2 & \cellcolor{CoCurveShade!32}\textbf{70.9} & \cellcolor{CoCurveShade!13}61.4 & \cellcolor{CoCurveShade!26}\underline{50.1} \\
 & \textbf{CoCurve} & \cellcolor{CoCurveShade!32}\textbf{74.4} & \cellcolor{CoCurveShade!32}\textbf{76.2} & \cellcolor{CoCurveShade!32}\textbf{63.6} & \cellcolor{CoCurveShade!26}\underline{69.7} & \cellcolor{CoCurveShade!32}\textbf{61.9} & \cellcolor{CoCurveShade!32}\textbf{55.1} \\
\midrule
\rowcolor{headgray}\multicolumn{8}{l}{\textbf{Qwen3-VL-8B}\quad \textbf{Dense Avg$_7$:} 83.4}  \\
 & Magnitude & \cellcolor{CoCurveShade!13}57.0 & \cellcolor{CoCurveShade!13}68.8 & \cellcolor{CoCurveShade!13}37.7 & \cellcolor{CoCurveShade!13}58.3 & \cellcolor{CoCurveShade!7}35.6 & \cellcolor{CoCurveShade!7}33.3 \\
 & Wanda-sp & \cellcolor{CoCurveShade!7}40.5 & \cellcolor{CoCurveShade!7}51.3 & \cellcolor{CoCurveShade!7}35.9 & \cellcolor{CoCurveShade!7}40.8 & \cellcolor{CoCurveShade!13}35.6 & \cellcolor{CoCurveShade!20}34.0 \\
 & FLAP & \cellcolor{CoCurveShade!26}\underline{68.4} & \cellcolor{CoCurveShade!20}70.7 & \cellcolor{CoCurveShade!20}46.5 & \cellcolor{CoCurveShade!20}63.9 & \cellcolor{CoCurveShade!26}\underline{54.5} & \cellcolor{CoCurveShade!13}34.0 \\
 & LLM-Pruner & \cellcolor{CoCurveShade!20}65.8 & \cellcolor{CoCurveShade!26}\underline{72.8} & \cellcolor{CoCurveShade!26}\underline{47.0} & \cellcolor{CoCurveShade!26}\underline{68.6} & \cellcolor{CoCurveShade!32}\textbf{59.8} & \cellcolor{CoCurveShade!32}\textbf{36.6} \\
 & \textbf{CoCurve} & \cellcolor{CoCurveShade!32}\textbf{71.1} & \cellcolor{CoCurveShade!32}\textbf{75.1} & \cellcolor{CoCurveShade!32}\textbf{57.7} & \cellcolor{CoCurveShade!32}\textbf{69.5} & \cellcolor{CoCurveShade!20}52.1 & \cellcolor{CoCurveShade!26}\underline{34.8} \\
\midrule
\rowcolor{headgray}\multicolumn{8}{l}{\textbf{InternVL3-8B}\quad \textbf{Dense Avg$_7$:} 83.6}  \\
 & Magnitude & \cellcolor{CoCurveShade!13}70.7 & \cellcolor{CoCurveShade!13}75.0 & \cellcolor{CoCurveShade!13}45.2 & \cellcolor{CoCurveShade!13}64.6 & \cellcolor{CoCurveShade!26}\underline{56.4} & \cellcolor{CoCurveShade!13}33.8 \\
 & Wanda-sp & \cellcolor{CoCurveShade!7}49.8 & \cellcolor{CoCurveShade!7}51.0 & \cellcolor{CoCurveShade!7}45.0 & \cellcolor{CoCurveShade!7}49.6 & \cellcolor{CoCurveShade!7}44.9 & \cellcolor{CoCurveShade!20}35.2 \\
 & FLAP & \cellcolor{CoCurveShade!26}\underline{74.8} & \cellcolor{CoCurveShade!26}\underline{76.6} & \cellcolor{CoCurveShade!26}\underline{61.3} & \cellcolor{CoCurveShade!32}\textbf{69.9} & \cellcolor{CoCurveShade!13}54.6 & \cellcolor{CoCurveShade!7}33.8 \\
 & LLM-Pruner & \cellcolor{CoCurveShade!20}71.8 & \cellcolor{CoCurveShade!20}75.8 & \cellcolor{CoCurveShade!20}60.3 & \cellcolor{CoCurveShade!20}68.7 & \cellcolor{CoCurveShade!20}55.1 & \cellcolor{CoCurveShade!26}\underline{41.7} \\
 & \textbf{CoCurve} & \cellcolor{CoCurveShade!32}\textbf{75.9} & \cellcolor{CoCurveShade!32}\textbf{76.9} & \cellcolor{CoCurveShade!32}\textbf{62.4} & \cellcolor{CoCurveShade!26}\underline{69.2} & \cellcolor{CoCurveShade!32}\textbf{59.5} & \cellcolor{CoCurveShade!32}\textbf{53.8} \\
\bottomrule\end{tabular}
}
\end{minipage}\hfill
\begin{minipage}[t]{.445\textwidth}
\vspace{0pt}
\centering\fontsize{5.0}{5.55}\selectfont
\captionsetup{hypcap=false,skip=2pt}
\IfFileExists{figures/fig_vlm_recovery_ppl.pdf}{%
\includegraphics[width=\linewidth]{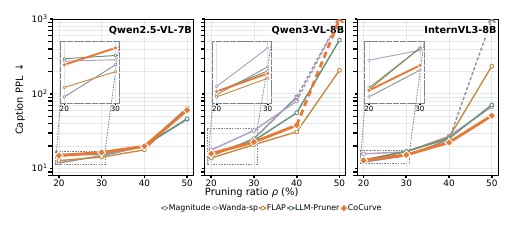}}{%
\pendingpanel{23mm}{Recovered caption perplexity for three VLMs at $20$--$50\%$.}}
\captionof{figure}{\textbf{Recovered caption perplexity $\downarrow$.}}
\label{fig:vlm-recovery-ppl}
\vspace{0.5mm}
\IfFileExists{figures/fig_vlm_quality_efficiency.pdf}{%
\includegraphics[width=\linewidth]{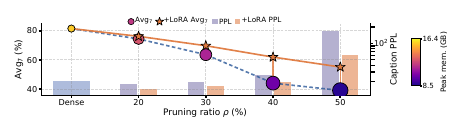}}{%
\pendingpanel{12mm}{Measured Avg$_7$, caption perplexity, prefill speedup, and peak memory.}}
\captionof{figure}{\textbf{Qwen2.5-VL quality--efficiency.}
Full: \appref{app:vlm-efficiency}.}
\label{fig:vlm-quality-efficiency}
\end{minipage}
\end{center}

\textbf{The selected multimodal structures remain recoverable.}
Under one LoRA recipe across methods, models, and ratios, \method ranks first or second in $11/12$ Avg$_7$ blocks through $50\%$ pruning and leads both Qwen2.5-VL-7B and InternVL3-8B at that boundary (\cref{tab:vlm}).
The caption-perplexity curves expose the same recovery boundary without changing masks: edges select a viable structure, and lightweight adaptation recovers capability within it (\cref{fig:vlm-recovery-ppl}).
After slicing, Qwen2.5-VL reaches $1.23\times$ dense prefill throughput with $43\%$ lower peak memory at $50\%$; memory falls monotonically across all three VLMs (\cref{fig:vlm-quality-efficiency}).

\FloatBarrier

\Needspace{8\baselineskip}
\subsection{Interaction-aware selection and calibration}
\label{sec:exp-ablation}
\vspace{-2pt}

\begin{center}
\begin{minipage}[t]{.49\textwidth}
\vspace{0pt}
\centering\footnotesize
\captionsetup{hypcap=false}
\captionof{table}{\textbf{Contribution of co-pruning edges.}
Full results: \cref{tab:ablation-full}.}
\label{tab:ablation-main}
\adjustbox{max width=\linewidth}{
\setlength{\tabcolsep}{3.5pt}\renewcommand{\arraystretch}{1.02}
\begin{tabular}{@{}llrrr@{}}\toprule
\textbf{$\rho$} & \textbf{Selection rule} & \textbf{$\Delta$ KL $\downarrow$} & \textbf{$\Delta$ Wiki $\downarrow$} & \textbf{$\Delta$ Avg$_{12}$ $\uparrow$} \\
\midrule
\multirow{3}{*}{20\%} & Diagonal & 0.000 & 0.0 & 0.0 \\
 & Edges, fixed split & \textbf{-0.049} & -0.4 & +1.2 \\
 & \textbf{CoCurve} & -0.032 & \textbf{-0.3} & \textbf{+2.0} \\
\midrule
\multirow{3}{*}{30\%} & Diagonal & 0.000 & 0.0 & 0.0 \\
 & Edges, fixed split & \textbf{-0.240} & -8.2 & +1.1 \\
 & \textbf{CoCurve} & -0.108 & \textbf{-9.3} & \textbf{+1.4} \\
\bottomrule\end{tabular}
}
\end{minipage}
\hfill
\begin{minipage}[t]{.47\textwidth}
\vspace{0pt}
\captionsetup{hypcap=false,skip=2pt}
\centering
\includegraphics[width=\linewidth]{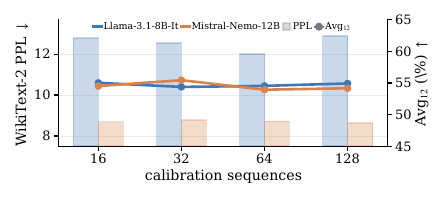}
\captionof{figure}{\textbf{Calibration-data scaling.}
Full sensitivity: \appref{app:robustness}.}
\label{fig:calibration-samples}
\end{minipage}
\end{center}

\textbf{Co-pruning edges improve decisions beyond diagonal saliency.}
Holding the diagonal selector's attention/FFN allocation fixed, edge-conditioned reordering lowers independent test KL by $0.049/0.240$ and WikiText-2 perplexity by $0.4/8.2$ at $20/30\%$, isolating the interaction benefit.
Joint competition then gives the largest Avg$_{12}$ gains ($2.0/1.4$ points) and $30\%$ perplexity reduction ($9.3$).
The fixed-split arm has lower test KL in all four blocks but joint allocation has higher Avg$_{12}$ in three, reflecting that label-free next-token KL preserves predictive fidelity rather than exactly ranking a heterogeneous task mixture.
Against diagonal selection, \method nevertheless improves test KL, perplexity, and Avg$_{12}$ in every block without downstream labels.

\textbf{The graph needs only modest calibration data.}
Rebuilding $\Hmat$ and reselecting from $16$--$128$ sequences changes WikiText-2 PPL by at most $0.87$ and Avg$_{12}$ by at most $1.50$ points across both models (\cref{fig:calibration-samples}).
Across $26$ interaction paths, every path places at least $98.5\%$ of held-out resamples within one standard error of minimum risk (\appref{app:robustness}), showing that both graph construction and held-out selection are stable.

\FloatBarrier
\section{What the Co-Pruning Edges Reveal}
\label{sec:exp-analysis}

Having established a consistent quality advantage, we ask what the co-pruning graph captures.
The evidence forms one chain: edges organize across modules, expose units missed by isolated saliency, reverse a conditional decision, and predict disproportionate damage under matched removal.
We therefore read effective interaction strength from changed selections and behavior, not from the numerical path coordinate alone.

\begin{figure}[H]
\vspace{-7pt}
\centering
\captionsetup{font=small,skip=2pt}
\includegraphics[width=.97\textwidth]{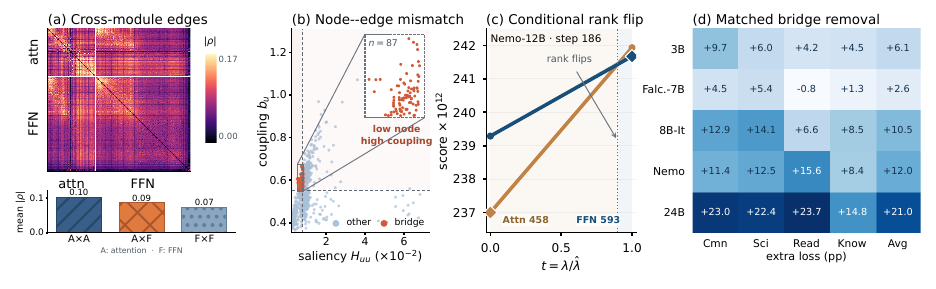}
\caption{\textbf{Co-pruning edges connect structure, decisions, and behavior.}
Full analyses: \appref{app:mechanism}.}
\label{fig:anatomy}
\vspace{-6pt}
\end{figure}

\textbf{Interactions span attention and FFN.}
For positive diagonals, normalize each off-diagonal as $|\gamma_{uv}|=|\Hmat_{uv}|/\sqrt{\Hmat_{uu}\Hmat_{vv}}$, excluding self-correlations and zero-diagonal rows.
Across $5$ language models, mean attention--FFN magnitude is $0.74$--$0.98$ times the within-module mean (\cref{fig:anatomy}a), with the same organization across $10$ LLMs and $7$ VLMs (\appref{app:mechanism}).
A shared budget changes only module totals; pair-specific cross-module structure becomes actionable when each candidate's cost depends on the units already removed.

\Needspace{5\baselineskip}
\textbf{Low saliency can coexist with strong coupling.}
Set $d_u=\Hmat_{uu}$ and $b_u=\sum_{v\ne u}|\Hmat_{uv}|$; bridge units lie below the $40$th percentile of $d_u$ but above the $60$th percentile of $b_u$.
This mask-independent intersection contains $87/768$ units on Llama-3.1-8B-Instruct and $3.5$--$14.5\%$ across $10$ LLM matrices (\cref{fig:anatomy}b), exposing a blind spot of diagonal saliency that \method addresses by conditioning on the current removal set.

\textbf{Accumulated edges reverse an actual decision.}
\Cref{fig:anatomy}c compares eligible units at one greedy prefix.
Attention unit $458$ has lower node cost than FFN unit $593$ but a larger accumulated edge term, so their scores cross at $0.90\hat\lambda$ and reverse the node-only ranking.
The diagonal mask removes units $458$ and $472$, whereas \method retains both and removes $593$.
The intervention records an actual conditional decision: \cref{tab:edge-decision} gives its components, \cref{cor:edge-reversal} generalizes the crossing, and \cref{prop:selection-stability} connects it to estimation error.

\textbf{The bridge population matters beyond its individual scores.}
On $5$ models, we force-remove bridges or equal-size low-coupling sets matched unit-for-unit on saliency, type, and parameter cost; one common greedy fill fixes total cost and the attention/FFN split.
Bridge removal degrades $19/20$ capability groups by up to $23.7$ points, including $14.8$--$23.7$ across all four Mistral-Small-24B groups (\cref{fig:anatomy}d).
Thus the matrix identifies low-saliency units with disproportionately damaging joint removals.
A stricter layer-matched intervention preserves this direction on both models, raising test KL and reducing Avg$_{12}$ (\cref{tab:depth-matched-bridge}).

\section{Related Work}
\label{sec:related}
\vspace{-2pt}

\textbf{Structured LLM pruning.}
Existing pruners score heads, channels, or blocks using gradient signals in LLM-Pruner~\citep{ma2023llmpruner}, activation statistics in FLAP and SlimLLM~\citep{an2024flap,tang2025slimllm}, or OBS reconstruction in SlimGPT~\citep{ling2024slimgpt}.
They may recover a fixed structure with LoRA~\citep{ma2023llmpruner,ling2024slimgpt}, interleave adaptation and pruning in LoRAPrune~\citep{zhang2024loraprune}, or optimize the retained basis and allocation~\citep{ashkboos2024slicegpt,li2025tyr,wang2026grasprune}.
CoCurve instead makes the compatibility of \emph{joint removals} the selection signal and lets attention heads and FFN groups compete within one conditional objective.

\textbf{Curvature and removal interactions.}
OBD established diagonal curvature saliency~\citep{lecun1989obd}, while OBS introduced curvature-aware compensation~\citep{hassibi1993obs}.
Later work models weight-space or structured interactions for joint selection and rearrangement~\citep{yu2022cbs,nonnenmacher2022sosp,kwon2022fastpruning}, or combines curvature with repair~\citep{vanderouderaa2024llmsurgeon}.
SAViT coordinates vision-Transformer structures with cross-component Taylor terms~\citep{zheng2022savit}.
Our distinction is an output-space, cross-layer attention--FFN graph estimated from finite forward ablations and reused for conditional pruning at multiple budgets.

\textbf{Forward perturbations and multimodal compression.}
Bonsai learns an additive predictor from forward removals~\citep{dery2024bonsai}; CoCurve's geometry additionally exposes pairwise effects without enumerating pairs.
Efficient-VLM work mostly prunes visual tokens~\citep{chen2024fastv,zhang2025sparsevlm,cao2024madtp,zou2026codecsight}, while parameter-structural methods use component-specific search and recovery~\citep{shi2023upop,liang2025efficientllava}.
We instead test one predictive-risk graph across both towers; \appref{app:related-ext} gives the full taxonomy.

\section{Conclusion}
\label{sec:conclusion}
\vspace{-2pt}

We presented \method, reframing structured pruning from ranking isolated units to selecting a jointly removable set.
CoCurve estimates set-dependent predictive risk as signed co-pruning edges from forward responses and uses these interactions to compile smaller dense models.
Across architectures and deployment conditions, this set-level view advances the quality--deployment frontier and traces the gain to joint-removal interactions rather than budget allocation alone.
More broadly, compression should ask not only which components matter in isolation, but which can be removed together while preserving the dependencies that sustain capability, offering a path toward smaller, more accessible foundation models.

\label{paper:main-text-end}

\clearpage
\bibliography{references}
\bibliographystyle{iclr2025_conference}

\newpage
\makeatletter
\newcommand{\AlphOverflow}[1]{%
\ifnum#1>26 A\@Alph{\numexpr#1-26\relax}\else\@Alph{#1}\fi} \makeatother
\appendix
\raggedbottom
\makeatletter
\setlength{\@fptop}{0pt}
\setlength{\@fpsep}{12pt plus 2pt minus 2pt}
\makeatother
\renewcommand{\thesection}{\AlphOverflow{\value{section}}}
\section*{Appendix Overview}
\label{app:overview}
\begingroup
\small
\setlength{\parskip}{1pt}

This appendix supplies the method specification, theoretical foundations, complete experimental record, mechanism analysis, and reproducibility material for CoCurve.
To keep that record navigable, it is organized into six parts that follow the paper's argument.

\paragraph{Part A: Method Details.}
This part turns the main method into a complete implementation path.
\begin{itemize}[leftmargin=1.6em,itemsep=0pt,topsep=1pt]
    \item \Appref{app:notation}: notation and the common attention--FFN unit inventory.
    \item \Appref{app:estimation}: single-unit perturbations and blockwise Fisher--Gram construction.
    \item \Appref{app:solver}: conditional shared-budget selection, complexity, and breakpoint-guided interaction selection.
\end{itemize}

\paragraph{Part B: Theoretical Foundations.}
This part connects the objective, estimator, and solver rather than treating them as independent results.
\begin{itemize}[leftmargin=1.6em,itemsep=0pt,topsep=1pt]
    \item \Appref{app:th-geom}: local teacher-KL geometry and its structured curvature.
    \item \Appref{app:th-est}: the exact Gram identity, finite-ablation surrogate, and approximation remainder.
    \item \Appref{app:solver-guarantees}: edge-induced ranking reversals, feasibility, the finite breakpoint path, and conditional-decision stability.
    \item \Appref{app:classical}: the diagonal Fisher limit, the geometry of normalized edges, and the co-pruning mass omitted by isolated scores.
\end{itemize}

\paragraph{Part C: Experimental Setup and Implementation Details.}
This part fixes the protocol before the complete results are shown.
\begin{itemize}[leftmargin=1.6em,itemsep=0pt,topsep=1pt]
    \item \Appref{app:datasets}: datasets, task groups, metrics, and aggregation.
    \item \Appref{app:models}: models, pruning budgets, and shared guards.
    \item \Appref{app:baselines}: baseline definitions and the matched comparison protocol.
    \item \Appref{app:evalproto}: language-model calibration, metrics, and reporting conventions.
    \item \Appref{app:vlm}: vision--language calibration and transfer protocol.
    \item \Appref{app:recipes}: common LoRA, quantization, and physical-slicing configurations.
\end{itemize}

\paragraph{Part D: Complete Compression Results.}
This part mirrors the order of the main empirical section.
\begin{itemize}[leftmargin=1.6em,itemsep=0pt,topsep=1pt]
    \item \Appref{app:lm-results}: full LLM quality and the measured deployment-frontier shift.
    \item \Appref{app:recovery-quant}: common-LoRA recovery and weight quantization.
    \item \Appref{app:efficiency}: physical inference efficiency and selection-cost accounting.
    \item \Appref{app:vlm-results}: full VLM transfer, diagonal control, and recovery.
    \item \Appref{app:behavioral}: behavioral fidelity and few-shot prompting.
    \item \Appref{app:deployment-context}: pruning a larger model versus starting from a native smaller checkpoint.
\end{itemize}

\paragraph{Part E: What the Co-Pruning Edges Reveal.}
This part expands the main paper's explanatory evidence as one mechanism chain.
\begin{itemize}[leftmargin=1.6em,itemsep=0pt,topsep=1pt]
    \item \Appref{app:mechanism}: graph structure, conditional decision reversals, bridge interventions, and learned allocation.
    \item \Appref{app:ablation-extra}: module eligibility and interaction-granularity controls.
    \item \Appref{app:surrogate}: same-budget discrimination by the quadratic surrogate.
    \item \Appref{app:robustness}: split-half, calibration-budget, and held-out-selection stability.
\end{itemize}

\paragraph{Part F: Discussion and Reproducibility.}
This part closes the paper with neighboring formulations, an evidence-matched scope statement, and release details.
\begin{itemize}[leftmargin=1.6em,itemsep=0pt,topsep=1pt]
    \item \Appref{app:connections}: structured pruning, curvature, and forward perturbations.
    \item \Appref{app:limitations}: demonstrated scope and natural extensions.
    \item \Appref{app:impact}: broader impact.
    \item \Appref{app:assets}: assets, licenses, and ethics.
    \item \Appref{app:repro}: reproducibility commitments.
\end{itemize}
\endgroup

\clearpage
\FloatBarrier
\section{Method Details}
\label{app:method}
This part gives a self-contained account of the method in implementation order.
Attention heads and feed-forward network (FFN) groups write to the model's shared residual stream~\citep{vaswani2017attention,elhage2021framework}, so pruning one can alter how another affects later predictions~\citep{mcgrath2023hydra,gong2026coablation}.
CoCurve therefore places all such units in one joint inventory: teacher Kullback--Leibler (KL) risk defines node weights for individual risk and signed co-pruning edges, single-unit ablations estimate both, conditional pruning constructs removal sets, and held-out risk selects the interaction level.
After fixing units and gate semantics in \appref{app:notation}, \appref{app:estimation} constructs the graph and \appref{app:solver} turns it into a physically pruned checkpoint.

\subsection{Notation and Structured Units}

\label{app:notation}
\Cref{tab:notation} collects only the objects that recur across the method and theory.
Local loop indices and one-off implementation quantities are defined where they are used.
\begin{table}[h]
\centering\footnotesize
\caption{\textbf{Core notation.}}
\label{tab:notation}
\setlength{\tabcolsep}{5pt}\renewcommand{\arraystretch}{1.03}
\begin{tabular}{@{}lp{.74\linewidth}@{}}
\toprule
\textbf{Symbol} & \textbf{Meaning} \\
\midrule
\rowcolor{headgray}\multicolumn{2}{c}{\textbf{Units and budget}}\\
$\Uset$ & Joint inventory $\{u_1,\ldots,u_M\}$ of removable attention units and FFN groups \\
$s,S$ & Gate vector $s=(s_u)$ (continuous near $\mathbf0$, binary at deployment) and set $S=\{u:s_u{=}1\}$ \\
$c_u,\rho$ & Removable parameter cost and target fraction relative to $\sum_uc_u$ \\
\midrule
\rowcolor{headgray}\multicolumn{2}{c}{\textbf{Predictive risk and graph}}\\
$f_0,p_0,p_s$ & Frozen dense model and dense/gated next-token distributions \\
$\mathcal C,\mathcal C'$ & Disjoint graph-estimation and held-out model-selection splits \\
$\mathcal R(s)$ & Expected token-level teacher KL, $\mathbb E_{x,t}[\KL(p_0\|p_s)]$ \\
$\Hmat^\star,\Hmat$ & Exact mask Hessian and finite-ablation Gram estimate \\
$\Hmat_{uu},\Hmat_{uv}$ & Diagonal node weight and signed co-pruning edge weight \\
$\delta z_u,\bar D_u$ & Single-unit next-token-logit response and its stacked Fisher feature \\
\midrule
\rowcolor{headgray}\multicolumn{2}{c}{\textbf{Conditional selection}}\\
$\lambda,\hat\lambda,\mathcal F$ & Interaction-path coordinate, its selected representative, and the candidate removal family \\
$\widehat{\mathcal R}_\lambda$ & Quadratic risk with off-diagonal edges weighted by $\lambda$ \\
$g_u^{(S)},\Delta\widehat{\mathcal R}_\lambda(u\mid S)$ & Incident-edge sum and incremental surrogate risk \\
\bottomrule
\end{tabular}
\end{table}
\paragraph{Mask semantics and physical pruning.}

\textbf{Attention units.}
For multi-head attention (MHA), one head is one unit; pruning head $i$ removes the corresponding rows of $W_Q,W_K,W_V$ and the input columns of $W_O$.
For grouped-query attention (GQA)~\citep{ainslie2023gqa} with $n_g=n_q/n_{kv}$ query heads per key/value (KV) head, the unit is the KV-aligned query-head group $\mathcal{I}_g=\{g n_g,\dots,(g{+}1)n_g{-}1\}$; pruning group $g$ removes the rows of $W_Q$ for $i\in\mathcal{I}_g$, the rows of $W_K,W_V$ for KV head $g$, and the matching input columns of $W_O$.
The head dimension $d_h$ is preserved during slicing; the attention module stores the reduced $n_q',n_{kv}'$ explicitly.

\textbf{FFN groups.}
The intermediate dimension is partitioned into $n_f$ contiguous groups.
The FFN output is linear in the post-activation intermediate vector and therefore decomposes exactly into the sum of these groups' contributions.
At runtime, pruning group $j$ zeros that post-activation slice before the down projection; physical pruning removes the corresponding rows of $W_{\text{gate}},W_{\text{up}}$ (or $W_1$) and columns of $W_{\text{down}}$ (or $W_2$).
The graph construction and solver require only this deployable unit registry, so changing $n_f$ changes the deployment granularity and graph size but not the objective or algorithm.

\textbf{Mask boundary and equivalence.}
A runtime mask zeros only a unit's residual-stream contribution and never alters normalization, rotary position embedding (RoPE), residual addition, causal/padding masks, GQA repeat rules, or dropout (the model is in eval mode).
Evaluation and timing require two distinct checks.
Zeroing a unit's weight slices is bit-identical to masking it at runtime, which we verify on a held-out batch, so a masked evaluation and a zeroed one are the same measurement.
Slicing the weight matrices instead is what the timings run on, and its correctness is checked by the realized parameter count of every timed model: on Llama-3.1-8B-Instruct at $\rho{=}20\%$ the sliced model holds $6.633$B parameters against the dense $8.030$B, and the logit equivalence of the slicing path is verified on a synthetic grouped-query Llama to $10^{-4}$.
Reported speedups therefore come from a smaller model, not from zeroed weights in a full-size one.

\subsection{Forward Estimation}

\label{app:estimation}

\cref{alg:estimate} turns each unit's next-token-logit response into one column of the Fisher feature matrix.
The crux (\cref{prop:gram}) is that gathering every ablation on the same teacher top-$r$ support $\mathcal{V}_r(x,t)$ places all responses in one coordinate system.
Their Gram matrix then yields every node weight and signed edge without model interventions beyond the $M$ single-unit sweeps.
Positive inner products mean reinforcing Fisher-weighted logit responses; negative ones mean partial cancellation.
We never materialize the full vocabulary-by-vocabulary Fisher matrix.

\begin{algorithm}[H]
\caption{\methodfull\ edge estimation (single-unit, top-$r$ Fisher Gram)}
\label{alg:estimate}
\begin{algorithmic}[1]
\REQUIRE frozen $f_0$, calibration $\mathcal{C}$, units $\Uset$, top-$r$, positions per sequence $P_{\rm tok}$
\STATE \textbf{Teacher pass:} for each $x\in\mathcal{C}$, select token positions $\mathcal{T}(x)$; store
  top-$r$ indices $\mathcal{V}_r(x,t)$, teacher logits $z_0$ on that support, and support-conditioned probabilities $p_0^{(r)}$
\FOR[$M$ calibration sweeps; embarrassingly parallel]{each unit $u\in\Uset$}
  \STATE run $f_{-u}$ (mask only $u$); gather $z_{-u}$ on the teacher support $\mathcal{V}_r(x,t)$
  \STATE $\delta z_u\leftarrow z_0-z_{-u}$;\quad
         $\widetilde{\delta z}_u\leftarrow \sqrt{p_0^{(r)}}\odot(\delta z_u-\mathbb{E}_{p_0^{(r)}}[\delta z_u])$
  \STATE append $\widetilde{\delta z}_u$ to feature store $\bar D_u$ (CPU/memmap)
\ENDFOR
\STATE $\Hmat\leftarrow \tfrac1P\,\bar D^\top\bar D$ (blockwise Gram); $\Hmat\leftarrow\tfrac12(\Hmat+\Hmat^\top)$
\STATE assert $\min_u\Hmat_{uu}\ge -\epsilon$ \COMMENT{non-negativity check (\cref{prop:gram})}
\RETURN $\Hmat$
\end{algorithmic}
\end{algorithm}

\textbf{Top-$r$ coordinate consistency.}
Using a per-unit top-$r$ would place different units in different coordinate systems and invalidate the Gram; the shared teacher support is essential.
Defaults: $r{=}256$, $128$ C4 sequences of length $2048$, and deterministic subsampling of $P_{\rm tok}{=}128$ positions per sequence, giving $P=16384$ stored positions.
Increasing $r$ from $256$ to $512$ leaves the selected mask and reported quality unchanged in our support-budget study (\appref{app:sensitivity}).

\subsection{The Greedy Co-Pruning Solver}

\label{app:solver}

\cref{alg:edge-taylor} gives the complete pipeline: ablate individual units, form the Gram edge matrix, enumerate candidate removal sets, select one by held-out risk, and physically remove its structures.
\begin{algorithm}[htbp]
\caption{\methodfull: estimation, held-out selection, and physical removal}
\label{alg:edge-taylor}
\begin{algorithmic}[1]
\REQUIRE frozen $f_0$, disjoint splits $\mathcal C,\mathcal C'$, units and costs, ratio $\rho$, guard
\STATE $\Hmat\leftarrow\textsc{Estimate}(f_0,\mathcal C,\Uset)$ using \cref{alg:estimate}
\STATE $\mathcal F\leftarrow\textsc{BreakpointPath}(\Hmat,\{c_u\},\rho,\text{guard})$, including the diagonal solution
\STATE evaluate $\mathcal R(S;\mathcal C')$ for each distinct $S\in\mathcal F$
\STATE $\widehat S\leftarrow\arg\min_{S\in\mathcal F}\mathcal R(S;\mathcal C')$
\STATE verify the realized budget; slice the weight matrices associated with $\widehat S$
\RETURN compressed model $f_{\widehat S}$ and its realized removal ratio
\end{algorithmic}
\end{algorithm}

\cref{alg:solve} gives the solver over the joint attention--FFN inventory.
For binary gates, $\widehat{\mathcal R}_\lambda(s)=\tfrac12\sum_u\Hmat_{uu}s_u+\lambda\sum_{u<v}\Hmat_{uv}s_us_v$; for removal set $S$, maintaining $g_u^{(S)}=\sum_{v\in S}\Hmat_{uv}$ makes each step $O(M)$.

\begin{algorithm}[H]
\caption{Cost-normalized greedy co-pruning on a fixed surrogate}
\label{alg:solve}
\begin{algorithmic}[1]
\REQUIRE $\Hmat$, costs $\{c_u\}$, ratio $\rho$, interaction weight $\lambda$, per-layer cap $\kappa$, protected layers $\mathcal{L}_p$
\STATE $S\leftarrow\emptyset$;\; $g\leftarrow\mathbf{0}$;\; $C_S\leftarrow0$;\; $C_\ell\leftarrow0$ for every layer
\STATE $C_\ell^{\rm tot}\leftarrow\sum_{u:\operatorname{layer}(u)=\ell}c_u$ for every layer
\WHILE{$C_S<\rho\sum_u c_u$}
  \STATE $\mathcal{A}\leftarrow\{u\notin S:\ \mathrm{layer}(u)\notin\mathcal{L}_p,\ C_{\mathrm{layer}(u)}+c_u\le\kappa\,C^{\rm tot}_{\mathrm{layer}(u)}\}$;\quad \textbf{if} $\mathcal{A}=\emptyset$ \textbf{then break}
  \STATE $u^\star\leftarrow\arg\min_{u\in\mathcal{A}}\big(\tfrac12\Hmat_{uu}+\lambda\,g_u\big)/c_u$
  \STATE $S\leftarrow S\cup\{u^\star\}$;\; $g\leftarrow g+\Hmat_{:,u^\star}$;\; $C_S\leftarrow C_S+c_{u^\star}$
  \STATE $C_{\operatorname{layer}(u^\star)}\leftarrow C_{\operatorname{layer}(u^\star)}+c_{u^\star}$
\ENDWHILE
\RETURN removal set $S$, $\rho_{\rm actual}=C_S/\sum_u c_u$, and target-reached flag
\end{algorithmic}
\end{algorithm}

\textbf{Feasibility guards.}
A per-layer cost cap and first/last-layer protection constrain which units are eligible without altering the conditional score.
Both follow one rule---cap $=1.5\rho$, with the first and last two layers protected---applied identically to every model, ratio, and method, with no per-model choice or reference ratio (\appref{app:protocol}).

\subsubsection{Complexity and cost}
\label{app:complexity}
Estimation is $M$ forward-only ablation passes over $|\mathcal{C}|$ sequences (no backward pass, Hessian-vector product, or pairwise ablation); the Gram is $O(M^2 P r)$ multiply-adds, accumulated blockwise.
The solver is $O(M|S|)$ for $|S|$ units selected for removal.
With FFN grouping ($M\!\sim\!10^3$), the Gram and solver remain tractable, and each model needs only one reusable curvature build; measured single-accelerator wall-clock is reported in \appref{app:efficiency}.

\textbf{Why a conditional greedy solver.}
The guarded, heterogeneous-cost objective in \cref{eq:objective} contains minimum knapsack as its diagonal special case and becomes a constrained binary quadratic program when edges are active, so exact subset search is exponential in general.
Maintaining the incident sum $g^{(S)}$ makes every greedy step score all eligible additions in $O(M)$ and chooses the exact lowest conditional marginal of the stated surrogate; beam search multiplies this cost by its width, while even one-swap enumeration introduces $O(|S|(M-|S|))$ neighboring sets per iterate.
The model-scale audit in \appref{app:surrogate} complements this complexity argument: among $26$ same-budget masks, including ten local swaps around the deployed set, the surrogate retains positive local rank correlation with measured risk rather than succeeding only by separating random masks.

\subsubsection{Selecting along the interaction path}
\label{app:lambda}
This section enumerates the finite interaction path, fixes the held-out protocol, and records the selected value for every model.

\textbf{Select masks, not surrogate magnitudes.}
The candidate removal sets, rather than the numerical scale of surrogate risk, determine deployment quality; a common positive rescaling of $\Hmat$ leaves every greedy comparison unchanged.
We therefore choose among them using the measured teacher KL of each set.

\textbf{Breakpoint-guided enumeration.}
\Cref{alg:lambda-path} gives the numerical path construction.
At a fixed greedy prefix, candidate scores are affine in $\lambda$.
Specifically, the cost-normalized incremental risk is $\mu_u+\lambda\nu_u$, with $\mu_u=\Hmat_{uu}/(2c_u)$ and $\nu_u=g_u^{(S)}/c_u$.
A competitor overtakes the current choice only at a crossing.
Intersecting these conditions over all steps identifies an interval on which the complete trace is fixed.

\begin{algorithm}[htbp]
\caption{Numerical breakpoint path for candidate generation}
\label{alg:lambda-path}
\begin{algorithmic}[1]
\REQUIRE $\Hmat$, costs, ratio, guard; range $[0,1]$; solve limit $N_{\max}=400000$
\STATE $\lambda\leftarrow0$; $\mathcal F\leftarrow\emptyset$; $n\leftarrow0$
\WHILE{$\lambda\le1$ and $n<N_{\max}$}
  \STATE run \cref{alg:solve} at $\lambda$, recording prefixes; break score ties by smallest unit index
  \STATE find the earliest finite crossing $\tau>\lambda+10^{-12}$ after which an eligible competitor becomes strictly better than the chosen unit
  \STATE record removal set $S$ and the next-crossing location; add $S$ to $\mathcal F$ if distinct
  \STATE $n\leftarrow n+1$; \textbf{if} $\tau\ge1$ or no finite crossing exists \textbf{then break}
  \STATE $\lambda\leftarrow\tau+10^{-9}$
\ENDWHILE
\STATE \textbf{if} the last interval ends below $1-10^{-9}$ \textbf{then raise} incomplete-path error
\RETURN distinct removal sets $\mathcal F$ and recorded intervals
\end{algorithmic}
\end{algorithm}

The implementation uses float64 scores, smallest-index tie-breaking, a $10^{-12}$ forward-crossing tolerance, and a $10^{-9}$ advance beyond each breakpoint.
It checks coverage through $1$ to within $10^{-9}$ and raises an error if the solve limit is exhausted.
Reported family sizes count the distinct interval sets returned under these numerical conventions; isolated exact-tie solutions and intervals narrower than the advance are not enumerated separately.

At $\rho=20\%$, the walk returns $288$ masks on Llama-3.2-3B, $303$ on Llama-3.1-8B-Instruct, $181$ on Falcon3-7B, and $400$ on Yi-1.5-9B.
A uniform $0.002$ grid reaches $199$ of the $288$ recorded 3B masks and misses the selected minimum.
Breakpoint guidance resolves these uneven intervals without choosing a global grid spacing.
Enumeration is CPU-only; each distinct candidate requires a held-out masked forward evaluation.

\textbf{The criterion is stable across held-out splits.}
$\lambda$ substantially changes the attention/FFN split---on Llama-3.1-8B-Instruct, attention falls from $38\%$ of removed units at $\lambda\!=\!0$ to $2\%$ at $\lambda\!=\!1$---so we test how strongly the choice depends on the held-out split $\mathcal C'$.
Over both context lengths and both corpora, at $\rho{=}20\%$ with everything else fixed:

\begin{center}
\footnotesize\setlength{\tabcolsep}{7pt}
\begin{tabular}{@{}ll r@{}}
\toprule
\textbf{$\mathcal{C}'$ length} & \textbf{$\mathcal{C}'$ corpus} & \textbf{$J$ vs. deployed}\\
\midrule
$1024$ & C4 & $0.967$\\
$1024$ & WikiText & $0.877$\\
$\mathbf{2048}$ & \textbf{C4} & \emph{deployed}\\
$2048$ & WikiText & $0.844$\\
\bottomrule
\end{tabular}

\end{center}

\noindent The selected removal sets agree with the deployed one at Jaccard $0.84$ to $0.97$ across the four splits.
The selected sets substantially overlap, although they are not identical.
We therefore use the deployment context length and C4 split consistently.

\Cref{fig:lambdapath} shows the path itself on four models.
\textbf{Path outcomes.}
Held-out risk selects an edge-aware removal set when interactions improve prediction and retains the diagonal endpoint otherwise.
In the broader ten-model diagnostic, held-out risk rejects the full-strength endpoint on every model, confirming that finite removal requires calibrated rather than undamped local edges.
Because the numerical coordinate is not an effect size, \cref{tab:ablation} reports path sizes and the measured consequences of removing or fully undamping the edges.
These describe held-out model selection rather than independent statistical significance.

\begin{figure}[H]
\vspace{-6pt}
\centering
\captionsetup{skip=1pt}
\includegraphics[width=.70\textwidth]{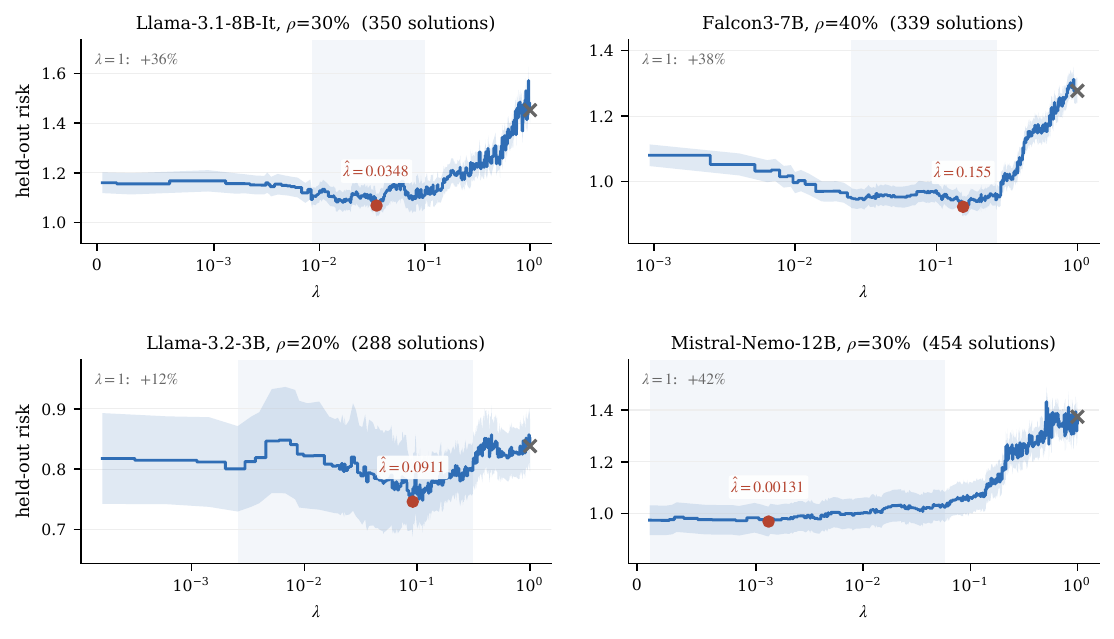}
\caption{\textbf{Held-out selection along the interaction path.}
Four held-out risk paths over $\lambda$; red dots select $\hat\lambda$, bands show one standard error, and crosses mark $\lambda{=}1$.
Every minimum is interior, while full-strength interaction is worse.}
\label{fig:lambdapath}
\vspace{-4pt}
\end{figure}

\clearpage
\FloatBarrier
\section{Theoretical Foundations and Guarantees}
\label{app:theory}

For $M$ removable units, let $s\in\mathbb R^M$ be the smooth removal-gate vector around the dense point $\mathbf0$, with $s_u{=}1$ denoting deployment removal.
For calibration sequence $x$ and token position $t$, let $z_0,z_s$ and $p_0,p_s$ be the corresponding dense and gated next-token logits and distributions, write $\Delta z_s=z_0-z_s$, and define the teacher Kullback--Leibler (KL) risk $\mathcal R(s)=\mathbb E_{x,t}[\KL(p_0\|p_s)]$.
This distribution-matching objective uses the dense model as an unlabeled teacher and vanishes at the dense point~\citep{hinton2015distilling}.

The theory follows the method's dependency chain: \appref{app:th-geom} derives its local geometry; \appref{app:th-est} constructs and diagnoses the finite-ablation Gram estimate; \appref{app:solver-guarantees} shows how its edges change conditional decisions and establishes path, budget, and perturbation properties; and \appref{app:classical} identifies the diagonal special case and its omitted interaction term.
This chain supports three design claims: individual pruning risk and co-pruning edges arise from the same second-order predictive objective; one forward ablation per unit yields all estimated edges; and their incident sum is the precise surrogate correction that can reverse an isolated ranking.
The held-out-risk and fixed-allocation controls in \cref{sec:exp-ablation}, the mechanism interventions in \cref{sec:exp-analysis}, and the same-budget analysis in \appref{app:surrogate} then connect these mathematical properties to finite-ratio model behavior.

\subsection{Geometry of the Pruning Risk}
\label{app:th-geom}
\paragraph{Why the local pruning risk begins at second order.}
\begin{lemma}[Vanishing zeroth- and first-order terms]\label{lem:vanish}
With $s$ ranging over a neighborhood of $\mathbf 0$, so that $\mathbf 0$ is interior and $\mathcal{R}$ is differentiable there, $\mathcal{R}(\mathbf 0)=0$ and $\nabla_s\mathcal{R}(\mathbf 0)=0$; the Taylor expansion at the full model begins at second order.
\end{lemma}
\begin{proof}
At $s=\mathbf 0$, $p_{\mathbf 0}=p_0$, so every term is $\KL(p_0\|p_0)=0$.
For the gradient, with $g(s)=\KL(p_0\|p_s)$ and $p_s$ a softmax, $\partial_{s_u}g=-\mathbb{E}_{p_0}[\partial_{s_u}z_s]+\mathbb{E}_{p_s}[\partial_{s_u}z_s]$; at $s=\mathbf 0$, $p_s=p_0$ and the two expectations cancel.
Averaging over positions preserves equality.
\end{proof}
The absence of a first-order term makes curvature the first informative carrier of both individual pruning risk and removal interactions.

\begin{proposition}[Curvature of teacher-KL pruning risk]\label{prop:fisher}
Let the masked logits be twice differentiable near $s=\mathbf0$, and let $G=\partial(z_0-z_s)/\partial s|_{\mathbf0}$.
For the empirical calibration distribution,
\begin{equation}
\Hmat^\star=\nabla_s^2\mathcal R(\mathbf0)
=\mathbb E_{x,t}\!\left[G^\top\big(\operatorname{diag}(p_0)-p_0p_0^\top\big)G\right].
\label{eq:app-mask-hessian}
\end{equation}
Here $F=\operatorname{diag}(p_0)-p_0p_0^\top$ at each position, and therefore $\Hmat^\star$ is positive semidefinite (PSD); moreover, $\KL(p_0\|p_s)=\tfrac12\Delta z_s^\top F\Delta z_s+O(\|\Delta z_s\|^3)$ locally.
\end{proposition}
\begin{proof}
The gradient of teacher KL with respect to student logits is $p_s-p_0$, which vanishes at $s=\mathbf0$, and its logit Hessian there is $F$.
In the second-order chain rule, the second derivatives of the logit map multiply this zero gradient.
The remaining term is $(-G)^\top F(-G)=G^\top FG$, the standard Fisher pullback through the gate-to-logit Jacobian~\citep{amari1998natural}.
Averaging over the finite calibration distribution gives the identity.
$F$ is the covariance matrix of a categorical one-hot variable and is positive semidefinite, so every pullback $G^\top FG$ and their expectation are positive semidefinite.
Expanding the smooth log-partition function at $z_0$ gives the stated logit-space remainder.
\end{proof}
For binary gates, expanding $\tfrac12s^\top\Hmat^\star s$ yields diagonal node weights and one signed co-pruning edge per unit pair.
CoCurve estimates these derivative columns from single-unit next-token-logit responses on a shared teacher support.

\subsection{Estimator Identities and Approximation Diagnostics}

\label{app:th-est}
\paragraph{The full graph from $M$ ablations.}
\begin{proposition}[Gram equivalence]\label{prop:gram}
Let $\delta z_u=z_0-z_{-u}$ be the finite next-token-logit response to pruning unit $u$.
For any normalized distribution $p$ on a common support, let $F_p=\operatorname{diag}(p)-pp^\top$ and $\widetilde{\delta z}_u=\sqrt{p}\odot(\delta z_u-\mathbb{E}_{p}[\delta z_u])$.
Then $\delta z_u^\top F_p\delta z_v=\widetilde{\delta z}_u^\top\widetilde{\delta z}_v$.
Consequently, stacking these features over $P$ sampled positions as the columns $\bar D_u$ of $\bar D$ gives $\Hmat=\tfrac1P\bar D^\top\bar D\succeq0$ and $\Hmat_{uu}=\tfrac1P\|\bar D_u\|^2\ge0$, recoverable from $M$ single-unit ablations without pairwise model ablations.
\end{proposition}
\begin{proof}
$a^\top F_pb=\sum_y p_ya_yb_y-(\sum_y p_ya_y)(\sum_y p_yb_y)=\sum_y p_y(a_y-\bar a)(b_y-\bar b) =(\sqrt p\odot(a-\bar a))^\top(\sqrt p\odot(b-\bar b))$, where $\bar a=\mathbb E_p[a]$ and likewise for $\bar b$.
Apply the identity with $a=\delta z_u,b=\delta z_v$ separately at every sampled position, then stack the resulting features over positions to obtain the Gram, which is PSD with non-negative diagonal.
\end{proof}
Thus $O(M)$ ablations expose $O(M^2)$ edges through inner products, while guaranteeing symmetry, positive semidefiniteness, and a non-negative diagonal in exact arithmetic.
\Cref{alg:estimate} forms this Gram matrix blockwise, explicitly symmetrizes roundoff, and checks its diagonal; in floating point, positive semidefiniteness holds up to numerical error.

\paragraph{Exact curvature vs.\ the estimator we use.}
Two steps are exact; estimating the curvature introduces the approximation.
\begin{remark}[What is exact, what is a surrogate]\label{rem:exact}
The local expansion $\KL(p_0\|p_s)=\tfrac12\Delta z_s^\top F\Delta z_s+o(\|\Delta z_s\|^2)$ (\cref{prop:fisher}) and the Gram identity $\delta z_u^\top F\,\delta z_v=\widetilde{\delta z}_u^\top\widetilde{\delta z}_v$ (\cref{prop:gram}) are exact.
We estimate rather than compute the curvature matrix itself.
The exact Hessian, which here coincides with the Gauss--Newton matrix because the logit-space gradient vanishes at $s=\mathbf 0$ (\cref{lem:vanish}), is $\Hmat^\star=\mathbb{E}[G^\top F G]$, where $G=\partial\Delta z_s/\partial s|_{s=\mathbf0}=-\partial z_s/\partial s|_{s=\mathbf0}$.
We replace each infinitesimal column $G_u$ by the finite single-unit ablation $\delta z_u=z_0-z_{-u}$ and compute the implemented Fisher inner products on the teacher's top-$r$ support, renormalized to sum to one.
The implementation therefore forms a finite-difference, renormalized top-$r$ Fisher--Gram surrogate for $\Hmat^\star$; it equals $\Hmat^\star$ when the logit map is linear in the removal gates and the full vocabulary is retained.
At $r{=}256$, the median retained teacher mass is $0.993$ on Llama-3.2-3B and $0.992$ on Mistral-Small-24B, over $16{,}384$ sampled calibration positions per model.
\Cref{tab:sensitivity} measures how varying the support affects selection and downstream quality.
Its conditional ranking signal is validated directly in \appref{app:surrogate} and \cref{sec:exp-analysis}.
\end{remark}

\paragraph{From local geometry to finite pruning.}
The exact result is local, whereas the deployed masks remove finite units.
Three approximation steps bridge the two: the quadratic model is evaluated away from the origin, derivative columns are replaced by finite ablations, and the Fisher is restricted to a renormalized vocabulary support.
The proposition below controls the first step locally; the following measurements test finite-ablation fidelity and joint-logit additivity, while \cref{tab:sensitivity} tests support sensitivity.
Together, these results separate exact local guarantees from measured finite-removal fidelity.

\begin{proposition}[Third-order remainder]\label{prop:remainder}
Let $\Hmat^\star=\mathbb{E}[G^\top F G]$ with $G=\partial(z_0-z_s)/\partial s|_{s=\mathbf 0}$ be the exact Gauss--Newton curvature of $\mathcal R$.
If the logit map is twice continuously differentiable in the mask variables on a neighborhood of $s=\mathbf 0$, then for every sufficiently small $\eta>0$ whose closed $\ell_1$ ball lies in that neighborhood, there is a constant $K_{\mathcal R}$, independent of $s$, with $\big|\mathcal{R}(s)-\tfrac12 s^\top\Hmat^\star s\big|\le K_{\mathcal R}\|s\|_1^3$ for all $\|s\|_1\le\eta$.
\end{proposition}
\begin{proof}
Fix one calibration position and write $f(s)=\Delta z_s$, so $f(\mathbf0)=0$ and $Df(\mathbf0)=G$.
Bounded second derivatives of the logit map give, uniformly on a sufficiently small $\ell_1$ ball,
$f(s)=Gs+e(s)$ with $\|e(s)\|_2\le L\|s\|_1^2$ and $\|f(s)\|_2+\|Gs\|_2\le K\|s\|_1$ for finite $L,K$.
The softmax log-partition is analytic; its third derivative is bounded on the compact local image of this ball.
Taylor's theorem therefore gives
$\KL(p_0\|p_s)=\tfrac12 f(s)^\top Ff(s)+r(s)$ with $|r(s)|\le B\|f(s)\|_2^3$.
Moreover,
\begin{equation}
\left|f(s)^\top Ff(s)-(Gs)^\top F(Gs)\right|
\le \|F\|_2\,\|e(s)\|_2\bigl(\|f(s)\|_2+\|Gs\|_2\bigr)
\le \|F\|_2LK\|s\|_1^3.
\label{eq:app-remainder-bound}
\end{equation}
Combining the two bounds yields the claim at this position.
The empirical calibration distribution is finite, so taking the maximum of its local constants and then averaging establishes one finite constant $K_{\mathcal R}$ for $\mathcal R$.
\end{proof}
\noindent This result explains why a quadratic graph is the right local object; \appref{app:surrogate} asks whether its rankings remain useful for finite removals.

\textbf{(i) Finite ablation in place of the derivative column.}
We replace $G_u$ by $\delta z_u=z_0-z_{-u}$, a secant over the whole mask interval rather than a tangent.
Scaling a unit's residual contribution by $1-\epsilon$ and forming $\delta z_u^{(\epsilon)}=(z_0-z_{u,\epsilon})/\epsilon$ for $\epsilon\in\{1,\tfrac12,\tfrac14,\tfrac18\}$ lets us compare the two directly.
On a sample of $15$--$18$ units per model, stratified by depth, unit type, and saliency, the hard-ablation feature agrees with the finest small-step feature at a median cosine of $0.949$--$0.981$ across four models after the same probability weighting and mean centering as in \cref{prop:gram}; the median norm ratio is $1.03$--$1.07$.
Thus the logit-response directions and scales used by the estimator remain close to the corresponding small-step proxies.
We use fp32 because the small-$\epsilon$ difference reaches the bf16 logit noise floor.

\textbf{(ii) Joint-logit additivity diagnostic.}
The surrogate uses $\Delta z_S\approx\sum_{u\in S}\delta z_u$.
Let $\Delta z_S=z_0-z_{-S}$ be the true joint displacement and let $\widetilde{\Delta z}_S$ denote its probability-weighted, mean-centered feature, formed as in \cref{prop:gram}.
We compare the predicted risk $\tfrac12s^\top\Hmat s$ with $\tfrac12\mathbb E_{x,t}\|\widetilde{\Delta z}_S\|^2$; this isolates joint-logit additivity from the other approximations.
\par\Needspace{5\baselineskip}\noindent
Across ten nested diagonal-greedy prefixes on Llama-3.2-3B, from two-unit removal through $\rho=25\%$, the predicted-to-measured quadratic-risk ratio remains within $0.67$--$1.18$ despite extending the local construction to finite removal sets.
Both diagnostics rebuild fp32 features on eight sequences of $512$ tokens, with held-out text used for the joint-additivity test; the deployed matrix uses bf16 forwards over $128$ sequences of $2048$.

A uniform positive rescaling of $\Hmat$ leaves selection unchanged (\appref{app:solver-guarantees}), so conditional ranking matters more than absolute surrogate scale; held-out predictive risk then chooses among the actual masks, and \appref{app:surrogate} evaluates their ranking fidelity.

\subsection{Conditional Pruning and Solver Guarantees}
\label{app:solver-guarantees}

We next establish a non-negative surrogate, identify when an edge changes a decision, and then derive path, budget, and perturbation properties in algorithmic order.
Each unit has removable parameter cost $c_u>0$, and $S$ denotes the current removal set.

\paragraph{Positive-semidefinite path.}
If $\Hmat\succeq0$, then $\Hmat_\lambda=(1-\lambda)\operatorname{diag}(\Hmat)+\lambda\Hmat\succeq0$ for every $\lambda\in[0,1]$, and hence $\widehat{\mathcal R}_\lambda(s)=\tfrac12s^\top\Hmat_\lambda s\ge0$ for every real gate vector $s$.
This follows because $\operatorname{diag}(\Hmat)\succeq0$ and the positive-semidefinite cone is convex.

\paragraph{Conditional marginal and solver cost.}
Expanding the set objective gives $\Delta\widehat{\mathcal R}_\lambda(u\mid S)=\tfrac12\Hmat_{uu}+\lambda g_u^{(S)}$, where $g_u^{(S)}=\sum_{v\in S}\Hmat_{uv}$.
The vector $g^{(S)}$ stores these incident-edge sums; updating it by $g\leftarrow g+\Hmat_{:,u^\star}$ costs $O(M)$ per choice and $O(M|S|)$ for a complete fixed-$\lambda$ solve.

\begin{corollary}[Edge-induced ranking reversal]\label{cor:edge-reversal}
Fix a greedy prefix $S$ and two eligible units $u,v$ with positive costs.
Let $\mu_w=\Hmat_{ww}/(2c_w)$ and $\nu_w=g_w^{(S)}/c_w$ for $w\in\{u,v\}$.
Their cost-normalized conditional scores are $q_\lambda(w\mid S)=\mu_w+\lambda\nu_w$.
If $\mu_u<\mu_v$ but $\nu_u>\nu_v$, their conditional scores cross once at $\lambda_{uv}=(\mu_v-\mu_u)/(\nu_u-\nu_v)$.
Whenever $\lambda_{uv}\in(0,1)$, the diagonal score prefers $u$, whereas every $\lambda\in(\lambda_{uv},1]$ prefers $v$ at the same prefix.
\end{corollary}
\begin{proof}
The score difference is $q_\lambda(u\mid S)-q_\lambda(v\mid S)=(\mu_u-\mu_v)+\lambda(\nu_u-\nu_v)$.
Under the stated inequalities it is strictly increasing from a negative value and has the unique zero $\lambda_{uv}$, which gives the two orderings.
\end{proof}
This is the precise sense in which a co-pruning edge changes a pruning decision: individual saliency fixes the intercept, while accumulated incident edges set the slope.

\paragraph{Scale invariance.}
For any $a>0$, replacing $\Hmat$ by $a\Hmat$ multiplies every cost-normalized conditional score by the same positive constant.
Hence the complete trace and stopping point are unchanged under the same guard and tie-breaking.

\begin{proposition}[Finite piecewise-constant selection path]\label{prop:finite-path}
For a finite unit inventory, fixed guard and budget, and deterministic tie-breaking, the complete greedy trace is piecewise constant in $\lambda\in[0,1]$, with changes possible only at finitely many pairwise score crossings.
\end{proposition}
\begin{proof}
At any fixed prefix, every feasible candidate score is affine in $\lambda$; two such scores can exchange order only at a crossing, unless they coincide identically, in which case deterministic tie-breaking fixes their order.
There are finitely many prefixes and candidate pairs.
Removing their finitely many crossing values partitions $[0,1]$ into open intervals on which the first choice is fixed; induction fixes every subsequent prefix, the stopping point, and hence the complete trace on each interval.
At the finitely many boundary points, deterministic tie-breaking defines a unique trace, possibly as a singleton piece.
\end{proof}
The mathematical path therefore induces a finite family of distinct removal sets that can be evaluated by measured teacher KL; \appref{app:lambda} specifies its numerical enumeration.

\paragraph{Budget control.}
Let $C=\sum_uc_u$ and $c_{\max}=\max_uc_u$.
If the target is reached, the stopping rule yields $0\le\rho_{\rm actual}-\rho<c_{\max}/C$, equivalently $\rho C\le\sum_{u\in S}c_u<\rho C+c_{\max}$.
To see this, the prefix immediately before stopping has cost below $\rho C$; the final unit adds at most $c_{\max}$, so the stopped cost is below $\rho C+c_{\max}$.

\paragraph{Role of cost normalization.}
CoCurve greedily minimizes $\Delta\widehat{\mathcal R}_\lambda(u\mid S)/c_u$ over eligible candidates.
Positive semidefiniteness guarantees non-negative total quadratic risk, not non-negative discrete marginals; the density rule is an efficient solver rather than a global-optimality result.

\begin{proposition}[Stability of a conditional decision]\label{prop:selection-stability}
Suppose $\max_{u,v}|\Hmat_{uv}-\widetilde{\Hmat}_{uv}|\le\varepsilon$.
Fix the same prefix $S$, interaction weight $\lambda\in[0,1]$, positive costs, and feasible pool.
Writing $q_{\Hmat}(u\mid S)=(\Hmat_{uu}/2+\lambda\sum_{v\in S}\Hmat_{uv})/c_u$, we have
\begin{equation}
|q_{\Hmat}(u\mid S)-q_{\widetilde{\Hmat}}(u\mid S)|
\le \frac{\varepsilon(1/2+\lambda|S|)}{c_u}.
\label{eq:app-score-perturbation}
\end{equation}
If $u^\star$ uniquely minimizes $q_{\Hmat}$, its choice is preserved whenever, for every other feasible $v$,
\begin{equation}
q_{\Hmat}(v\mid S)-q_{\Hmat}(u^\star\mid S)>
\varepsilon(1/2+\lambda|S|)
\left(\frac1{c_v}+\frac1{c_{u^\star}}\right).
\label{eq:app-stability-condition}
\end{equation}
\end{proposition}
\begin{proof}
The first inequality follows by bounding the diagonal error and the $|S|$ accumulated edge errors separately.
Applying this bound to both candidates leaves their score difference strictly positive under the stated condition.
If the condition holds at every prefix of a fixed-$\lambda$ run, induction gives the same complete trace and stopping point.
\end{proof}
The result connects estimation accuracy to the margin between actual candidates: stability depends on accumulated edge error relative to the decision gap, rather than on matrix correlation alone.
The guarantee applies along each fixed-$\lambda$ trace; held-out predictive risk then selects among those traces.
This stability guarantee concerns perturbations of the surrogate Gram matrix, rather than agreement between surrogate decisions and the true finite-removal KL risk, which may additionally involve higher-order interactions.

\subsection{What Co-Pruning Edges Add Beyond Diagonal Saliency}
\label{app:classical}
\begin{remark}[Diagonal \method recovers structured Optimal Brain Damage (OBD)]\label{rem:obd}
For binary removal masks, dropping off-diagonals from the implemented surrogate reduces it to $\tfrac12\sum_u\Hmat_{uu}s_u$, with $\Hmat_{uu}=\mathbb{E}[\delta z_u^\top F^{(r)}\delta z_u]\ge0$ and $F^{(r)}=\operatorname{diag}(p_0^{(r)})-p_0^{(r)}p_0^{(r)\top}$ on the retained teacher support.
This is a structured OBD saliency~\citep{lecun1989obd} with units in place of weights and a self-distillation KL Fisher in place of the training-loss Hessian, consistent with prior Fisher, joint-removal, and interaction-aware pruning views~\citep{theis2018faster,yu2022cbs,kwon2022fastpruning,nonnenmacher2022sosp}.
Here the first-order term vanishes exactly at $s=\mathbf 0$ (\cref{lem:vanish}), rather than through an assumption of convergence.
\method adds signed cross-unit curvature $\Hmat_{uv}$, including attention--feed-forward network (FFN) edges, which the diagonal cannot express.
\end{remark}
\noindent\emph{Significance.}
This places \method on a continuum with classical second-order pruning: $\lambda{=}0$ is the structured KL--Fisher analogue of OBD, whereas $\lambda{=}1$ uses the full estimated Gram matrix.
Comparing the diagonal endpoint with the selected interaction path in \cref{tab:ablation} therefore isolates what the co-pruning edges add while holding the unit features, guard, and budget fixed.

\paragraph{The exact term omitted by a diagonal method.}
We can quantify exactly what is lost when the same second-order surrogate $\tfrac12 s^\top\Hmat s$ is diagonalized (\cref{rem:exact} separates that surrogate from the true risk).
The statement uses the implemented Gram matrix $\Hmat$; the same algebra applies to the exact curvature $\Hmat^\star$.
\begin{proposition}[The diagonal omission is the off-diagonal co-pruning mass]\label{prop:offdiag}
For any removal set $S$ with indicator mask $s_S\in\{0,1\}^{M}$ ($s_{S,u}{=}1$ iff $u\in S$), the second-order co-pruning risk and its diagonal (OBD) approximation differ by exactly the off-diagonal mass of $\Hmat$ on $S$:
\begin{equation}
\Delta_{\rm off}(S)\;:=\;\tfrac12 s_S^\top\Hmat\,s_S-\tfrac12\sum_{u\in S}\Hmat_{uu}
\;=\;\tfrac12\!\!\sum_{\substack{u,v\in S\\ u\neq v}}\!\!\Hmat_{uv}
\;=\;\tfrac12\!\!\sum_{\substack{u,v\in S\\ u\neq v}}\!\!\gamma_{uv}\sqrt{\Hmat_{uu}\Hmat_{vv}},
\label{eq:app-offdiagonal-mass}
\end{equation}
with $\gamma_{uv}=\Hmat_{uv}/\sqrt{\Hmat_{uu}\Hmat_{vv}}$ whenever the denominator is positive.
By \cref{prop:gram}, $\gamma_{uv}$ is exactly the cosine similarity between the stacked Fisher features $\bar D_u$ and $\bar D_v$.
If either diagonal is zero, positive semidefiniteness forces $\Hmat_{uv}=0$; we set $\gamma_{uv}=0$ for that zero-weight pair.
Let $W_S=\sum_{\substack{u,v\in S\\u\ne v}}\sqrt{\Hmat_{uu}\Hmat_{vv}}$.
When $W_S>0$, define $\bar\gamma_S=W_S^{-1}\sum_{\substack{u,v\in S\\u\ne v}}\gamma_{uv}\sqrt{\Hmat_{uu}\Hmat_{vv}}$; set $\bar\gamma_S=0$ when $W_S=0$.
Then $\Delta_{\rm off}(S)=\tfrac12\,\bar\gamma_S\big[(\sum_{u\in S}\!\sqrt{\Hmat_{uu}})^2-\sum_{u\in S}\Hmat_{uu}\big]$.
For the interaction-weighted surrogate $\Hmat_\lambda$, the corresponding diagonal-omission term is $\lambda\Delta_{\rm off}(S)$.
\end{proposition}
\begin{proof}
Expanding the symmetric quadratic gives
$\tfrac12s_S^\top\Hmat s_S=\tfrac12\sum_{u\in S}\Hmat_{uu}+\tfrac12\sum_{\substack{u,v\in S\\u\ne v}}\Hmat_{uv}$, which proves the first identity.
For positive diagonal pairs, substitute $\Hmat_{uv}=\gamma_{uv}\sqrt{\Hmat_{uu}\Hmat_{vv}}$.
For a positive-semidefinite matrix, $|\Hmat_{uv}|^2\le\Hmat_{uu}\Hmat_{vv}$, so a zero diagonal forces the corresponding edge to vanish and the stated convention preserves the identity.
Finally, when $W_S>0$, factor out the weighted mean and use $W_S=(\sum_{u\in S}\sqrt{\Hmat_{uu}})^2-\sum_{u\in S}\Hmat_{uu}$.
When $W_S=0$, every off-diagonal term vanishes by the same positive-semidefinite inequality, so both sides are zero under the stated convention.
\end{proof}
\noindent\emph{Significance.}
Within the quadratic surrogate, the omitted cross terms are themselves second order: a positive edge raises joint risk above the isolated sum, whereas a negative edge lowers it through Fisher-weighted cancellation; their aggregate magnitude is governed by $\bar\gamma_S$ and the removal set's saliency mass.
Its discrete increment---the incident-edge sum---is precisely the additional signal that can re-rank candidates relative to the diagonal rule.
The controlled comparisons in \cref{tab:ablation} then test whether using that signal improves the selected model while holding the unit features, guard, and budget fixed.

\paragraph{What the theory establishes.}
Teacher KL supplies a vanishing first-order term, exact gate curvature, and a cubic local remainder (\cref{lem:vanish}, \cref{prop:fisher}, and \cref{prop:remainder}); single-unit perturbations expose a PSD graph at linear intervention cost (\cref{prop:gram}); and the graph yields an explicit ranking-reversal condition, efficient conditional marginals, a finite interaction path, budget control, and decision stability (\cref{cor:edge-reversal}, \cref{prop:finite-path}, and \cref{prop:selection-stability}).
The approximation diagnostics in \appref{app:th-est}, same-budget ranking test in \appref{app:surrogate}, and support sensitivity in \appref{app:sensitivity} connect this local foundation to deployable masks.
Together with \cref{prop:offdiag}, the chain explains both what isolated saliency omits and how CoCurve turns that missing information into conditional pruning; the matched controls in \cref{sec:exp-ablation} test whether that theoretically identified signal improves finite-ratio quality under the same units, budget, and guard.

\FloatBarrier

\clearpage
\FloatBarrier
\section{Experimental Setup}
\suppressfloats[t]
\label{app:setup}
\subsection{Datasets}

\label{app:datasets}
\cref{tab:datasets} lists all evaluation benchmarks with task type, metric, and split.
In compact tables, A-e/A-c abbreviate ARC-easy/challenge and R-m/R-h abbreviate RACE-middle/high; HS, WG, CSQA, and OBQA denote HellaSwag, WinoGrande, CommonsenseQA, and OpenBookQA.
Calibration is $128$ C4 sequences (length $2048$); no task labels are used for pruning.
The twelve-task panel is the complete capability evaluation suite.
Pruning masks are selected without downstream labels or benchmark scores.

\begin{table}[htbp]
\centering
\caption{\textbf{Evaluation suite.}
Every multiple-choice task is scored on its full split through one shared harness; $n$ is the number of items actually scored, and for the perplexity corpora the number of target tokens under the Llama-3.1-8B-Instruct tokenizer (the token count, not the corpus, varies across tokenizers).
Multiple choice is length-normalized log-likelihood over the candidate answers, zero-shot.
The same taxonomy organizes the main and appendix downstream tables.}
\label{tab:datasets}
\footnotesize
\setlength{\tabcolsep}{5pt}
\begin{tabular}{lllr}
\toprule
\textbf{Group} & \textbf{Benchmark} & \textbf{Split} & \textbf{$n$}\\
\midrule
Language modeling & WikiText-2~\citep{merity2017wikitext} & test & $288{,}936$\\
 & PTB~\citep{marcus1993ptb} & test & $98{,}914$\\
 & C4~\citep{raffel2020c4} & validation, $500$ docs & $229{,}527$\\
\midrule
Commonsense & PIQA~\citep{bisk2020piqa} & validation & $1{,}838$\\
 & HellaSwag~\citep{zellers2019hellaswag} & validation & $10{,}042$\\
 & WinoGrande~\citep{sakaguchi2021winogrande} & validation & $1{,}267$\\
 & CommonsenseQA~\citep{talmor2019commonsenseqa} & validation & $1{,}221$\\
\midrule
Science QA & ARC-easy~\citep{clark2018arc} & test & $2{,}376$\\
 & ARC-challenge~\citep{clark2018arc} & test & $1{,}172$\\
 & OpenBookQA~\citep{mihaylov2018openbookqa} & test & $500$\\
\midrule
Reading & BoolQ~\citep{clark2019boolq} & validation & $3{,}270$\\
 & RACE~\citep{lai2017race} & middle, test & $1{,}436$\\
 & RACE-high~\citep{lai2017race} & high, test & $3{,}498$\\
 & QuAIL~\citep{rogers2020quail} & validation & $2{,}164$\\
\midrule
Knowledge & MMLU~\citep{hendrycks2021mmlu} & test & $14{,}042$\\
\bottomrule
\end{tabular}
\end{table}

Avg$_{12}$ is the equal-task macro average, $\bar a=12^{-1}\sum_t a_t$.
Its reported standard error is the task-stratified binomial estimate $100\,12^{-1}\sqrt{\sum_t \hat p_t(1-\hat p_t)/n_t}$ in percentage points, using the item count $n_t$ in \cref{tab:datasets}.
It quantifies finite-benchmark sampling uncertainty; the pruning rules and full-split evaluations themselves are deterministic, so it is not presented as a training-seed deviation.

\subsection{Models and Experimental Protocol}
\label{app:models}

\label{app:protocol}
All models use GQA (attention units are KV-aligned query-head groups); FFN groups are $n_f{=}16$ per layer; cost is removable parameter count.
The unit definition is a property of the comparison rather than of any method: every pruning method here, ours included, chooses among the same units under the same budget.

\textbf{Guard and realized budget.}
Every shared-protocol method protects the first and last two layers and limits each remaining layer to a removed-cost fraction $\kappa=1.5\rho$.
With equal removable cost per layer, continuous capacity is $(1-4/L)\min(\kappa,1)$; for the reported $10$--$50\%$ pruning-only sweep it exceeds $\rho$ for every evaluated depth.
This is a capacity check, not a guarantee of discrete feasibility.
For every method using the shared budget solver, we verify the realized mask and report only runs within $0.2$ percentage points of the target; baselines with a native discrete allocation retain that allocation.
If the target is reached, overshoot is bounded by one unit, as established in \appref{app:solver-guarantees}.

\textbf{Guard sensitivity.}
At $20\%$ on Llama-3.1-8B-Instruct, removing the guard and recalibrating selects a nearby point on the same interaction path.
Without the guard, FLAP and Wanda-sp delete the first seven layers and reach WikiText perplexities of $41{,}189$ and $30{,}742$.
The unguarded CoCurve mask shows the same early-layer concentration tendency.
Thus the guard prevents extreme early-layer concentration under a common method-independent rule.
All main comparisons use the same fixed rule.

\Cref{tab:hparam} gives the structure of the models in the pruning sweep.
The $10$-model mechanism panel in \appref{app:analysis} additionally includes Mistral-7B~\citep{jiang2023mistral}, Qwen2.5-7B/14B~\citep{yang2024qwen25}, Yi-1.5-9B~\citep{young2024yi}, Falcon3-10B~\citep{falcon3}, and Llama-2-13B~\citep{touvron2023llama2}.
The $7$-VLM analysis additionally includes Qwen2-VL-2B~\citep{wang2024qwen2vl}, Qwen2.5-VL-3B~\citep{bai2025qwen25vl}, LLaVA-NeXT-7B~\citep{liu2024llavanext}, and LLaVA-1.5-7B~\citep{liu2024llava15}.
The 70B checkpoint is loaded in NF4 for all pruners in that scale study; the remaining language-model experiments use bf16 unless a quantization condition is stated explicitly.

\begin{table}[H]
\centering
\caption{\textbf{Unit inventory} for the models evaluated with the shared pruning harness.
The guard is one rule for every row and every ratio; nothing here is configured per model.}
\label{tab:hparam}
\footnotesize
\setlength{\tabcolsep}{6pt}
\begin{tabular}{l cccc}
\toprule
\textbf{Model} & \textbf{Family} & \textbf{Layers} & \textbf{Units $M$} & \textbf{Attn/FFN}\\
\midrule
Llama-3.1-8B-Instruct & Llama-3 & 32 & 768 & 256/512\\
Llama-3.2-3B & Llama-3 & 28 & 672 & 224/448\\
Falcon3-7B & Falcon3 & 28 & 560 & 112/448\\
Mistral-Nemo-12B & Mistral & 40 & 960 & 320/640\\
Mistral-Small-24B & Mistral & 40 & 960 & 320/640\\
Llama-3.1-70B-Instruct & Llama-3 & 80 & 1920 & 640/1280\\
\bottomrule
\end{tabular}
\end{table}
\subsection{Baseline Methods}
\label{app:baselines}

\textbf{Comparison design.}
The main study compares selection signals under a common unit inventory, budget, guard, and evaluation harness, with weights fixed.
\Cref{tab:baseline-config} records each adaptation, while the common-recipe experiment separately tests the structures after lightweight recovery.

\begin{table}[htbp]
\centering\fontsize{6.30}{6.90}\selectfont
\caption{\textbf{Pruning-only baseline configurations.}
Mechanism and matched $20\%$ outcome on Llama-3.1-8B-Instruct.}
\label{tab:baseline-config}
\begin{adjustbox}{width=.84\textwidth}
\renewcommand{\arraystretch}{1.02}\setlength{\tabcolsep}{2.15pt}
\begin{tabular}{@{}l l >{\raggedright\arraybackslash}p{.205\linewidth}
>{\centering\arraybackslash}p{.075\linewidth}rrrrr@{}}
\toprule
\textbf{Method} & \textbf{Signal} & \textbf{Candidate damage} & \textbf{$S$-dep.} & \textbf{$\rho$} & \textbf{A/F} & \textbf{Wiki}$\downarrow$ & \textbf{Avg$_{12}$}$\uparrow$ \\
\midrule
Random & None & uniform & \xmark & 20.06 & 34/111 & \cellcolor{CoCurveShade!7}{21.2} & \cellcolor{edgeorange!7}{46.2} \\
Magnitude & Weights & slice norm & \xmark & 20.03 & 2/126 & \cellcolor{CoCurveShade!11}{20.5} & \cellcolor{edgeorange!9}{48.1} \\
Wanda-sp & Activations & input-weighted norm & \xmark & 20.03 & 65/96 & \cellcolor{CoCurveShade!15}{20.0} & \cellcolor{edgeorange!20}{51.7} \\
FLAP & Activations & fluctuation-weighted norm & \xmark & 20.08 & 72/93 & \cellcolor{CoCurveShade!9}{20.9} & \cellcolor{edgeorange!15}{51.1} \\
LLM-Pruner & Gradients & Taylor group score & \xmark & 20.06 & 34/111 & \cellcolor{CoCurveShade!28}{\underline{13.2}} & \cellcolor{edgeorange!28}{\underline{54.5}} \\
SlimGPT & Curvature & OBS reconstruction error & \xmark & 20.07 & 74/92 & \cellcolor{CoCurveShade!20}{19.2} & \cellcolor{edgeorange!11}{49.7} \\
SlimLLM & Activations & output similarity & \xmark & 20.15 & 52/103 & \cellcolor{CoCurveShade!24}{17.4} & \cellcolor{edgeorange!24}{52.1} \\
AMP & Activations & $L_1$ act. $\times$ weight & \xmark & 20.00 & 71/93 & \cellcolor{CoCurveShade!14}{20.4} & \cellcolor{edgeorange!18}{51.7} \\
2SSP$^{\ddagger}$ & Acts. $+$ PPL & two-stage width/depth & stage 2 & 20.19 & 0/128 & \cellcolor{CoCurveShade!26}{16.3} & \cellcolor{edgeorange!28}{\underline{54.5}} \\
\textbf{CoCurve} & Output deltas & node $+$ incident edges & \textbf{\cmark} & \textbf{20.02} & \textbf{46/105} & \cellcolor{CoCurveShade!32}{\textbf{12.9}} & \cellcolor{edgeorange!32}{\textbf{56.0}} \\
\bottomrule
\end{tabular}
\end{adjustbox}
\vspace{-1pt}
\parbox{\textwidth}{\fontsize{5.45}{5.95}\selectfont\emph{Note.}
$S$-dep. means that a candidate's marginal cost changes with the removed set.
$\rho$ is realized parameter pruning (\%); A/F counts removed attention/FFN units.
All rows use the shared evaluation harness without post-pruning updates; ``stage 2'' denotes 2SSP's sequential attention-depth search.
$^{\ddagger}$2SSP retains its native allocation, while the others use the shared guard.}
\end{table}

This taxonomy follows prior pruning comparisons in separating the signal, score construction, and weight-update policy~\citep{sun2024wanda}.
The final columns connect mechanism to the mask and outcome actually evaluated: all methods reach the matched parameter budget, yet they make materially different module allocations.
Most baselines assign independent scores; 2SSP conditionally evaluates attention blocks only in its native depth stage, whereas CoCurve updates every candidate in one joint attention--FFN pool from its edges to all units already pruned.
The asymptotic construction and solver costs are given in \appref{app:complexity}; measured selection cost and its reuse across ratios are reported in \appref{app:selection-cost}.
We do not turn native wall-clock numbers into a ranking because published pipelines include different combinations of scoring, reconstruction, and recovery.

Wanda-sp is the structured generalization of Wanda~\citep{sun2024wanda} introduced by \citet{an2024flap}.
FLAP's published pipeline includes bias compensation; SlimGPT includes OBS reconstruction; SlimLLM includes output regression~\citep{an2024flap,ling2024slimgpt,tang2025slimllm}.
The reported rows therefore compare the masks induced by their selection signals before method-specific weight updates.

\textbf{Coverage and realized budgets.}
The primary $20$--$40\%$ grid contains all seven baselines for every model from $3$B through $24$B.
The reported $10\%$ and $50\%$ boundary cells are measured rather than imputed, and every available cell passes the same realized-budget check.
The $24$B main blocks include all seven methods, using streamed scoring for LLM-Pruner and SlimGPT.
At $70$B, packed NF4 supports Random, Magnitude, Wanda-sp, FLAP, LLM-Pruner, and CoCurve; SlimGPT and SlimLLM require dense second-order state beyond the single-A100 memory contract and are the only unavailable cells at that scale.

\Needspace{7\baselineskip}
\textbf{Additional methods and controls.}
On Llama-3.1-8B-Instruct, \cref{tab:additional-selectors} evaluates the two recent structured pruners configured in \cref{tab:baseline-config} across the full ratio sweep.
AMP uses an $L_1$ activation--weight score~\citep{sultan2025amp}; 2SSP combines FFN-width and attention-depth pruning~\citep{sandri2025twossp}.
AMP uses the shared guard; 2SSP retains its own two-stage allocation.
Single-module, diagonal, and interaction-granularity controls are reported in \appref{app:ablation-extra}.

\subsection{Evaluation Protocol}

\label{app:evalproto}
Multiple-choice tasks are scored by length-normalized log-likelihood over the candidate answers (zero-shot, \texttt{acc\_norm}); BoolQ uses the full validation split and MMLU is zero-shot log-likelihood.
Perplexity uses a $2048$-token context with stride $512$ over the full corpus.
All scores come from one shared harness.

For language models, candidate-path calibration evaluates full-vocabulary teacher KL on the $16$ disjoint sequences; the top-$256$ support is used to construct the Fisher--Gram features, not to score candidate masks.
The support-conditioned VLM path is specified below in \appref{app:vlm}.

A sliding window gives every predicted token more left context than a non-overlapping pass, so we use the same context and stride for all pruned models and their dense reference.

\subsection{Vision--Language Protocol}
\label{app:vlm}

\textbf{Units and budgets.}
The removable inventory contains attention heads---grouped by shared key/value head under GQA---and sixteen FFN channel groups per layer in both towers.
The common budget spans this joint inventory, while the per-layer cap $\kappa=1.5\rho$ and first/last-two-layer protection apply within each tower.
We introduce no tower quota, tower weight, or additional selection hyperparameter.

\textbf{Calibration.}
For each model, \method uses $128$ Flickr30k image--caption pairs~\citep{young2014flickr30k}; Wanda-sp and FLAP use $64$, and LLM-Pruner uses $32$ for its gradient pass.
Another $48$ disjoint pairs select $\lambda$.
These fixed counts are used for every model and drawn deterministically from one seeded calibration split.
Flickr30k avoids overlap with evaluation suites built from COCO images.
Only caption tokens are scored, and all unit ablations share the teacher's top-$64$ support.
The held-out path evaluates teacher KL on the same support; no benchmark labels enter estimation or selection.

\textbf{Evaluation and baselines.}
Each benchmark contributes $1000$ items under one answer-token scoring harness, with dense models reevaluated under that harness.
The structured baselines share the same inventory, budget, guard, and evaluation code.
Magnitude uses all weights owned by a unit; Wanda-sp and FLAP score inputs to residual-writing projections; LLM-Pruner accumulates first-order Taylor scores one image--caption pair per backward pass; and the layer-drop control uses ShortGPT's block-influence signal.

\textbf{Recovery.}
The language tower receives rank-$128$ LoRA~\citep{hu2022lora} with $\alpha=256$ while the vision encoder remains frozen and the multimodal projector is trained full-rank.
At global batch $16$, language and projector learning rates are $2.5\times10^{-5}$ and $2\times10^{-5}$.
Training uses $3000$ image--instruction samples, split equally between LLaVA conversations with COCO-train/Visual Genome images~\citep{liu2023llava} and short-answer TextVQA-train examples~\citep{singh2019textvqa}; training and evaluation images are disjoint.

\subsection{Recovery, Quantization, and Deployment Configuration}
\label{app:recipes}
Recovery and quantization use the complete twelve-task suite of \appref{app:datasets}, with the same four capability groups and Avg$_{12}$ aggregation as the pruning-only study.
Their quality measurements are therefore directly comparable across deployment stages.
\begin{table}[H]
\centering\footnotesize
\caption{\textbf{LLM recovery and numerical configuration.}
These settings describe the evaluated runs; the VLM configuration is specified in \appref{app:vlm}.}
\label{tab:adaptation-recipe}
\renewcommand{\arraystretch}{1.02}\setlength{\tabcolsep}{5pt}
\begin{tabular}{@{}>{\raggedright\arraybackslash}p{.19\linewidth}p{.76\linewidth}@{}}
\toprule
\textbf{Stage} & \textbf{Configuration} \\
\midrule
\rowcolor{headgray}\multicolumn{2}{c}{\textbf{Estimation and selection}} \\
LLM model forwards & bf16, evaluation mode for estimation and selection \\
Feature store / Gram & float32 memory-mapped features and block products; float64 assembly buffer, saved as float32 \\
Candidate solver & float64 scores and interaction sums; deterministic index ties \\
\midrule
\rowcolor{headgray}\multicolumn{2}{c}{\textbf{Common-LoRA recovery}} \\
LoRA data & Cleaned Alpaca; one tokenizer-packed pass ($589$--$629$ optimizer steps, model dependent) \\
LoRA adapters & Rank $16$, $\alpha=32$, dropout $0$; Q/K/V/O and gate/up/down projections; no bias training \\
Training & Microbatch $2$, accumulation $8$; AdamW, $\beta=(0.9,0.95)$, weight decay $0$, gradient clip $1$ \\
70B recovery & The same data, rank, effective batch, and schedule with microbatch $1$ and accumulation $16$; NF4 base weights, paged AdamW, and the fixed runtime CoCurve mask \\
Schedule & OneCycleLR with cosine decay, peak learning rate $10^{-4}$, $3\%$ warm-up phase; seed $0$ \\
\midrule
\rowcolor{headgray}\multicolumn{2}{c}{\textbf{Weight quantization}} \\
NF4 quality test & bitsandbytes NF4~\citep{dettmers2023qlora}, double quantization, bf16 compute; quantize at loading, then install the fixed runtime mask \\
INT8 quality test & bitsandbytes LLM.int8()~\citep{dettmers2022llmint8} with default outlier handling; install the same fixed runtime mask after loading \\
\midrule
\rowcolor{headgray}\multicolumn{2}{c}{\textbf{Physical deployment}} \\
Inference timing & Physically sliced bf16 weights on one isolated A100-80GB with eager attention; $32$ fixed prefill prompts, $16$ fixed decode prompts, $128$ new tokens, and $7$ synchronized repeats; the same prompts are used per dense/pruned pair \\
Quantized timing & The same physical $30/40\%$ structures as the quality rows; bitsandbytes NF4/INT8 on the same isolated A100; prefill batch $16$, decode batch $8$, $128$ new tokens, and $5$ synchronized repeats; throughput and latency are measured in the same run \\
\bottomrule
\end{tabular}
\end{table}

Peak memory always denotes maximum allocated CUDA memory under a fixed workload: synchronized generation for the sliced-bf16 efficiency sweep, and the forward workload for the quantized, deployment-context, and VLM comparisons.
All memory and timing comparisons are made only within their shared harness.

\FloatBarrier

\clearpage
\FloatBarrier
\section{Complete Compression Results}
\label{app:results}
The evidence follows the deployment path used in the main paper.
\Appref{app:lm-results} establishes pruning-only LLM quality and locates the frontier shift, \appref{app:recovery-quant} tests recovery and quantization, \appref{app:efficiency} measures physical cost, and \appref{app:vlm-results} transfers the formulation to VLMs.
\Appref{app:behavioral} changes the prompting protocol without changing the pruner, and \appref{app:deployment-context} compares pruning inherited capacity with starting from a native smaller checkpoint.

\subsection{Language Modeling and Downstream Accuracy}
\label{app:lm-results}
All results use the task taxonomy in \appref{app:datasets}.
\Cref{tab:ratio-full} resolves the perplexity curves numerically, and \cref{tab:full-unified} gives the complete downstream record.

\subsubsection{Perplexity Across Pruning Ratios}
\label{app:ratio}
The pruning-only sweep covers $10\%$ through $50\%$.
\Cref{tab:ratio-full} is the numerical counterpart of \cref{fig:pareto}: it reports every available method on WikiText-2, PTB, and C4 for six model slots under the common guard.
The $70$B block covers $10$--$50\%$ under the packed-NF4 protocol.
Only SlimGPT and SlimLLM are unavailable because their dense second-order state cannot be constructed on the single A100; all other $70$B arms are complete.

\paragraph{Collapse marking.}
For readability, pale orange hatching marks perplexities in the practical failure regime: the model's outputs have become sufficiently degraded that fine rank distinctions no longer support a meaningful deployment comparison.  The marker preserves the measured value and its within-column rank; it does not discard any result.

\begin{table}[htbp]
\centering
\caption{\textbf{Perplexity $\downarrow$ across pruning ratios.}
Six-model WikiText-2, PTB, and C4 breakdown; bold/underline mark first/second.}
\label{tab:ratio-full}
\adjustbox{width=\textwidth}{
{\fontsize{4.85}{5.25}\selectfont\setlength{\tabcolsep}{.42pt}\renewcommand{\arraystretch}{.84}
\begin{tabular}{@{}>{\centering\arraybackslash}m{2.2em}l>{\centering\arraybackslash}m{3.05em}>{\centering\arraybackslash}m{3.05em}>{\centering\arraybackslash}m{3.05em}>{\centering\arraybackslash}m{3.05em}>{\centering\arraybackslash}m{3.05em}>{\centering\arraybackslash}m{3.05em}>{\centering\arraybackslash}m{3.05em}>{\centering\arraybackslash}m{3.05em}>{\centering\arraybackslash}m{3.05em}>{\centering\arraybackslash}m{3.05em}>{\centering\arraybackslash}m{3.05em}>{\centering\arraybackslash}m{3.05em}>{\centering\arraybackslash}m{3.05em}>{\centering\arraybackslash}m{3.05em}>{\centering\arraybackslash}m{3.05em}>{\centering\arraybackslash}m{3.05em}>{\centering\arraybackslash}m{3.05em}>{\centering\arraybackslash}m{3.05em}@{}}
\toprule
 & & \multicolumn{3}{c}{\textbf{Llama-3.2-3B}} & \multicolumn{3}{c}{\textbf{Falcon3-7B}} & \multicolumn{3}{c}{\textbf{Llama-3.1-8B-It}} & \multicolumn{3}{c}{\textbf{Mistral-Nemo-12B}} & \multicolumn{3}{c}{\textbf{Mistral-Small-24B}} & \multicolumn{3}{c}{\textbf{Llama-3.1-70B-It}} \\
\cmidrule(lr){3-5}\cmidrule(lr){6-8}\cmidrule(lr){9-11}\cmidrule(lr){12-14}\cmidrule(lr){15-17}\cmidrule(lr){18-20}
\textbf{$\rho$} & \textbf{Method} & \textbf{Wiki} & \textbf{PTB} & \textbf{C4} & \textbf{Wiki} & \textbf{PTB} & \textbf{C4} & \textbf{Wiki} & \textbf{PTB} & \textbf{C4} & \textbf{Wiki} & \textbf{PTB} & \textbf{C4} & \textbf{Wiki} & \textbf{PTB} & \textbf{C4} & \textbf{Wiki} & \textbf{PTB} & \textbf{C4} \\
\midrule
\rowcolor{headgray}0\% & \textbf{Dense} & 6.9 & 10.5 & 10.4 & 5.4 & 9.0 & 12.2 & 6.4 & 10.4 & 10.7 & 5.0 & 35.2 & 8.9 & 4.7 & 24.6 & 8.5 & 3.7 & 8.0 & 8.5 \\
\bottomrule
\end{tabular}}
}
\end{table}

\textbf{Equal-perplexity compression.}
Interpolating the baseline envelope linearly in log-perplexity between measured ratios, without extrapolation, gives a complementary view of the curve.
At $30\%$, CoCurve permits $2.2$--$6.1$ percentage points more pruning at the same WikiText-2 perplexity on all five completed models.
At $40\%$, four of five gain $2.4$--$6.6$ points, while the $3$B model trails the interpolated envelope by $1.6$ points.
These are interpolated comparisons within the measured sweep, not additional evaluated masks.

Mistral-Small-24B gives the largest high-ratio margin in the completed sweep.
At $\rho=30\%$, \method reaches WikiText-2 perplexity $20.5$ against $78.8$ for the strongest baseline, a $3.8\times$ ratio; at $40\%$, the ratio reaches $11.9\times$ (\cref{tab:ratio-full}).

\subsubsection{Downstream Accuracy}
\label{app:full-results}
\Cref{tab:full-unified} gives all twelve downstream tasks for every method at $20$--$40\%$, plus the measured $50\%$ boundary on the $7$B--$24$B models.

\Needspace{8\baselineskip}
{\fontsize{5.35}{5.85}\selectfont\setlength{\tabcolsep}{3.65pt}\setlength{\fboxsep}{0.8pt}
\begin{longtable}{lllrrrrrrrrrrrr!{\color{black!35}\vrule width .5pt}r}
\caption{\textbf{Full downstream accuracy by task.}
All five language models and every method in the matched grid; the 7B--24B models additionally report the 50\% boundary.}
\label{tab:full-unified}\\
\toprule
 & & & \multicolumn{4}{c}{\textbf{Commonsense $\uparrow$}} & \multicolumn{3}{c}{\textbf{Science QA $\uparrow$}} & \multicolumn{4}{c}{\textbf{Reading $\uparrow$}} & \multicolumn{1}{c}{\textbf{Know. $\uparrow$}} & \\
\cmidrule(lr){4-7}\cmidrule(lr){8-10}\cmidrule(lr){11-14}\cmidrule(lr){15-15}
\textbf{Model} & \textbf{$\rho$} & \textbf{Method} & \textbf{PIQA} & \textbf{HellaS} & \textbf{WinoG} & \textbf{CSQA} & \textbf{ARC-e} & \textbf{ARC-c} & \textbf{OBQA} & \textbf{BoolQ} & \textbf{RACE} & \textbf{RACE-h} & \textbf{QuAIL} & \textbf{MMLU} & \textbf{Avg$_{12}$} \\
\midrule
\endfirsthead
\toprule
 & & & \multicolumn{4}{c}{\textbf{Commonsense $\uparrow$}} & \multicolumn{3}{c}{\textbf{Science QA $\uparrow$}} & \multicolumn{4}{c}{\textbf{Reading $\uparrow$}} & \multicolumn{1}{c}{\textbf{Know. $\uparrow$}} & \\
\cmidrule(lr){4-7}\cmidrule(lr){8-10}\cmidrule(lr){11-14}\cmidrule(lr){15-15}
\textbf{Model} & \textbf{$\rho$} & \textbf{Method} & \textbf{PIQA} & \textbf{HellaS} & \textbf{WinoG} & \textbf{CSQA} & \textbf{ARC-e} & \textbf{ARC-c} & \textbf{OBQA} & \textbf{BoolQ} & \textbf{RACE} & \textbf{RACE-h} & \textbf{QuAIL} & \textbf{MMLU} & \textbf{Avg$_{12}$} \\
\midrule
\endhead
\midrule\multicolumn{16}{r}{\scriptsize\emph{continued on the next page}}\\
\endfoot
\bottomrule
\endlastfoot
\rowcolor{headgray}\textbf{3B} & -- & \textbf{Dense} & 77.4 & 72.8 & 58.6 & 57.8 & 71.3 & 46.9 & 43.2 & 74.3 & 59.9 & 45.5 & 54.9 & 41.8 & 58.7\se{0.35} \\
 & 20\% & Random & \cellc{0}{63.1} & \cellc{0}{42.4} & \cellc{14}{51.5} & \cellc{0}{32.1} & \cellc{0}{45.0} & \cellc{14}{27.1} & \cellc{0}{28.6} & \cellc{0}{44.5} & \cellc{14}{37.7} & \cellc{0}{29.4} & \cellc{14}{39.7} & \cellc{14}{28.8} & \cellavg{0}{39.2\se{0.34}} \\
 &  & Magnitude & \cellc{14}{66.0} & \cellc{43}{46.8} & \cellc{100}{\textbf{54.2}} & \cellc{57}{37.9} & \cellc{29}{50.5} & \cellc{57}{31.8} & \cellc{57}{31.8} & \cellc{14}{45.8} & \cellc{57}{42.8} & \cellc{43}{33.8} & \cellc{86}{\underline{44.1}} & \cellc{57}{31.6} & \cellavg{29}{43.1\se{0.34}} \\
 &  & Wanda-sp & \cellc{43}{68.9} & \cellc{57}{48.6} & \cellc{29}{51.6} & \cellc{43}{37.7} & \cellc{57}{53.5} & \cellc{43}{30.1} & \cellc{43}{31.6} & \cellc{43}{59.3} & \cellc{43}{41.4} & \cellc{57}{34.0} & \cellc{57}{43.0} & \cellc{43}{30.7} & \cellavg{57}{44.2\se{0.34}} \\
 &  & FLAP & \cellc{29}{67.1} & \cellc{29}{46.4} & \cellc{57}{52.7} & \cellc{29}{36.5} & \cellc{43}{51.4} & \cellc{0}{27.0} & \cellc{14}{28.8} & \cellc{86}{\underline{63.9}} & \cellc{29}{40.1} & \cellc{29}{31.8} & \cellc{29}{42.1} & \cellc{29}{30.2} & \cellavg{43}{43.2\se{0.34}} \\
 &  & LLM-Pruner & \cellc{86}{\underline{70.6}} & \cellc{86}{\underline{55.3}} & \cellc{71}{53.4} & \cellc{86}{\underline{42.3}} & \cellc{86}{\underline{55.2}} & \cellc{86}{\underline{32.7}} & \cellc{100}{\textbf{34.6}} & \cellc{57}{60.8} & \cellc{71}{43.2} & \cellc{86}{\underline{35.0}} & \cellc{43}{42.4} & \cellc{86}{\underline{33.5}} & \cellavg{86}{\underline{46.6}\se{0.35}} \\
 &  & SlimGPT & \cellc{57}{69.3} & \cellc{14}{45.7} & \cellc{43}{52.0} & \cellc{14}{33.1} & \cellc{14}{49.7} & \cellc{29}{27.8} & \cellc{29}{30.6} & \cellc{71}{63.0} & \cellc{0}{37.4} & \cellc{14}{30.7} & \cellc{0}{39.7} & \cellc{0}{28.8} & \cellavg{14}{42.3\se{0.34}} \\
\end{longtable}}

\Needspace{5\baselineskip}
Across the $19$ complete model--ratio blocks, \method ranks first on Avg$_{12}$ in $13$ and first or second in $18$.
The coverage is not confined to one architecture or an easy ratio: every completed block from $30\%$ through $50\%$ is top-two, spanning all five models represented in the table.
Together with the corpus-level sweep above, this shows that the gain transfers from next-token prediction to a broad capability profile rather than being carried by one task family.

\subsubsection{Additional Method and Granularity Coverage}
\begin{table}[htbp]
\centering\footnotesize
\caption{\textbf{Additional pruning methods across $20$--$50\%$} on Llama-3.1-8B-Instruct.
AMP uses the shared guard; $^{\ddagger}$2SSP retains its native two-stage allocation.
All twelve tasks use exactly the scoring and aggregation rules of \cref{tab:flagship} and \appref{app:datasets}.}
\label{tab:additional-selectors}
\adjustbox{width=\textwidth}{
\tiny\setlength{\tabcolsep}{1.25pt}\renewcommand{\arraystretch}{.96}
\begin{tabular}{@{}c l ccc ccccccccccccc@{}}
\toprule
& & \multicolumn{3}{c}{\textbf{Perplexity}\,\dn} & \multicolumn{13}{c}{\textbf{Downstream accuracy (\%)}\,\up}\\
\cmidrule(lr){3-5}\cmidrule(lr){6-18}
\textbf{$\rho$} & \textbf{Method} & \textbf{Wiki} & \textbf{PTB} & \textbf{C4} & \textbf{PIQA} & \textbf{HS} & \textbf{WG} & \textbf{CSQA} & \textbf{A-e} & \textbf{A-c} & \textbf{OBQA} & \textbf{BoolQ} & \textbf{R-m} & \textbf{R-h} & \textbf{QuAIL} & \textbf{MMLU} & \textbf{Avg$_{12}$}\\
\midrule
\rowcolor{headgray} -- & \textbf{Dense} & 6.40 & 10.36 & 10.67 & 80.9 & 77.4 & 66.2 & 59.2 & 78.1 & 54.4 & 47.6 & 85.0 & 68.8 & 54.1 & 62.7 & 47.8 & 65.2\\
\midrule
\multirow{3}{*}{20\%} & AMP & \cellc{14}{20.39} & \cellc{14}{47.27} & \cellc{14}{33.26} & \cellc{40}{\snd{74.7}} & \cellc{14}{58.1} & \cellc{14}{56.7} & \cellc{14}{49.2} & \cellc{40}{\snd{63.9}} & \cellc{14}{38.8} & \cellc{40}{\snd{38.8}} & \cellc{40}{\snd{69.4}} & \cellc{14}{51.1} & \cellc{14}{39.5} & \cellc{14}{44.1} & \cellc{14}{35.7} & \cellavg{21}{51.7}\\
 & 2SSP$^{\ddagger}$ & \cellc{40}{\snd{16.29}} & \cellc{64}{\best{24.07}} & \cellc{40}{\snd{29.00}} & \cellc{14}{73.6} & \cellc{40}{\snd{61.9}} & \cellc{40}{\snd{59.7}} & \cellc{40}{\snd{49.4}} & \cellc{64}{\best{66.8}} & \cellc{40}{\snd{41.3}} & \cellc{14}{38.6} & \cellc{14}{64.3} & \cellc{64}{\best{60.0}} & \cellc{64}{\best{45.9}} & \cellc{64}{\best{54.2}} & \cellc{40}{\snd{37.9}} & \cellavg{60}{\snd{54.5}}\\
 & \textbf{CoCurve} & \cellc{64}{\best{12.88}} & \cellc{40}{\snd{29.63}} & \cellc{64}{\best{19.58}} & \cellc{64}{\best{75.8}} & \cellc{64}{\best{67.5}} & \cellc{64}{\best{61.0}} & \cellc{64}{\best{50.4}} & \cellc{14}{62.5} & \cellc{64}{\best{41.9}} & \cellc{64}{\best{39.6}} & \cellc{64}{\best{78.3}} & \cellc{40}{\snd{58.1}} & \cellc{40}{\snd{44.0}} & \cellc{40}{\snd{54.1}} & \cellc{64}{\best{38.5}} & \cellavg{96}{\best{56.0}}\\
\midrule
\multirow{3}{*}{30\%} & AMP & \cellc{14}{87.89} & \cellc{14}{158.9} & \cellc{14}{121.3} & \cellc{14}{66.8} & \cellc{14}{37.5} & \cellc{14}{51.9} & \cellc{40}{\snd{40.7}} & \cellc{40}{\snd{52.2}} & \cellc{14}{29.1} & \cellc{14}{30.0} & \cellc{14}{45.6} & \cellc{14}{36.9} & \cellc{14}{29.7} & \cellc{14}{38.4} & \cellc{14}{30.1} & \cellavg{21}{40.7}\\
 & 2SSP$^{\ddagger}$ & \cellc{40}{\snd{43.63}} & \cellc{40}{\snd{73.44}} & \cellc{40}{\snd{70.37}} & \cellc{40}{\snd{68.0}} & \cellc{40}{\snd{47.0}} & \cellc{40}{\snd{53.8}} & \cellc{14}{37.5} & \cellc{14}{51.2} & \cellc{40}{\snd{31.5}} & \cellc{40}{\snd{32.2}} & \cellc{40}{\snd{63.7}} & \cellc{64}{\best{48.8}} & \cellc{40}{\snd{35.7}} & \cellc{40}{\snd{44.4}} & \cellc{40}{\snd{31.8}} & \cellavg{60}{\snd{45.5}}\\
 & \textbf{CoCurve} & \cellc{64}{\best{21.59}} & \cellc{64}{\best{36.31}} & \cellc{64}{\best{31.99}} & \cellc{64}{\best{70.5}} & \cellc{64}{\best{55.8}} & \cellc{64}{\best{55.5}} & \cellc{64}{\best{41.0}} & \cellc{64}{\best{55.6}} & \cellc{64}{\best{34.4}} & \cellc{64}{\best{35.2}} & \cellc{64}{\best{73.1}} & \cellc{40}{\snd{46.5}} & \cellc{64}{\best{37.3}} & \cellc{64}{\best{50.1}} & \cellc{64}{\best{33.4}} & \cellavg{96}{\best{49.0}}\\
\midrule
\multirow{3}{*}{40\%} & AMP & \collapsecell{1.3e3} & \collapsecell{2.0e3} & \collapsecell{1.0e3} & \cellc{14}{55.2} & \cellc{14}{26.5} & \cellc{14}{48.5} & \cellc{14}{26.5} & \cellc{14}{34.3} & \cellc{14}{23.7} & \cellc{14}{23.4} & \cellc{14}{38.0} & \cellc{14}{27.6} & \cellc{14}{26.0} & \cellc{14}{21.7} & \cellc{14}{27.3} & \cellavg{21}{31.6}\\
 & 2SSP$^{\ddagger}$ & \cellc{40}{\snd{292.2}} & \cellc{40}{\snd{508.2}} & \cellc{40}{\snd{334.3}} & \cellc{40}{\snd{62.1}} & \cellc{40}{\snd{34.7}} & \cellc{40}{\snd{50.9}} & \cellc{40}{\snd{27.6}} & \cellc{40}{\snd{38.2}} & \cellc{40}{\snd{25.8}} & \cellc{40}{\snd{27.8}} & \cellc{40}{\snd{62.0}} & \cellc{40}{\snd{39.7}} & \cellc{40}{\snd{30.4}} & \cellc{40}{\snd{38.2}} & \cellc{40}{\snd{28.0}} & \cellavg{60}{\snd{38.8}}\\
 & \textbf{CoCurve} & \cellc{64}{\best{59.13}} & \cellc{64}{\best{108.4}} & \cellc{64}{\best{76.05}} & \cellc{64}{\best{64.0}} & \cellc{64}{\best{41.0}} & \cellc{64}{\best{51.4}} & \cellc{64}{\best{31.9}} & \cellc{64}{\best{41.6}} & \cellc{64}{\best{28.7}} & \cellc{64}{\best{29.6}} & \cellc{64}{\best{62.3}} & \cellc{64}{\best{41.9}} & \cellc{64}{\best{32.2}} & \cellc{64}{\best{42.8}} & \cellc{64}{\best{29.3}} & \cellavg{96}{\best{41.4}}\\
\midrule
\multirow{3}{*}{50\%} & AMP & \collapsecell{3.6e4} & \collapsecell{5.4e4} & \collapsecell{2.2e4} & \cellc{14}{52.8} & \cellc{14}{25.1} & \cellc{14}{48.6} & \cellc{14}{20.4} & \cellc{14}{27.3} & \cellc{40}{\snd{22.5}} & \cellc{14}{21.8} & \cellc{14}{39.9} & \cellc{14}{26.7} & \cellc{14}{26.2} & \cellc{40}{\snd{27.0}} & \cellc{14}{26.4} & \cellavg{21}{30.4}\\
 & 2SSP$^{\ddagger}$ & \collapsecell{\snd{2.6e3}} & \collapsecell{\snd{2.3e3}} & \collapsecell{\snd{1.9e3}} & \cellc{40}{\snd{53.5}} & \cellc{40}{\snd{26.5}} & \cellc{40}{\snd{49.0}} & \cellc{40}{\snd{22.0}} & \cellc{40}{\snd{29.8}} & \cellc{14}{22.1} & \cellc{64}{\best{25.6}} & \cellc{40}{\snd{53.1}} & \cellc{40}{\snd{30.8}} & \cellc{40}{\snd{26.6}} & \cellc{64}{\best{30.8}} & \cellc{40}{\snd{26.6}} & \cellavg{60}{\snd{33.0}}\\
 & \textbf{CoCurve} & \cellc{64}{\best{956.4}} & \collapsecell{\best{1.8e3}} & \collapsecell{\best{828.8}} & \cellc{64}{\best{55.7}} & \cellc{64}{\best{30.0}} & \cellc{64}{\best{50.4}} & \cellc{64}{\best{24.2}} & \cellc{64}{\best{30.7}} & \cellc{64}{\best{26.5}} & \cellc{40}{\snd{24.8}} & \cellc{64}{\best{60.8}} & \cellc{64}{\best{31.1}} & \cellc{64}{\best{27.8}} & \cellc{14}{24.5} & \cellc{64}{\best{27.5}} & \cellavg{96}{\best{34.5}}\\
\bottomrule
\end{tabular}
}
\end{table}

The comparison preserves AMP's shared-budget protocol and 2SSP's native allocation.
Across $20$--$50\%$, \method leads Avg$_{12}$ at all four ratios and $11$ of the $12$ corpus-level perplexity columns; the sole exception is PTB at $20\%$, where it ranks second.
The result extends the advantage to a wider method pool under the same evaluation contract.

\paragraph{Finer-grained baseline units.}
Wanda-sp, FLAP, and Magnitude are also evaluated with $64$ FFN groups per layer on Llama-3.1-8B-Instruct at $20\%$, preserving the guard and evaluation harness.
Wanda-sp improves from $20.02$ to $14.86$ WikiText-2 perplexity and FLAP from $20.88$ to $16.23$, while Magnitude declines on four of the five summary metrics.
CoCurve with $16$ groups retains lower perplexity and higher accuracy in every capability group than these finer-grained arms.
This control shows that its advantage is not explained by giving the baselines a coarser FFN inventory.

\subsubsection{Where CoCurve Advances the Deployment Frontier}
\label{app:frontier-analysis}
\paragraph{Frontier criterion.}
The quality--deployment frontier is an operating curve, not a single pruning ratio.
A method advances it by retaining more quality at a fixed resource budget or by reaching the same quality with fewer resources.
We test both directions and then connect them to measured memory and throughput (\cref{fig:frontier-analysis}).

\paragraph{Quality at a fixed pruning budget.}
Across the $6$ LLMs and $3$ corpora, CoCurve wins $83\%$, $94\%$, and $83\%$ of the $18$ matched comparisons at $20/30/40\%$ pruning.
More importantly, the median strongest-baseline-to-CoCurve perplexity ratio grows from $1.06\times$ to $1.16\times$ and $1.36\times$ as compression tightens.
The benefit is smallest while every pruner retains ample redundancy and becomes most visible as baseline curves begin to bend.

\paragraph{Resources at matched quality.}
Matched-quality interpolation gives the deployment meaning of that separation.
On the five-model shared grid, CoCurve supports $2.2$--$6.6$ additional pruning points in $9/10$ model--ratio comparisons at $30/40\%$; the sole exception is the $3$B model at $40\%$, where its small capacity has already entered a brittle regime.
The largest shifts occur on the $24$B model, but positive results on Falcon, Llama, and Mistral-Nemo show that scale alone is not the condition.
The gain is strongest near the curve's bend, where usable redundancy remains but isolated scores begin to miss compatible removals.

\paragraph{Measured physical path.}
At a fixed pruning ratio, all sliced selectors occupy a narrow throughput--memory band (\cref{tab:efficiency-band}), so CoCurve's gain comes from retaining more quality at essentially the same physical footprint; moving to a higher ratio then converts that quality margin into additional memory and throughput savings.
On Llama-3.1-8B-Instruct, the measured CoCurve path lowers peak memory from $17.18$ to $11.53$ GB and raises prefill throughput to $1.41\times$ at $40\%$, while the full cross-model paths are given in \cref{fig:quality-efficiency-matrix}.
Interaction-aware selection therefore expands the useful pruning-only region; common recovery and low-bit weights subsequently extend the same selected structures beyond its bend (\appref{app:recovery-quant}).

\begin{figure}[H]
\centering
\includegraphics[width=.98\textwidth]{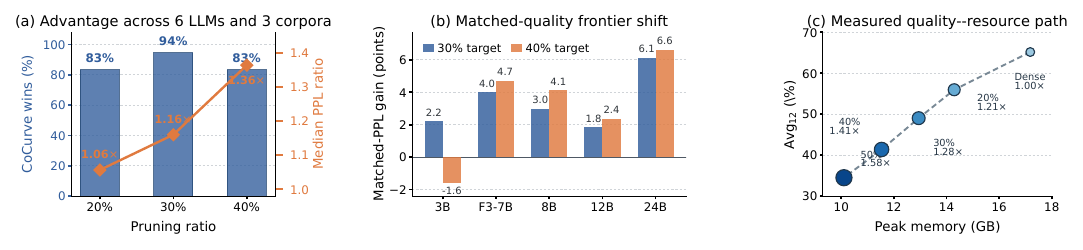}
\caption{\textbf{Where CoCurve moves the measured deployment frontier.}
\textbf{(a)} Win rate and median perplexity ratio against the strongest available baseline over $6$ LLMs and $3$ corpora at each shared operating point.
\textbf{(b)} Matched-quality pruning-ratio shift on the five-model grid; positive values mean that CoCurve permits more pruning at the same WikiText-2 perplexity.
\textbf{(c)} The measured Llama-3.1-8B-Instruct path in Avg$_{12}$, isolated-A100 peak memory, and prefill speedup.
Values use linear interpolation in log-perplexity between measured baseline-envelope points without extrapolation.}
\label{fig:frontier-analysis}
\end{figure}

\subsection{Recovery and Quantization}
\label{app:recovery-quant}
\subsubsection{Common-LoRA Recovery}
\label{app:recovery}
We apply one common LoRA stage to the physically sliced $8$B and $24$B models from all five reported pruning methods, keeping the selected structures fixed.
\Cref{tab:adaptation-recipe} gives the complete recipe: one tokenizer-packed Alpaca pass gives $589$/$601$ optimizer steps for the $8$B/$24$B models at an effective batch of $16$ sequences.

\begingroup
\centering\fontsize{5.9}{6.4}\selectfont
\captionof{table}{\textbf{Common-LoRA recovery by task.}
Panel (a) reports full task sweeps on Llama-3.1-8B-Instruct and Mistral-Small-24B; panel (b) reports PPL-only CoCurve recovery on the remaining scales.
Dense and CoCurve without LoRA are reference states; every recovered checkpoint uses the identical recipe in \cref{tab:adaptation-recipe}.}
\label{tab:recovery-full}
\addtocounter{table}{-1}
\textbf{(a) Full two-model comparison.}\par\vspace{1pt}
\setlength{\tabcolsep}{.50pt}\renewcommand{\arraystretch}{.82}
\setlength\LTleft{0pt}\setlength\LTright{0pt}
\begin{longtable}{@{}>{\centering\arraybackslash}m{2.3em}l>{\centering\arraybackslash}m{3.18em}>{\centering\arraybackslash}m{3.18em}>{\centering\arraybackslash}m{3.18em}!{\color{black!35}\vrule width .5pt}>{\centering\arraybackslash}m{3.18em}>{\centering\arraybackslash}m{3.18em}>{\centering\arraybackslash}m{3.18em}>{\centering\arraybackslash}m{3.18em}>{\centering\arraybackslash}m{3.18em}>{\centering\arraybackslash}m{3.18em}>{\centering\arraybackslash}m{3.18em}>{\centering\arraybackslash}m{3.18em}>{\centering\arraybackslash}m{3.18em}>{\centering\arraybackslash}m{3.18em}>{\centering\arraybackslash}m{3.18em}>{\centering\arraybackslash}m{3.18em}!{\color{black!35}\vrule width .5pt}>{\centering\arraybackslash}m{4.55em}@{}}
\toprule
 & & \multicolumn{3}{c}{\textbf{Perplexity} $\downarrow$} & \multicolumn{12}{c}{\textbf{Downstream accuracy} $\uparrow$} & \\
\cmidrule(lr){3-5}\cmidrule(lr){6-17}
\textbf{$\rho$} & \textbf{Method} & \textbf{Wiki} & \textbf{PTB} & \textbf{C4} & \textbf{PIQA} & \textbf{HS} & \textbf{WG} & \textbf{CSQA} & \textbf{A-e} & \textbf{A-c} & \textbf{OBQA} & \textbf{BoolQ} & \textbf{R-m} & \textbf{R-h} & \textbf{QuAIL} & \textbf{MMLU} & \textbf{Avg$_{12}$} \\
\midrule
\endfirsthead
\multicolumn{18}{r}{\fontsize{5.4}{6}\selectfont\itshape Table~\thetable(a), continued} \\
\toprule
 & & \multicolumn{3}{c}{\textbf{Perplexity} $\downarrow$} & \multicolumn{12}{c}{\textbf{Downstream accuracy} $\uparrow$} & \\
\cmidrule(lr){3-5}\cmidrule(lr){6-17}
\textbf{$\rho$} & \textbf{Method} & \textbf{Wiki} & \textbf{PTB} & \textbf{C4} & \textbf{PIQA} & \textbf{HS} & \textbf{WG} & \textbf{CSQA} & \textbf{A-e} & \textbf{A-c} & \textbf{OBQA} & \textbf{BoolQ} & \textbf{R-m} & \textbf{R-h} & \textbf{QuAIL} & \textbf{MMLU} & \textbf{Avg$_{12}$} \\
\midrule
\endhead
\midrule\multicolumn{18}{r}{\fontsize{5.4}{6}\selectfont\itshape continued on next page} \\
\endfoot
\bottomrule
\endlastfoot
\rowcolor{headgray}\multicolumn{18}{c}{\textbf{Llama-3.1-8B-Instruct}} \\*
\rowcolor{black!9}0\% & \textbf{Dense} & 6.4 & 10.4 & 10.7 & 80.9 & 77.4 & 66.2 & 59.2 & 78.1 & 54.4 & 47.6 & 85.0 & 68.8 & 54.1 & 62.7 & 47.8 & 65.2\se{0.34} \\
\midrule
\multirow{6}{*}{20\%} & Random & \cellcolor{CoCurveShade!7}11.3 & \cellcolor{CoCurveShade!26}\underline{16.2} & \cellcolor{CoCurveShade!13}17.9 & \cellcolor{CoCurveShade!13}76.3 & \cellcolor{CoCurveShade!7}65.6 & \cellcolor{CoCurveShade!32}\textbf{62.0} & \cellcolor{CoCurveShade!32}\textbf{50.7} & \cellcolor{CoCurveShade!32}\textbf{68.6} & \cellcolor{CoCurveShade!13}42.3 & \cellcolor{CoCurveShade!24}\underline{43.2} & \cellcolor{CoCurveShade!20}78.6 & \cellcolor{CoCurveShade!13}57.0 & \cellcolor{CoCurveShade!13}43.7 & \cellcolor{CoCurveShade!20}55.7 & \cellcolor{CoCurveShade!20}38.4 & \cellcolor{edgeorange!20}56.8\se{0.35} \\*
 & LLM-Pruner & \cellcolor{CoCurveShade!32}\textbf{9.9} & \cellcolor{CoCurveShade!32}\textbf{14.7} & \cellcolor{CoCurveShade!26}\underline{16.9} & \cellcolor{CoCurveShade!7}76.0 & \cellcolor{CoCurveShade!26}\underline{68.0} & \cellcolor{CoCurveShade!13}60.9 & \cellcolor{CoCurveShade!7}42.8 & \cellcolor{CoCurveShade!7}62.8 & \cellcolor{CoCurveShade!7}40.4 & \cellcolor{CoCurveShade!7}39.2 & \cellcolor{CoCurveShade!32}\textbf{81.2} & \cellcolor{CoCurveShade!26}\underline{60.4} & \cellcolor{CoCurveShade!26}\underline{46.6} & \cellcolor{CoCurveShade!13}55.4 & \cellcolor{CoCurveShade!26}\underline{39.0} & \cellcolor{edgeorange!13}56.0\se{0.34} \\
\end{longtable}

\vspace{2pt}
\textbf{(b) CoCurve PPL recovery across the remaining scales.}\par\vspace{1pt}
\adjustbox{max width=\textwidth}{
\setlength{\tabcolsep}{.42pt}\renewcommand{\arraystretch}{.88}
\begin{tabular}{@{}>{\centering\arraybackslash}m{2.0em}>{\centering\arraybackslash}m{2.45em}>{\centering\arraybackslash}m{2.45em}>{\centering\arraybackslash}m{2.45em}>{\centering\arraybackslash}m{2.45em}>{\centering\arraybackslash}m{2.45em}>{\centering\arraybackslash}m{2.45em}>{\centering\arraybackslash}m{2.45em}>{\centering\arraybackslash}m{2.45em}>{\centering\arraybackslash}m{2.45em}>{\centering\arraybackslash}m{2.45em}>{\centering\arraybackslash}m{2.45em}>{\centering\arraybackslash}m{2.45em}>{\centering\arraybackslash}m{2.45em}>{\centering\arraybackslash}m{2.45em}>{\centering\arraybackslash}m{2.45em}>{\centering\arraybackslash}m{2.45em}>{\centering\arraybackslash}m{2.45em}>{\centering\arraybackslash}m{2.45em}>{\centering\arraybackslash}m{2.45em}>{\centering\arraybackslash}m{2.45em}>{\centering\arraybackslash}m{2.45em}>{\centering\arraybackslash}m{2.45em}>{\centering\arraybackslash}m{2.45em}>{\centering\arraybackslash}m{2.45em}@{}}\toprule
\textbf{$\rho$} & \multicolumn{6}{c}{\textbf{Llama-3.2-3B}} & \multicolumn{6}{c}{\textbf{Falcon3-7B}} & \multicolumn{6}{c}{\textbf{Mistral-Nemo-12B}} & \multicolumn{6}{c}{\textbf{Llama-3.1-70B-It}} \\
\cmidrule(lr){2-7}\cmidrule(lr){8-13}\cmidrule(lr){14-19}\cmidrule(lr){20-25}
 & \multicolumn{3}{c}{\textbf{w/o LoRA}} & \multicolumn{3}{c}{\textbf{$+$LoRA}} & \multicolumn{3}{c}{\textbf{w/o LoRA}} & \multicolumn{3}{c}{\textbf{$+$LoRA}} & \multicolumn{3}{c}{\textbf{w/o LoRA}} & \multicolumn{3}{c}{\textbf{$+$LoRA}} & \multicolumn{3}{c}{\textbf{w/o LoRA}} & \multicolumn{3}{c}{\textbf{$+$LoRA}} \\
\bottomrule\end{tabular}
}
\par
\endgroup

\begin{figure}[H]
\centering
\includegraphics[width=.84\textwidth]{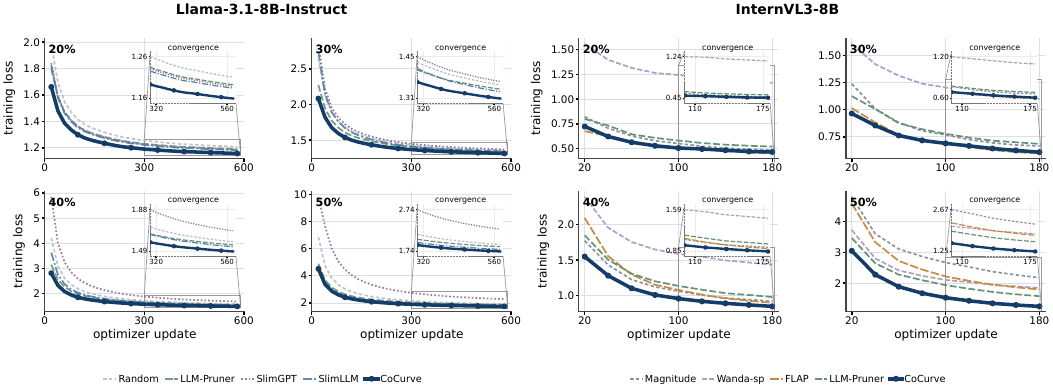}
\caption{\textbf{Optimization trajectories under the common LoRA recipe.}
The left $2{\times}2$ matrix reports Llama-3.1-8B-Instruct with an identical $589$-update recipe; the right reports InternVL3-8B with an identical $188$-update recipe.
CoCurve has the lowest recorded LLM loss throughout and finishes top-two on every VLM trajectory, leading at $40/50\%$.
Held-out results in \cref{tab:recovery-full} and \cref{tab:vlm-recovery}, rather than training loss, establish recovered quality.}
\label{fig:recovery-trajectories}
\end{figure}

\Cref{tab:recovery-main} gives the compact WikiText-2 and Avg$_{12}$ view.
\Cref{tab:recovery-full} resolves the two full-evaluation models into all three perplexity corpora and all twelve downstream tasks across $20$--$50\%$ pruning.
Its second panel extends the CoCurve-only PPL comparison to the remaining model scales; it is used only to measure recovery, not to make downstream claims.
At $70$B, the same data, rank, effective batch, and optimizer schedule are applied as QLoRA on the experiment's NF4 base; the already evaluated CoCurve structure remains a fixed runtime mask because packed low-bit matrices cannot be sliced as ordinary dense tensors.
Each panel includes the same CoCurve mask before recovery, exposing both the remaining gap to the original model and the improvement supplied by the common LoRA stage.
Across all eight $8$B/$24$B blocks, \method is top-two in $14/16$ WikiText-2/Avg$_{12}$ columns.
Across all six scales, the identical recipe lowers CoCurve's WikiText-2 perplexity at all $24$ measured model--ratio points.
The effect remains substantial at the hardest completed boundary: at $50\%$ on the $8$B model, perplexity falls from $956.4$ to $24.2$ and Avg$_{12}$ rises from $34.5$ to $45.3$.
Thus recovery is not compensating for a single favorable checkpoint; it repeatedly exposes useful capacity retained by the selected structure.

Across the language panels, \cref{fig:recovery-trajectories} shows a lower training objective throughout the identical schedule, with particularly clear early separation at $30\%$ and $40\%$.
The VLM panels retain the same favorable ordering at the harder ratios.
The held-out tables complete the argument: CoCurve's structures are easier to optimize under common recipes and retain better capability after optimization.

\FloatBarrier
\subsubsection{Composition with Quantization}
\label{app:quant}
We load the same fixed masks with bitsandbytes NF4 and INT8 weights (\cref{tab:adaptation-recipe}).
Quality evaluation quantizes at loading and then applies each fixed mask; the systems columns independently time the corresponding physically sliced, low-bit checkpoint under the isolated A100 protocol.
Dense NF4/INT8 isolate numerical compression, CoCurve bf16 isolates structural compression, and CoCurve+NF4/INT8 realizes their composition.

Across the two ratios, two precisions, and three corpora, CoCurve leads all $12$ matched perplexity comparisons.
Its four Avg$_{12}$ entries are either best or within $0.2$ points of the best alternative.
Across the same CoCurve structures, NF4 halves bf16 peak memory while retaining $96\%$ of prefill throughput, and physical slicing then moves beyond dense NF4 to lower-memory, higher-throughput operating points.
\Cref{tab:quant-main} gives the compact quality--systems view, while \cref{tab:quant-full} resolves all twelve downstream tasks.

\begin{table}[htbp]
\centering\fontsize{6.0}{6.5}\selectfont
\caption{\textbf{Quantized pruning by task} on Llama-3.1-8B-Instruct.
The same fixed masks are evaluated under NF4 and INT8 at $30\%$ and $40\%$ pruning.}
\label{tab:quant-full}
\adjustbox{max width=\textwidth}{
\setlength{\tabcolsep}{1.25pt}\renewcommand{\arraystretch}{.91}
\begin{tabular}{@{}>{\centering\arraybackslash}m{2.3em}l>{\centering\arraybackslash}m{2.7em}>{\centering\arraybackslash}m{3.3em}>{\centering\arraybackslash}m{3.3em}>{\centering\arraybackslash}m{3.3em}!{\color{black!35}\vrule width .5pt}>{\centering\arraybackslash}m{3.3em}>{\centering\arraybackslash}m{3.3em}>{\centering\arraybackslash}m{3.3em}>{\centering\arraybackslash}m{3.3em}>{\centering\arraybackslash}m{3.3em}>{\centering\arraybackslash}m{3.3em}>{\centering\arraybackslash}m{3.3em}>{\centering\arraybackslash}m{3.3em}>{\centering\arraybackslash}m{3.3em}>{\centering\arraybackslash}m{3.3em}>{\centering\arraybackslash}m{3.3em}>{\centering\arraybackslash}m{3.3em}!{\color{black!35}\vrule width .5pt}>{\centering\arraybackslash}m{3.8em}@{}}\toprule
 & & & \multicolumn{3}{c}{\textbf{Perplexity} $\downarrow$} & \multicolumn{12}{c}{\textbf{Downstream accuracy} $\uparrow$} & \\
\cmidrule(lr){4-6}\cmidrule(lr){7-18}
\textbf{$\rho$} & \textbf{Method} & \textbf{Prec.} & \textbf{Wiki} & \textbf{PTB} & \textbf{C4} & \textbf{PIQA} & \textbf{HS} & \textbf{WG} & \textbf{CSQA} & \textbf{A-e} & \textbf{A-c} & \textbf{OBQA} & \textbf{BoolQ} & \textbf{R-m} & \textbf{R-h} & \textbf{QuAIL} & \textbf{MMLU} & \textbf{Avg$_{12}$} \\
\midrule
\rowcolor{headgray}-- & \textbf{Dense} & INT8 & 6.5 & 10.4 & 10.9 & 81.1 & 77.3 & 66.2 & 59.0 & 78.0 & 54.1 & 48.4 & 84.9 & 68.5 & 53.9 & 62.5 & 47.7 & 65.1 \\
\rowcolor{headgray} &  & NF4 & 6.8 & 10.8 & 11.5 & 79.9 & 76.9 & 66.4 & 59.8 & 76.2 & 52.8 & 46.2 & 85.4 & 67.0 & 53.0 & 61.9 & 46.7 & 64.3 \\
\midrule
\multirow{6}{*}{30\%} & \multirow{2}{*}{LLM-Pruner} & INT8 & \cellcolor{CoCurveShade!20}\underline{26.8} & \cellcolor{CoCurveShade!20}\underline{39.6} & \cellcolor{CoCurveShade!20}\underline{40.8} & \cellcolor{CoCurveShade!20}\underline{69.9} & \cellcolor{CoCurveShade!20}\underline{54.1} & \cellcolor{CoCurveShade!20}\underline{52.4} & \cellcolor{CoCurveShade!32}\textbf{42.2} & \cellcolor{CoCurveShade!7}50.4 & \cellcolor{CoCurveShade!20}\underline{33.1} & \cellcolor{CoCurveShade!20}\underline{30.8} & \cellcolor{CoCurveShade!32}\textbf{73.8} & \cellcolor{CoCurveShade!32}\textbf{48.3} & \cellcolor{CoCurveShade!32}\textbf{38.7} & \cellcolor{CoCurveShade!20}\underline{47.6} & \cellcolor{CoCurveShade!20}\underline{31.8} & \cellcolor{edgeorange!20}\underline{47.8}\se{0.34} \\
 &  & NF4 & \cellcolor{CoCurveShade!20}\underline{29.2} & \cellcolor{CoCurveShade!20}\underline{42.2} & \cellcolor{CoCurveShade!20}\underline{44.3} & \cellcolor{CoCurveShade!20}\underline{68.8} & \cellcolor{CoCurveShade!20}\underline{52.9} & \cellcolor{CoCurveShade!20}\underline{53.7} & \cellcolor{CoCurveShade!32}\textbf{42.3} & \cellcolor{CoCurveShade!7}\underline{51.0} & \cellcolor{CoCurveShade!20}\underline{31.4} & \cellcolor{CoCurveShade!20}\underline{29.2} & \cellcolor{CoCurveShade!20}\underline{71.0} & \cellcolor{CoCurveShade!32}\textbf{48.0} & \cellcolor{CoCurveShade!32}\textbf{37.9} & \cellcolor{CoCurveShade!20}\underline{47.7} & \cellcolor{CoCurveShade!20}\underline{31.4} & \cellcolor{edgeorange!20}\underline{47.1}\se{0.34} \\
 & \multirow{2}{*}{SlimLLM} & INT8 & \cellcolor{CoCurveShade!7}80.0 & \cellcolor{CoCurveShade!7}102.1 & \cellcolor{CoCurveShade!7}115.3 & \cellcolor{CoCurveShade!7}68.9 & \cellcolor{CoCurveShade!7}37.4 & \cellcolor{CoCurveShade!7}52.2 & \cellcolor{CoCurveShade!7}35.7 & \cellcolor{CoCurveShade!20}\underline{52.7} & \cellcolor{CoCurveShade!7}28.6 & \cellcolor{CoCurveShade!7}28.2 & \cellcolor{CoCurveShade!7}62.3 & \cellcolor{CoCurveShade!7}36.2 & \cellcolor{CoCurveShade!7}28.9 & \cellcolor{CoCurveShade!7}36.7 & \cellcolor{CoCurveShade!7}30.5 & \cellcolor{edgeorange!7}41.5\se{0.34} \\
 &  & NF4 & \cellcolor{CoCurveShade!7}80.0 & \cellcolor{CoCurveShade!7}101.5 & \cellcolor{CoCurveShade!7}114.4 & \cellcolor{CoCurveShade!7}67.8 & \cellcolor{CoCurveShade!7}36.9 & \cellcolor{CoCurveShade!7}51.7 & \cellcolor{CoCurveShade!7}36.4 & \cellcolor{CoCurveShade!32}\textbf{53.1} & \cellcolor{CoCurveShade!7}28.2 & \cellcolor{CoCurveShade!7}28.8 & \cellcolor{CoCurveShade!7}62.3 & \cellcolor{CoCurveShade!7}36.7 & \cellcolor{CoCurveShade!7}29.3 & \cellcolor{CoCurveShade!7}37.1 & \cellcolor{CoCurveShade!7}30.4 & \cellcolor{edgeorange!7}41.6\se{0.34} \\
 & \multirow{2}{*}{\textbf{CoCurve}} & INT8 & \cellcolor{CoCurveShade!32}\textbf{21.7} & \cellcolor{CoCurveShade!32}\textbf{36.9} & \cellcolor{CoCurveShade!32}\textbf{32.3} & \cellcolor{CoCurveShade!32}\textbf{70.7} & \cellcolor{CoCurveShade!32}\textbf{55.2} & \cellcolor{CoCurveShade!32}\textbf{55.9} & \cellcolor{CoCurveShade!20}\underline{40.4} & \cellcolor{CoCurveShade!32}\textbf{54.5} & \cellcolor{CoCurveShade!32}\textbf{35.1} & \cellcolor{CoCurveShade!32}\textbf{35.4} & \cellcolor{CoCurveShade!20}\underline{72.8} & \cellcolor{CoCurveShade!20}\underline{46.7} & \cellcolor{CoCurveShade!20}\underline{37.0} & \cellcolor{CoCurveShade!32}\textbf{50.3} & \cellcolor{CoCurveShade!32}\textbf{33.3} & \cellcolor{edgeorange!32}\textbf{48.9}\se{0.35} \\
 &  & NF4 & \cellcolor{CoCurveShade!32}\textbf{23.1} & \cellcolor{CoCurveShade!32}\textbf{38.5} & \cellcolor{CoCurveShade!32}\textbf{34.8} & \cellcolor{CoCurveShade!32}\textbf{69.7} & \cellcolor{CoCurveShade!32}\textbf{54.1} & \cellcolor{CoCurveShade!32}\textbf{54.7} & \cellcolor{CoCurveShade!20}\underline{37.8} & \cellcolor{CoCurveShade!7}\underline{51.0} & \cellcolor{CoCurveShade!32}\textbf{31.7} & \cellcolor{CoCurveShade!32}\textbf{32.6} & \cellcolor{CoCurveShade!32}\textbf{72.7} & \cellcolor{CoCurveShade!20}\underline{45.7} & \cellcolor{CoCurveShade!20}\underline{36.4} & \cellcolor{CoCurveShade!32}\textbf{48.7} & \cellcolor{CoCurveShade!32}\textbf{32.4} & \cellcolor{edgeorange!32}\textbf{47.3}\se{0.34} \\
\midrule
\multirow{6}{*}{40\%} & \multirow{2}{*}{LLM-Pruner} & INT8 & \cellcolor{CoCurveShade!20}\underline{101.3} & \cellcolor{CoCurveShade!20}\underline{162.9} & \cellcolor{CoCurveShade!20}\underline{124.9} & \cellcolor{CoCurveShade!32}\textbf{65.3} & \cellcolor{CoCurveShade!32}\textbf{42.3} & \cellcolor{CoCurveShade!32}\textbf{53.2} & \cellcolor{CoCurveShade!32}\textbf{32.8} & \cellcolor{CoCurveShade!20}\underline{40.3} & \cellcolor{CoCurveShade!20}\underline{27.8} & \cellcolor{CoCurveShade!32}\textbf{28.2} & \cellcolor{CoCurveShade!32}\textbf{65.3} & \cellcolor{CoCurveShade!20}\underline{37.0} & \cellcolor{CoCurveShade!32}\textbf{32.3} & \cellcolor{CoCurveShade!20}\underline{40.1} & \cellcolor{CoCurveShade!32}\textbf{29.6} & \cellcolor{edgeorange!32}\textbf{41.2}\se{0.34} \\
 &  & NF4 & \cellcolor{CoCurveShade!20}\underline{116.1} & \cellcolor{CoCurveShade!20}\underline{194.2} & \cellcolor{CoCurveShade!20}\underline{138.8} & \cellcolor{CoCurveShade!32}\textbf{64.8} & \cellcolor{CoCurveShade!32}\textbf{42.2} & \cellcolor{CoCurveShade!20}\underline{52.4} & \cellcolor{CoCurveShade!32}\textbf{33.9} & \cellcolor{CoCurveShade!20}\underline{39.2} & \cellcolor{CoCurveShade!32}\textbf{27.7} & \cellcolor{CoCurveShade!20}\underline{25.6} & \cellcolor{CoCurveShade!20}\underline{61.8} & \cellcolor{CoCurveShade!20}\underline{36.4} & \cellcolor{CoCurveShade!32}\textbf{31.3} & \cellcolor{CoCurveShade!20}\underline{39.4} & \cellcolor{CoCurveShade!32}\textbf{29.1} & \cellcolor{edgeorange!32}\textbf{40.3}\se{0.33} \\
 & \multirow{2}{*}{SlimLLM} & INT8 & \cellcolor{CoCurveShade!7}187.8 & \cellcolor{CoCurveShade!7}219.6 & \cellcolor{CoCurveShade!7}222.5 & \cellcolor{CoCurveShade!7}61.5 & \cellcolor{CoCurveShade!7}30.9 & \cellcolor{CoCurveShade!7}\underline{50.2} & \cellcolor{CoCurveShade!7}28.8 & \cellcolor{CoCurveShade!7}39.8 & \cellcolor{CoCurveShade!7}23.6 & \cellcolor{CoCurveShade!7}25.0 & \cellcolor{CoCurveShade!7}53.1 & \cellcolor{CoCurveShade!7}32.7 & \cellcolor{CoCurveShade!7}27.0 & \cellcolor{CoCurveShade!7}32.9 & \cellcolor{CoCurveShade!7}27.7 & \cellcolor{edgeorange!7}36.1\se{0.33} \\
 &  & NF4 & \cellcolor{CoCurveShade!7}192.8 & \cellcolor{CoCurveShade!7}272.0 & \cellcolor{CoCurveShade!7}242.6 & \cellcolor{CoCurveShade!7}61.4 & \cellcolor{CoCurveShade!7}30.1 & \cellcolor{CoCurveShade!32}\textbf{53.7} & \cellcolor{CoCurveShade!7}29.1 & \cellcolor{CoCurveShade!32}\textbf{40.7} & \cellcolor{CoCurveShade!7}23.6 & \cellcolor{CoCurveShade!7}25.4 & \cellcolor{CoCurveShade!7}49.5 & \cellcolor{CoCurveShade!7}32.0 & \cellcolor{CoCurveShade!7}27.0 & \cellcolor{CoCurveShade!7}33.2 & \cellcolor{CoCurveShade!7}27.7 & \cellcolor{edgeorange!7}36.1\se{0.33} \\
 & \multirow{2}{*}{\textbf{CoCurve}} & INT8 & \cellcolor{CoCurveShade!32}\textbf{57.6} & \cellcolor{CoCurveShade!32}\textbf{105.4} & \cellcolor{CoCurveShade!32}\textbf{75.0} & \cellcolor{CoCurveShade!20}\underline{63.9} & \cellcolor{CoCurveShade!20}\underline{40.6} & \cellcolor{CoCurveShade!7}\underline{50.2} & \cellcolor{CoCurveShade!20}\underline{30.1} & \cellcolor{CoCurveShade!32}\textbf{41.2} & \cellcolor{CoCurveShade!32}\textbf{29.9} & \cellcolor{CoCurveShade!20}\underline{28.0} & \cellcolor{CoCurveShade!20}\underline{62.2} & \cellcolor{CoCurveShade!32}\textbf{41.2} & \cellcolor{CoCurveShade!20}\underline{32.2} & \cellcolor{CoCurveShade!32}\textbf{43.2} & \cellcolor{CoCurveShade!20}\underline{29.2} & \cellcolor{edgeorange!20}\underline{41.0}\se{0.34} \\
 &  & NF4 & \cellcolor{CoCurveShade!32}\textbf{66.9} & \cellcolor{CoCurveShade!32}\textbf{140.8} & \cellcolor{CoCurveShade!32}\textbf{86.7} & \cellcolor{CoCurveShade!20}\underline{62.2} & \cellcolor{CoCurveShade!20}\underline{39.8} & \cellcolor{CoCurveShade!7}50.9 & \cellcolor{CoCurveShade!20}\underline{30.2} & \cellcolor{CoCurveShade!7}38.1 & \cellcolor{CoCurveShade!20}\underline{27.4} & \cellcolor{CoCurveShade!32}\textbf{28.4} & \cellcolor{CoCurveShade!32}\textbf{62.1} & \cellcolor{CoCurveShade!32}\textbf{40.1} & \cellcolor{CoCurveShade!20}\underline{31.1} & \cellcolor{CoCurveShade!32}\textbf{41.9} & \cellcolor{CoCurveShade!20}\underline{28.6} & \cellcolor{edgeorange!20}\underline{40.1}\se{0.34} \\
\bottomrule\end{tabular}
}
\end{table}

\FloatBarrier
\subsection{Inference Efficiency}

\label{app:efficiency}
\begin{table}[H]
\centering\fontsize{6.15}{6.70}\selectfont
\caption{\textbf{Inference efficiency of every timed method at $20\%$ and $40\%$.}
Llama-3.1-8B-Instruct, physical removal, one isolated A100-80GB.}
\label{tab:efficiency-band}
\adjustbox{width=.78\textwidth}{
\setlength{\tabcolsep}{2.75pt}\renewcommand{\arraystretch}{.96}
\begin{tabular}{@{}lrrrr!{\color{black!35}\vrule width .5pt}rrrr@{}}
\toprule
 & \multicolumn{4}{c}{$\boldsymbol{\rho=20\%}$} & \multicolumn{4}{c}{$\boldsymbol{\rho=40\%}$} \\
\cmidrule(lr){2-5}\cmidrule(lr){6-9}
\textbf{Method} & \textbf{Red.}$\uparrow$ & \textbf{Pref.}$\uparrow$ & \textbf{Dec.}$\uparrow$ & \textbf{Peak}$\downarrow$ & \textbf{Red.}$\uparrow$ & \textbf{Pref.}$\uparrow$ & \textbf{Dec.}$\uparrow$ & \textbf{Peak}$\downarrow$ \\
\midrule
\rowcolor{headgray} \textbf{Dense} & 0.00 & 1.000 & 1.000 & 17.18 & 0.00 & 1.000 & 1.000 & 17.18\\
\midrule
Random & 17.44 & 1.196 & 1.006 & 14.31 & 34.89 & 1.462 & 1.021 & 11.44\\
Magnitude & 17.41 & 1.167 & 1.011 & 14.38 & 34.89 & 1.389 & 1.038 & 11.58\\
Wanda-sp & 17.41 & 1.224 & 0.996 & 14.25 & 34.82 & 1.468 & 1.009 & 11.45\\
FLAP & 17.45 & 1.231 & 0.976 & 14.23 & 34.86 & 1.478 & 1.028 & 11.43\\
LLM-Pruner & 17.44 & 1.196 & 1.008 & 14.31 & 34.84 & 1.430 & 1.031 & 11.51\\
SlimGPT & 17.45 & 1.233 & 0.989 & 14.22 & 34.77 & 1.483 & 0.990 & 11.43\\
SlimLLM & 17.52 & 1.219 & 1.021 & 14.26 & 34.83 & 1.517 & 1.010 & 11.37\\
\textbf{CoCurve} & 17.40 & 1.211 & 1.020 & 14.29 & 34.81 & 1.415 & 0.998 & 11.53\\
\bottomrule
\end{tabular}
}
\vspace{-1pt}
\parbox{\textwidth}{\fontsize{5.45}{5.95}\selectfont\emph{Note.}
Param. is parameter reduction (\%); Pref./Dec. are matched-dense throughput ratios; peak is generation GB.}
\end{table}

\Cref{tab:efficiency} reports absolute throughput and peak memory, alongside relative gains from physical slicing.
At fixed $20\%$ and $40\%$ budgets on Llama-3.1-8B-Instruct, the matched-method table separates selection quality from system effects by timing every physically sliced mask under the same prompts, repeats, and isolated A100 protocol (\cref{tab:efficiency-band}).
Decode remains sensitive to the selected widths, so the parameter budget alone does not determine kernel efficiency; the isolated measurements provide the quantitative systems comparison.
\Cref{fig:quality-efficiency-matrix} extends the same trade-off view across model scales.

At each pruning ratio, all methods realize essentially the same parameter reduction and occupy narrow systems bands: prefill ranges from $1.167$--$1.233\times$ at $20\%$ and $1.389$--$1.517\times$ at $40\%$, while peak memory spans only $14.22$--$14.38$ and $11.37$--$11.58$ GB, respectively.
The quality ordering is therefore not explained by a more favorable physical footprint for CoCurve; selection improves which structure is retained, while slicing converts any such structure into the measured deployment gain.

\begin{figure}[H]
\centering
\IfFileExists{figures/fig_quality_efficiency_matrix.pdf}{%
\includegraphics[width=.98\textwidth]{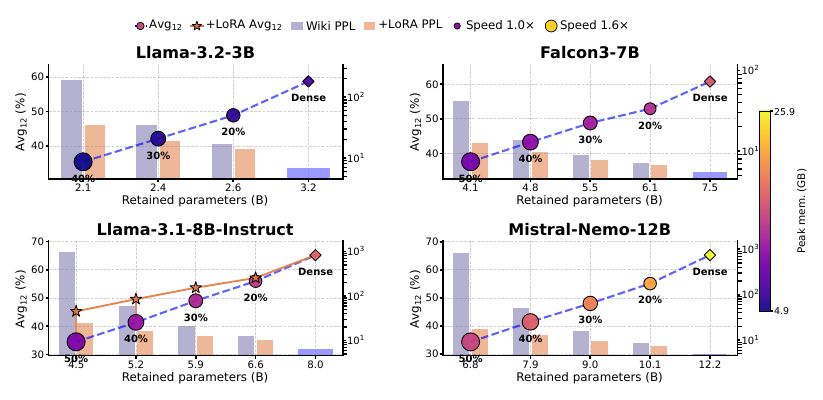}}{%
\pendingpanel{52mm}{A full-width $2\times2$ quality--efficiency frontier grid over the $3$B, $7$B, $8$B, and $12$B models.}}
\caption{\textbf{Cross-model quality--efficiency frontiers.}
The recovery stars appear only for the $8$B model, whose post-recovery Avg$_{12}$ was measured; every panel reports before/after-recovery perplexity.}
\label{fig:quality-efficiency-matrix}
\end{figure}

\paragraph{Selection cost and amortization.}
\label{app:selection-cost}

Complete selection cost includes the reusable graph build, exact breakpoint enumeration, and held-out scoring of every distinct candidate mask; \cref{tab:selection-cost} reports all three rather than graph construction alone.
For the 8B model, the five-ratio sweep evaluates $1{,}362$ distinct candidates and takes $17.4$ recorded end-to-end hours, or $3.5$ hours per compiled operating point after amortization.
The same $\Hmat$ serves all ratios and interaction strengths, and only a $2.3$ MB Gram matrix is retained.
Moreover, $768$ independent single-unit sweeps replace the $294{,}528$ pairwise interventions of a direct pair-ablation design; this counts model evaluations rather than Gram arithmetic.
The sweeps can be sharded without synchronization, reducing wall-clock but not aggregate compute.
We document native pipeline cost rather than rank methods by wall-clock because published implementations combine different amounts of scoring, reconstruction, and recovery.
The quality side of this trade-off is measured at the same operating points in \cref{fig:pareto} and \cref{tab:flagship}: on the headline 8B model the reusable build supports the $19.0\%$ and $42.8\%$ WikiText reductions at $30\%$ and $40\%$ over the strongest matched baseline.

\begin{table}[htbp]
\centering\small
\caption{\textbf{Complete CoCurve selection-cost accounting over the $10$--$50\%$ sweep.}
Path time includes exact CPU breakpoint enumeration and GPU scoring of each distinct mask on $16$ held-out sequences; build time includes the $128$-sequence single-unit feature pass and Gram construction.
The 3B/8B runs use one RTX 6000 Ada and the packed-NF4 70B run one A100-80GB, so the measurements document cost rather than a cross-device speed ranking.}
\label{tab:selection-cost}
\adjustbox{max width=.93\textwidth}{\setlength{\tabcolsep}{5.5pt}\renewcommand{\arraystretch}{1.05}
\begin{tabular}{@{}lrrrrrl@{}}\toprule
\textbf{Model} & \textbf{Units} & \textbf{Build (h)} & \textbf{Candidates} & \textbf{Paths (h)} & \textbf{Total (h)} & \textbf{Stored $\Hmat$} \\
\midrule
Llama-3.2-3B & 672 & 6.0 & 1,310 & 2.3 & 8.3 & 1.8 MB \\
Llama-3.1-8B-Instruct & 768 & 12.6 & 1,362 & 4.8 & 17.4 & 2.3 MB \\
Llama-3.1-70B-Instruct (NF4) & 1,920 & 43.1 & 2,361 & 19.7 & 62.8 & 15 MB \\
\bottomrule\end{tabular}
}
\end{table}

\FloatBarrier
\subsection{Vision--Language Results}
\label{app:vlm-results}
\subsubsection{Pruning-Only Transfer}

\Cref{tab:vlm-full} resolves all seven tasks for three models at $20/30\%$ pruning.
Within each ratio block, the adjacent ``diagonal only'' and full CoCurve rows isolate the off-diagonal contribution while preserving the identical model, data, guard, and budget.
\Appref{app:vlm-recovery} separately tests whether one identical LoRA stage preserves recoverability at these operating points.

\FloatBarrier
{\fontsize{6.0}{6.65}\selectfont
\setlength{\tabcolsep}{1.55pt}\renewcommand{\arraystretch}{.90}
\begin{longtable}{@{}>{\centering\arraybackslash}m{2.1em}@{\hspace{2pt}}l>{\centering\arraybackslash}m{3.65em}!{\color{black!35}\vrule width .5pt}>{\centering\arraybackslash}m{3.65em}>{\centering\arraybackslash}m{3.65em}>{\centering\arraybackslash}m{3.65em}>{\centering\arraybackslash}m{3.65em}>{\centering\arraybackslash}m{3.65em}>{\centering\arraybackslash}m{3.65em}>{\centering\arraybackslash}m{3.65em}!{\color{black!35}\vrule width .5pt}>{\centering\arraybackslash}m{4.4em}>{\centering\arraybackslash}m{3.0em}@{}}
\caption{\textbf{Complete vision--language results by benchmark.}
Three representative models at $20\%$ and $30\%$; ablations are excluded from method ranking.}
\label{tab:vlm-full}\\
\toprule
\multirow{2}{*}{\textbf{$\rho$}} & \multirow{2}{*}{\textbf{Method}} & \multirow{2}{*}{\shortstack{\textbf{Caption}\\\textbf{PPL}$\downarrow$}} & \multicolumn{7}{c}{\textbf{Downstream accuracy} $\uparrow$} & \multirow{2}{*}{\textbf{Avg$_7$}} & \multirow{2}{*}{\textbf{Ret.}} \\
\cmidrule(lr){4-10}
 & & & \textbf{MMB} & \textbf{SEED} & \textbf{SQA} & \textbf{AI2D} & \textbf{MMStar} & \textbf{POPE} & \textbf{MME} & & \\
\midrule
\endfirsthead
\toprule
\multirow{2}{*}{\textbf{$\rho$}} & \multirow{2}{*}{\textbf{Method}} & \multirow{2}{*}{\shortstack{\textbf{Caption}\\\textbf{PPL}$\downarrow$}} & \multicolumn{7}{c}{\textbf{Downstream accuracy} $\uparrow$} & \multirow{2}{*}{\textbf{Avg$_7$}} & \multirow{2}{*}{\textbf{Ret.}} \\
\cmidrule(lr){4-10}
 & & & \textbf{MMB} & \textbf{SEED} & \textbf{SQA} & \textbf{AI2D} & \textbf{MMStar} & \textbf{POPE} & \textbf{MME} & & \\
\midrule
\endhead
\midrule
\rowcolor{headgray}\multicolumn{12}{c}{\textbf{Qwen2.5-VL-7B}} \\*
\midrule
\rowcolor{headgray}& \textbf{Dense reference} & 20.53 & 86.7 & 70.5 & 88.5 & 84.5 & 63.8 & 89.6 & 87.3 & 81.6\se{0.45} &  \\
\multirow{8}{*}{20\%} & Random & \cellcolor{CoCurveShade!21}{16.34} & \cellcolor{CoCurveShade!11}{71.4} & \cellcolor{CoCurveShade!14}{61.2} & \cellcolor{CoCurveShade!11}{70.1} & \cellcolor{CoCurveShade!11}{61.9} & \cellcolor{CoCurveShade!14}{43.5} & \cellcolor{CoCurveShade!14}{80.1} & \cellcolor{CoCurveShade!14}{62.3} & \cellcolor{edgeorange!14}{64.4}\se{0.56} & 65\% \\
 & Magnitude & \cellcolor{CoCurveShade!32}{\textbf{14.51}} & \cellcolor{CoCurveShade!18}{79.9} & \cellcolor{CoCurveShade!25}{65.1} & \cellcolor{CoCurveShade!18}{73.1} & \cellcolor{CoCurveShade!18}{69.9} & \cellcolor{CoCurveShade!28}{\underline{53.6}} & \cellcolor{CoCurveShade!25}{87.3} & \cellcolor{CoCurveShade!32}{\textbf{81.7}} & \cellcolor{edgeorange!21}{72.9}\se{0.52} & 83\% \\
 & Wanda-sp & \cellcolor{CoCurveShade!18}{17.43} & \cellcolor{CoCurveShade!14}{73.8} & \cellcolor{CoCurveShade!11}{56.7} & \cellcolor{CoCurveShade!14}{72.0} & \cellcolor{CoCurveShade!14}{63.5} & \cellcolor{CoCurveShade!11}{43.2} & \cellcolor{CoCurveShade!11}{74.4} & \cellcolor{CoCurveShade!11}{56.5} & \cellcolor{edgeorange!11}{62.9}\se{0.56} & 62\% \\
 & FLAP & \cellcolor{CoCurveShade!28}{\underline{15.17}} & \cellcolor{CoCurveShade!21}{80.4} & \cellcolor{CoCurveShade!21}{63.4} & \cellcolor{CoCurveShade!28}{\underline{78.6}} & \cellcolor{CoCurveShade!21}{70.5} & \cellcolor{CoCurveShade!18}{50.6} & \cellcolor{CoCurveShade!32}{\textbf{88.6}} & \cellcolor{CoCurveShade!28}{\underline{80.9}} & \cellcolor{edgeorange!25}{73.3}\se{0.51} & 83\% \\
\bottomrule
\end{longtable}

}

Across all six model--ratio blocks, full CoCurve has the highest Avg$_7$ among the pruning methods.
Relative to the strongest external baseline, the margin is $1.1$--$2.7$ points at $20\%$ and $1.1$--$10.7$ at $30\%$; relative to its diagonal-only counterpart, the edge term adds $0.9$--$7.6$ points across the full grid.
The consistent ordering across three VLM architectures shows that the interaction signal transfers beyond text-only models, while the growing $30\%$ margins identify where modeling those interactions matters most.

\subsubsection{Recovery with a Common Training Recipe}
\label{app:vlm-recovery}

We recover fixed Qwen2.5-VL-7B, Qwen3-VL-8B, and InternVL3-8B masks from $20\%$ through $50\%$ with one recipe for all methods and the same seven-benchmark harness.
Across the resulting $12$ model--ratio blocks, CoCurve ranks first or second on Avg$_7$ in $11$, including the best $50\%$ results on Qwen2.5-VL-7B and InternVL3-8B and second place on Qwen3-VL-8B (\cref{tab:vlm-recovery}).
The common training stage therefore extends the pruning-only advantage through $50\%$, showing that structure selection and recovery remain complementary under heavier compression.
MMB, SEED, SQA, and MMStar denote MMBench, SEED-Bench, ScienceQA, and MMStar; Avg$_7$ weights the seven benchmarks equally.

\begin{table}[H]
\centering\fontsize{6.0}{6.65}\selectfont
\captionsetup{skip=3pt}
\caption{\textbf{VLM recovery by benchmark.} Caption perplexity and seven downstream tasks under the common LoRA recipe.}
\label{tab:vlm-recovery}
\adjustbox{max width=\textwidth}{
\setlength{\tabcolsep}{2.15pt}\renewcommand{\arraystretch}{0.90}
\begin{tabular}{@{}>{\centering\arraybackslash}m{2.2em}l>{\centering\arraybackslash}m{3.55em}!{\color{black!35}\vrule width .5pt}>{\centering\arraybackslash}m{3.55em}>{\centering\arraybackslash}m{3.55em}>{\centering\arraybackslash}m{3.55em}>{\centering\arraybackslash}m{3.55em}>{\centering\arraybackslash}m{3.55em}>{\centering\arraybackslash}m{3.55em}>{\centering\arraybackslash}m{3.55em}!{\color{black!35}\vrule width .5pt}>{\centering\arraybackslash}m{3.8em}@{}}\toprule
\multirow{2}{*}{\textbf{$\rho$}} & \multirow{2}{*}{\textbf{$+$LoRA method}} & \multirow{2}{*}{\textbf{PPL}$\downarrow$} & \multicolumn{7}{c}{\textbf{Downstream accuracy $\uparrow$}} & \multirow{2}{*}{\textbf{Avg$_7$}} \\
\cmidrule(lr){4-10}
 & & & \textbf{MMB} & \textbf{SEED} & \textbf{SQA} & \textbf{AI2D} & \textbf{{\fontsize{5.8}{6.2}\selectfont MMStar}} & \textbf{POPE} & \textbf{MME} & \\
\midrule
\rowcolor{headgray}\multicolumn{11}{c}{\textbf{Qwen2.5-VL-7B}\quad \textbf{Dense:} PPL 20.53; Avg$_7$ 81.6}  \\
\multirow{5}{*}{20\%} & Magnitude & \cellcolor{CoCurveShade!32}\textbf{11.9} & \cellcolor{CoCurveShade!13}80.7 & \cellcolor{CoCurveShade!26}\underline{67.5} & \cellcolor{CoCurveShade!13}74.0 & \cellcolor{CoCurveShade!13}71.5 & \cellcolor{CoCurveShade!32}\textbf{55.5} & \cellcolor{CoCurveShade!32}\textbf{90.0} & \cellcolor{CoCurveShade!32}\textbf{86.3} & \cellcolor{edgeorange!13}75.1\se{0.50} \\
 & Wanda-sp & \cellcolor{CoCurveShade!13}15.2 & \cellcolor{CoCurveShade!7}75.5 & \cellcolor{CoCurveShade!7}59.1 & \cellcolor{CoCurveShade!7}71.0 & \cellcolor{CoCurveShade!7}65.3 & \cellcolor{CoCurveShade!7}48.9 & \cellcolor{CoCurveShade!7}88.3 & \cellcolor{CoCurveShade!7}76.6 & \cellcolor{edgeorange!7}69.2\se{0.53} \\
\bottomrule\end{tabular}
}
\end{table}

InternVL3-8B optimization trajectories are shown beside the LLM trajectories in \cref{fig:recovery-trajectories}; held-out VLM results remain the criterion in \cref{tab:vlm-recovery}.

\subsubsection{VLM Quality--Efficiency Frontier}
\label{app:vlm-efficiency}
\Cref{fig:vlm-quality-efficiency-matrix} combines selection-only and recovered quality with isolated-A100 measurements after physical slicing, exposing the VLM quality, throughput, and memory trade-off under one protocol.
At $50\%$ pruning, prefill throughput rises by $1.22$--$1.45\times$ and peak memory falls by $42$--$43\%$ across the three architectures; these systems gains accompany rather than replace the recovered-quality comparison.

\begin{figure}[H]
\centering
\includegraphics[width=.98\textwidth]{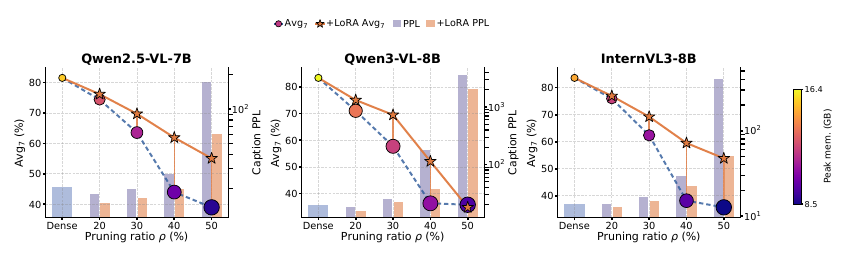}
\caption{\textbf{VLM quality--efficiency frontiers.}
Each panel follows one model from dense through $50\%$ pruning, combining Avg$_7$, caption perplexity before and after common recovery, prefill speedup, and isolated-A100 peak memory.}
\label{fig:vlm-quality-efficiency-matrix}
\end{figure}

\FloatBarrier
\subsection{Additional Behavioral Checks}
\label{app:behavioral}
\subsubsection{Fidelity to the Dense Model}
\label{app:fidelity}
Because selection minimizes token-level KL to the dense model, agreement with dense predictions is its direct behavioral test; downstream accuracy asks whether that fidelity is useful.
We measure the fraction of shared items on which the pruned and dense models choose the same normalized answer.
Across $20$--$40\%$ pruning, \method leads all twelve capability-group comparisons and all three Avg$_{12}$ columns, while ranking first or second on $35$ of $36$ individual task comparisons.
\Cref{tab:dense-agreement} therefore tests the objective at the decision level across the main operating range without introducing another benchmark or calibration signal.

\begin{table}[H]
\vspace{-6pt}
\centering\footnotesize
\captionsetup{skip=1pt}
\caption{\textbf{Agreement with the dense model's prediction} (\%, Llama-3.1-8B-Instruct).
Columns report all twelve tasks in \appref{app:datasets}; Avg$_{12}$ weights them equally.}
\label{tab:dense-agreement}
\adjustbox{max width=\textwidth}{
\setlength{\tabcolsep}{2.2pt}\renewcommand{\arraystretch}{.92}
\begin{tabular}{@{}cl*{13}{c}@{}}\toprule
\textbf{$\rho$} & \textbf{Method} & \textbf{PIQA} & \textbf{HS} & \textbf{WG} & \textbf{CSQA} & \textbf{A-e} & \textbf{A-c} & \textbf{OBQA} & \textbf{BoolQ} & \textbf{R-m} & \textbf{R-h} & \textbf{QuAIL} & \textbf{MMLU} & \textbf{Avg$_{12}$} \\
\midrule
\multirow{8}{*}{20\%} & Random & \cellcolor{CoCurveShade!11}80.1 & \cellcolor{CoCurveShade!7}69.9 & \cellcolor{CoCurveShade!11}66.7 & \cellcolor{CoCurveShade!7}63.2 & \cellcolor{CoCurveShade!7}61.8 & \cellcolor{CoCurveShade!7}54.9 & \cellcolor{CoCurveShade!11}61.8 & \cellcolor{CoCurveShade!14}67.8 & \cellcolor{CoCurveShade!7}63.4 & \cellcolor{CoCurveShade!11}63.8 & \cellcolor{CoCurveShade!18}69.4 & \cellcolor{CoCurveShade!11}58.5 & \cellcolor{edgeorange!7}65.1 \\
 & Magnitude & \cellcolor{CoCurveShade!7}78.6 & \cellcolor{CoCurveShade!11}71.0 & \cellcolor{CoCurveShade!7}66.5 & \cellcolor{CoCurveShade!7}63.2 & \cellcolor{CoCurveShade!11}63.4 & \cellcolor{CoCurveShade!11}55.3 & \cellcolor{CoCurveShade!7}60.2 & \cellcolor{CoCurveShade!11}66.8 & \cellcolor{CoCurveShade!14}65.4 & \cellcolor{CoCurveShade!14}65.3 & \cellcolor{CoCurveShade!21}70.2 & \cellcolor{CoCurveShade!7}58.0 & \cellcolor{edgeorange!11}65.3 \\
 & Wanda-sp & \cellcolor{CoCurveShade!15}85.0 & \cellcolor{CoCurveShade!14}74.2 & \cellcolor{CoCurveShade!21}70.4 & \cellcolor{CoCurveShade!24}73.7 & \cellcolor{CoCurveShade!21}74.3 & \cellcolor{CoCurveShade!25}66.2 & \cellcolor{CoCurveShade!32}\textbf{72.0} & \cellcolor{CoCurveShade!25}74.4 & \cellcolor{CoCurveShade!25}70.3 & \cellcolor{CoCurveShade!21}67.3 & \cellcolor{CoCurveShade!14}69.2 & \cellcolor{CoCurveShade!21}65.1 & \cellcolor{edgeorange!21}71.8 \\
 & FLAP & \cellcolor{CoCurveShade!15}85.0 & \cellcolor{CoCurveShade!21}75.9 & \cellcolor{CoCurveShade!18}69.0 & \cellcolor{CoCurveShade!11}70.0 & \cellcolor{CoCurveShade!25}74.4 & \cellcolor{CoCurveShade!18}64.9 & \cellcolor{CoCurveShade!25}71.6 & \cellcolor{CoCurveShade!21}72.6 & \cellcolor{CoCurveShade!18}66.8 & \cellcolor{CoCurveShade!18}65.4 & \cellcolor{CoCurveShade!7}65.9 & \cellcolor{CoCurveShade!18}64.5 & \cellcolor{edgeorange!18}70.5 \\
 & LLM-Pruner & \cellcolor{CoCurveShade!24}86.1 & \cellcolor{CoCurveShade!28}\underline{81.3} & \cellcolor{CoCurveShade!28}\underline{73.6} & \cellcolor{CoCurveShade!28}\underline{73.8} & \cellcolor{CoCurveShade!14}71.2 & \cellcolor{CoCurveShade!21}65.7 & \cellcolor{CoCurveShade!14}68.6 & \cellcolor{CoCurveShade!28}\underline{83.4} & \cellcolor{CoCurveShade!32}\textbf{75.9} & \cellcolor{CoCurveShade!32}\textbf{73.7} & \cellcolor{CoCurveShade!28}\underline{76.3} & \cellcolor{CoCurveShade!25}66.7 & \cellcolor{edgeorange!28}\underline{74.7} \\
 & SlimGPT & \cellcolor{CoCurveShade!20}85.3 & \cellcolor{CoCurveShade!18}75.7 & \cellcolor{CoCurveShade!14}67.5 & \cellcolor{CoCurveShade!15}70.2 & \cellcolor{CoCurveShade!18}73.8 & \cellcolor{CoCurveShade!14}64.2 & \cellcolor{CoCurveShade!21}70.8 & \cellcolor{CoCurveShade!7}65.4 & \cellcolor{CoCurveShade!11}64.5 & \cellcolor{CoCurveShade!7}62.7 & \cellcolor{CoCurveShade!11}66.5 & \cellcolor{CoCurveShade!14}62.8 & \cellcolor{edgeorange!14}69.1 \\
 & SlimLLM & \cellcolor{CoCurveShade!28}\underline{86.6} & \cellcolor{CoCurveShade!25}77.3 & \cellcolor{CoCurveShade!25}71.0 & \cellcolor{CoCurveShade!20}72.8 & \cellcolor{CoCurveShade!32}\textbf{77.1} & \cellcolor{CoCurveShade!28}\underline{66.9} & \cellcolor{CoCurveShade!18}70.6 & \cellcolor{CoCurveShade!18}70.0 & \cellcolor{CoCurveShade!21}69.7 & \cellcolor{CoCurveShade!25}68.9 & \cellcolor{CoCurveShade!25}71.5 & \cellcolor{CoCurveShade!28}\underline{66.8} & \cellcolor{edgeorange!25}72.4 \\
 & \textbf{CoCurve} & \cellcolor{CoCurveShade!32}\textbf{88.2} & \cellcolor{CoCurveShade!32}\textbf{83.5} & \cellcolor{CoCurveShade!32}\textbf{74.7} & \cellcolor{CoCurveShade!32}\textbf{76.5} & \cellcolor{CoCurveShade!28}\underline{75.1} & \cellcolor{CoCurveShade!32}\textbf{69.1} & \cellcolor{CoCurveShade!28}\underline{71.8} & \cellcolor{CoCurveShade!32}\textbf{83.5} & \cellcolor{CoCurveShade!28}\underline{75.2} & \cellcolor{CoCurveShade!28}\underline{73.4} & \cellcolor{CoCurveShade!32}\textbf{78.9} & \cellcolor{CoCurveShade!32}\textbf{69.2} & \cellcolor{edgeorange!32}\textbf{76.6} \\
\midrule
\multirow{8}{*}{30\%} & Random & \cellcolor{CoCurveShade!14}70.8 & \cellcolor{CoCurveShade!21}53.9 & \cellcolor{CoCurveShade!11}59.7 & \cellcolor{CoCurveShade!14}51.7 & \cellcolor{CoCurveShade!14}50.5 & \cellcolor{CoCurveShade!14}44.9 & \cellcolor{CoCurveShade!11}52.4 & \cellcolor{CoCurveShade!11}48.0 & \cellcolor{CoCurveShade!25}56.0 & \cellcolor{CoCurveShade!25}56.5 & \cellcolor{CoCurveShade!25}60.7 & \cellcolor{CoCurveShade!11}47.0 & \cellcolor{edgeorange!14}54.3 \\
 & Magnitude & \cellcolor{CoCurveShade!7}62.5 & \cellcolor{CoCurveShade!7}42.0 & \cellcolor{CoCurveShade!7}54.5 & \cellcolor{CoCurveShade!7}42.6 & \cellcolor{CoCurveShade!7}39.1 & \cellcolor{CoCurveShade!7}37.5 & \cellcolor{CoCurveShade!7}43.4 & \cellcolor{CoCurveShade!21}52.5 & \cellcolor{CoCurveShade!7}42.0 & \cellcolor{CoCurveShade!7}46.3 & \cellcolor{CoCurveShade!7}48.1 & \cellcolor{CoCurveShade!7}41.3 & \cellcolor{edgeorange!7}46.0 \\
 & Wanda-sp & \cellcolor{CoCurveShade!21}76.4 & \cellcolor{CoCurveShade!14}52.8 & \cellcolor{CoCurveShade!21}62.4 & \cellcolor{CoCurveShade!21}56.3 & \cellcolor{CoCurveShade!21}55.7 & \cellcolor{CoCurveShade!21}51.4 & \cellcolor{CoCurveShade!18}56.6 & \cellcolor{CoCurveShade!7}46.8 & \cellcolor{CoCurveShade!21}55.6 & \cellcolor{CoCurveShade!21}54.1 & \cellcolor{CoCurveShade!21}57.2 & \cellcolor{CoCurveShade!21}51.7 & \cellcolor{edgeorange!21}56.4 \\
 & FLAP & \cellcolor{CoCurveShade!18}74.1 & \cellcolor{CoCurveShade!25}55.1 & \cellcolor{CoCurveShade!18}60.7 & \cellcolor{CoCurveShade!18}55.9 & \cellcolor{CoCurveShade!18}53.5 & \cellcolor{CoCurveShade!18}47.3 & \cellcolor{CoCurveShade!21}57.2 & \cellcolor{CoCurveShade!18}51.9 & \cellcolor{CoCurveShade!14}49.9 & \cellcolor{CoCurveShade!18}52.3 & \cellcolor{CoCurveShade!11}51.8 & \cellcolor{CoCurveShade!18}49.8 & \cellcolor{edgeorange!18}55.0 \\
 & LLM-Pruner & \cellcolor{CoCurveShade!25}80.3 & \cellcolor{CoCurveShade!28}\underline{68.9} & \cellcolor{CoCurveShade!32}\textbf{67.6} & \cellcolor{CoCurveShade!28}\underline{61.8} & \cellcolor{CoCurveShade!25}59.7 & \cellcolor{CoCurveShade!28}\underline{55.5} & \cellcolor{CoCurveShade!25}58.4 & \cellcolor{CoCurveShade!32}\textbf{75.1} & \cellcolor{CoCurveShade!32}\textbf{62.7} & \cellcolor{CoCurveShade!32}\textbf{64.5} & \cellcolor{CoCurveShade!28}\underline{67.6} & \cellcolor{CoCurveShade!28}\underline{55.4} & \cellcolor{edgeorange!28}\underline{64.8} \\
 & SlimGPT & \cellcolor{CoCurveShade!11}69.6 & \cellcolor{CoCurveShade!11}52.8 & \cellcolor{CoCurveShade!14}60.4 & \cellcolor{CoCurveShade!11}49.5 & \cellcolor{CoCurveShade!11}50.0 & \cellcolor{CoCurveShade!11}44.0 & \cellcolor{CoCurveShade!14}53.6 & \cellcolor{CoCurveShade!14}50.2 & \cellcolor{CoCurveShade!11}46.6 & \cellcolor{CoCurveShade!11}50.8 & \cellcolor{CoCurveShade!14}53.5 & \cellcolor{CoCurveShade!14}47.5 & \cellcolor{edgeorange!11}52.4 \\
 & SlimLLM & \cellcolor{CoCurveShade!28}\underline{80.5} & \cellcolor{CoCurveShade!18}53.6 & \cellcolor{CoCurveShade!25}64.5 & \cellcolor{CoCurveShade!25}60.7 & \cellcolor{CoCurveShade!28}\underline{63.5} & \cellcolor{CoCurveShade!25}54.3 & \cellcolor{CoCurveShade!28}\underline{60.8} & \cellcolor{CoCurveShade!25}63.2 & \cellcolor{CoCurveShade!18}50.6 & \cellcolor{CoCurveShade!14}52.2 & \cellcolor{CoCurveShade!18}55.0 & \cellcolor{CoCurveShade!25}55.0 & \cellcolor{edgeorange!25}59.5 \\
 & \textbf{CoCurve} & \cellcolor{CoCurveShade!32}\textbf{81.0} & \cellcolor{CoCurveShade!32}\textbf{71.3} & \cellcolor{CoCurveShade!28}\underline{67.2} & \cellcolor{CoCurveShade!32}\textbf{64.5} & \cellcolor{CoCurveShade!32}\textbf{64.8} & \cellcolor{CoCurveShade!32}\textbf{58.9} & \cellcolor{CoCurveShade!32}\textbf{62.2} & \cellcolor{CoCurveShade!28}\underline{75.0} & \cellcolor{CoCurveShade!28}\underline{62.4} & \cellcolor{CoCurveShade!28}\underline{62.9} & \cellcolor{CoCurveShade!32}\textbf{71.0} & \cellcolor{CoCurveShade!32}\textbf{58.0} & \cellcolor{edgeorange!32}\textbf{66.6} \\
\midrule
\multirow{8}{*}{40\%} & Random & \cellcolor{CoCurveShade!18}59.8 & \cellcolor{CoCurveShade!21}38.4 & \cellcolor{CoCurveShade!21}54.0 & \cellcolor{CoCurveShade!21}39.7 & \cellcolor{CoCurveShade!21}37.2 & \cellcolor{CoCurveShade!18}34.6 & \cellcolor{CoCurveShade!21}41.4 & \cellcolor{CoCurveShade!18}43.0 & \cellcolor{CoCurveShade!21}41.1 & \cellcolor{CoCurveShade!21}42.3 & \cellcolor{CoCurveShade!25}46.7 & \cellcolor{CoCurveShade!21}40.4 & \cellcolor{edgeorange!21}43.2 \\
 & Magnitude & \cellcolor{CoCurveShade!11}54.4 & \cellcolor{CoCurveShade!7}28.8 & \cellcolor{CoCurveShade!7}49.9 & \cellcolor{CoCurveShade!7}24.2 & \cellcolor{CoCurveShade!7}28.6 & \cellcolor{CoCurveShade!7}28.8 & \cellcolor{CoCurveShade!14}35.8 & \cellcolor{CoCurveShade!21}48.7 & \cellcolor{CoCurveShade!7}28.0 & \cellcolor{CoCurveShade!7}29.2 & \cellcolor{CoCurveShade!14}29.1 & \cellcolor{CoCurveShade!7}28.5 & \cellcolor{edgeorange!11}34.5 \\
 & Wanda-sp & \cellcolor{CoCurveShade!21}59.9 & \cellcolor{CoCurveShade!18}33.2 & \cellcolor{CoCurveShade!14}50.4 & \cellcolor{CoCurveShade!18}37.2 & \cellcolor{CoCurveShade!18}35.4 & \cellcolor{CoCurveShade!21}34.9 & \cellcolor{CoCurveShade!18}37.6 & \cellcolor{CoCurveShade!11}37.8 & \cellcolor{CoCurveShade!18}30.4 & \cellcolor{CoCurveShade!18}33.8 & \cellcolor{CoCurveShade!18}29.2 & \cellcolor{CoCurveShade!18}38.4 & \cellcolor{edgeorange!18}38.2 \\
 & FLAP & \cellcolor{CoCurveShade!14}55.7 & \cellcolor{CoCurveShade!14}32.3 & \cellcolor{CoCurveShade!18}50.9 & \cellcolor{CoCurveShade!14}32.4 & \cellcolor{CoCurveShade!14}31.6 & \cellcolor{CoCurveShade!14}30.3 & \cellcolor{CoCurveShade!11}33.2 & \cellcolor{CoCurveShade!14}41.3 & \cellcolor{CoCurveShade!11}28.3 & \cellcolor{CoCurveShade!11}30.1 & \cellcolor{CoCurveShade!11}25.9 & \cellcolor{CoCurveShade!11}31.9 & \cellcolor{edgeorange!14}35.3 \\
 & LLM-Pruner & \cellcolor{CoCurveShade!28}\underline{71.8} & \cellcolor{CoCurveShade!32}\textbf{53.2} & \cellcolor{CoCurveShade!28}\underline{55.7} & \cellcolor{CoCurveShade!25}45.9 & \cellcolor{CoCurveShade!25}45.4 & \cellcolor{CoCurveShade!25}40.4 & \cellcolor{CoCurveShade!25}44.0 & \cellcolor{CoCurveShade!32}\textbf{65.7} & \cellcolor{CoCurveShade!28}\underline{43.6} & \cellcolor{CoCurveShade!28}\underline{47.3} & \cellcolor{CoCurveShade!28}\underline{53.8} & \cellcolor{CoCurveShade!25}45.2 & \cellcolor{edgeorange!28}\underline{51.0} \\
 & SlimGPT & \cellcolor{CoCurveShade!7}53.2 & \cellcolor{CoCurveShade!11}30.1 & \cellcolor{CoCurveShade!11}50.3 & \cellcolor{CoCurveShade!11}24.9 & \cellcolor{CoCurveShade!11}31.1 & \cellcolor{CoCurveShade!11}29.7 & \cellcolor{CoCurveShade!7}31.6 & \cellcolor{CoCurveShade!7}37.3 & \cellcolor{CoCurveShade!14}28.4 & \cellcolor{CoCurveShade!14}31.8 & \cellcolor{CoCurveShade!7}22.2 & \cellcolor{CoCurveShade!14}35.0 & \cellcolor{edgeorange!7}33.8 \\
 & SlimLLM & \cellcolor{CoCurveShade!25}70.5 & \cellcolor{CoCurveShade!25}44.3 & \cellcolor{CoCurveShade!32}\textbf{58.6} & \cellcolor{CoCurveShade!28}\underline{49.5} & \cellcolor{CoCurveShade!32}\textbf{51.1} & \cellcolor{CoCurveShade!28}\underline{44.5} & \cellcolor{CoCurveShade!28}\underline{49.4} & \cellcolor{CoCurveShade!25}50.1 & \cellcolor{CoCurveShade!25}43.1 & \cellcolor{CoCurveShade!25}44.9 & \cellcolor{CoCurveShade!21}45.3 & \cellcolor{CoCurveShade!28}\underline{46.9} & \cellcolor{edgeorange!25}49.8 \\
 & \textbf{CoCurve} & \cellcolor{CoCurveShade!32}\textbf{72.3} & \cellcolor{CoCurveShade!28}\underline{52.8} & \cellcolor{CoCurveShade!25}54.1 & \cellcolor{CoCurveShade!32}\textbf{50.5} & \cellcolor{CoCurveShade!28}\underline{49.7} & \cellcolor{CoCurveShade!32}\textbf{46.6} & \cellcolor{CoCurveShade!32}\textbf{52.0} & \cellcolor{CoCurveShade!28}\underline{63.1} & \cellcolor{CoCurveShade!32}\textbf{54.1} & \cellcolor{CoCurveShade!32}\textbf{53.9} & \cellcolor{CoCurveShade!32}\textbf{58.1} & \cellcolor{CoCurveShade!32}\textbf{48.8} & \cellcolor{edgeorange!32}\textbf{54.7} \\
\bottomrule\end{tabular}
}
\vspace{-4pt}
\end{table}

Its Avg$_{12}$ agreement margin over the strongest baseline is $1.9$, $1.8$, and $3.7$ points at $20\%$, $30\%$, and $40\%$, respectively.
The objective's token-level fidelity therefore survives at the decision level and remains useful across capability groups, even though no downstream benchmark enters selection.

\subsubsection{Few-Shot Prompting}
\label{app:fewshot}
Five-shot prompting puts subject-matched demonstrations before each MMLU question, changing the inference protocol while keeping the evaluated knowledge benchmark fixed.
The experiment asks whether the ordering found under zero-shot evaluation survives when the pruned model must use information supplied in context.
We treat it as a protocol-robustness test and do not aggregate it into the zero-shot Avg$_{12}$.

\begin{table}[H]
\centering\footnotesize
\caption{\textbf{Five-shot MMLU} (\%).
Subject-dev demonstrations are identical across pruning methods; dense references are gray.
All five available model--ratio comparisons favor \method.}
\label{tab:fewshot-main}
\adjustbox{max width=.82\textwidth}{
\footnotesize\setlength{\tabcolsep}{4.2pt}\renewcommand{\arraystretch}{.96}
\begin{tabular}{@{}ll rrr rrr@{}}
\toprule
 & & \multicolumn{3}{c}{\textbf{Zero-shot}} & \multicolumn{3}{c}{\textbf{Five-shot}} \\
\cmidrule(lr){3-5}\cmidrule(lr){6-8}
\textbf{Model} & \textbf{$\rho$} & \textbf{LLM-P} & \textbf{SlimLLM} & \textbf{CoCurve} & \textbf{LLM-P} & \textbf{SlimLLM} & \textbf{CoCurve} \\
\midrule
\rowcolor{headgray}\multicolumn{8}{@{}l}{\textbf{Llama-3.1-8B-Instruct}; Dense zero/five=47.8/50.8} \\
 & 20\% & 36.8 & \underline{37.2} & \textbf{38.5} & \underline{34.9} & 34.8 & \textbf{37.0} \\
 & 30\% & \underline{31.9} & 30.4 & \textbf{33.4} & \underline{31.3} & 27.0 & \textbf{32.7} \\
 & 40\% & \textbf{29.7} & 27.6 & \underline{29.3} & \underline{28.7} & 26.7 & \textbf{28.9} \\
\addlinespace[2pt]
\rowcolor{headgray}\multicolumn{8}{@{}l}{\textbf{Llama-3.2-3B}; Dense zero/five=41.8/41.5} \\
 & 20\% & \underline{33.5} & 32.0 & \textbf{33.9} & \underline{33.0} & 31.0 & \textbf{33.6} \\
 & 30\% & \underline{29.5} & 28.8 & \textbf{30.0} & \underline{29.3} & 28.4 & \textbf{30.1} \\
\bottomrule
\end{tabular}
}
\end{table}

\Cref{tab:fewshot-main} uses demonstrations drawn from each subject's dev split and identical for every method.
The ordering is the one the zero-shot tables show: \method leads the strongest evaluated baseline in all five comparisons by $0.2$--$2.1$ points.
The advantage the selection buys is therefore not confined to the zero-shot prompt format.

\FloatBarrier
\subsection{Comparison with a Native Smaller Model}
\label{app:deployment-context}

This comparison asks whether pruning can transfer useful inherited capacity from a larger pretrained model into a better memory--capability point than starting with a smaller model.
\Cref{tab:deployment-context} holds the evaluation harness fixed and contrasts a dense bf16 parent, the recovered and quantized CoCurve checkpoint, and a native smaller checkpoint in both precisions.
For model pairs with different tokenizers, the matched claim uses the shared twelve-task suite rather than cross-model perplexity.

\begin{table}[htbp]
\centering\fontsize{6.35}{6.95}\selectfont
\caption{\textbf{Full comparison with a native smaller model.}
The main comparison is summarized in \cref{tab:large-pruned-small-dense}.}
\label{tab:deployment-context}
\adjustbox{max width=\textwidth}{
\setlength{\tabcolsep}{1.65pt}\renewcommand{\arraystretch}{.94}
\begin{tabular}{@{}llcrrrrrrrrr@{}}
\toprule
\textbf{Role} & \textbf{Checkpoint} & \textbf{Prec.} & \textbf{Load}$\downarrow$ & \textbf{Peak}$\downarrow$ & \textbf{Wiki}$\downarrow$ & \textbf{Cmn$_4$}$\uparrow$ & \textbf{Sci$_3$}$\uparrow$ & \textbf{Read$_4$}$\uparrow$ & \textbf{Know$_1$}$\uparrow$ & \textbf{Avg$_{12}$}$\uparrow$ \\
\midrule
\rowcolor{headgray}\multicolumn{11}{c}{\textbf{Native-small memory--capability comparison}} \\
\rowcolor{headgray}Native parent & Mistral-Small-24B & bf16 & 23.57 & 45.02 & 4.7 & 73.5 & 63.4 & 65.9 & 52.3 & 66.7\se{0.33} \\
Native small & Mistral-7B & bf16 & 7.25 & 14.16 & 4.8 & \underline{69.5} & \underline{57.4} & \underline{61.9} & \underline{44.1} & \underline{61.8}\se{0.34} \\
Native small & Mistral-7B & NF4 & 7.25 & 4.51 & 4.9 & 69.6 & 57.0 & 61.9 & 43.9 & 61.7\se{0.34} \\
\textbf{CoCurve} & Mistral-Small-24B, $20\%$ & NF4 & 19.11 & \textbf{12.14} & 6.8 & \textbf{69.6} & \textbf{59.5} & \textbf{64.9} & \textbf{45.5} & \textbf{63.5}\se{0.34} \\
\bottomrule
\end{tabular}
}
\vspace{-1pt}
\parbox{\textwidth}{\fontsize{5.8}{6.4}\selectfont\emph{Note.}
Load: B parameters; Peak: maximum allocated GB in the fixed forward workload.}
\end{table}

All rows use the same task harness, and all memory rows use the same isolated-A100 measurement protocol.
The recovered $20\%$ CoCurve checkpoint uses less peak memory than native Mistral-7B in bf16 while improving every capability group and Avg$_{12}$ by $1.7$ points.
The dense bf16 parent anchors inherited quality, while the native-small bf16/NF4 rows separate structural inheritance from numerical compression.

\clearpage
\FloatBarrier
\section{What the Co-Pruning Edges Reveal}
\label{app:analysis}
We ask a simple interpretability question: what structure does the co-pruning graph expose, and when does that structure matter for pruning?
This part answers it through one mechanism chain.
\Appref{app:mechanism} moves from graph structure to conditional decisions and matched interventions; \appref{app:ablation-extra} isolates the edge term from eligibility and granularity; \appref{app:surrogate} tests whether the quadratic ranks fixed-budget masks; and \appref{app:robustness} asks whether the resulting graph and decisions survive estimator and calibration perturbations.
Together, the sections connect the object CoCurve estimates to the choices it changes and the behavior those choices preserve.

\subsection{From Graph Structure to Decisions and Behavior}

\label{app:mechanism}
\Cref{fig:anatomy-full} shows that the off-diagonal is structured by both depth and module pairing rather than behaving like homogeneous estimation noise.
Viewed as an output-space co-pruning atlas, the graph locates units whose individual perturbations point in similar directions, including relations that cross both layer and module boundaries.
This organization motivates examining units by their aggregate connectivity and then testing whether that relational structure changes actual selections.

\begin{figure}[H]
\centering
\includegraphics[width=\textwidth]{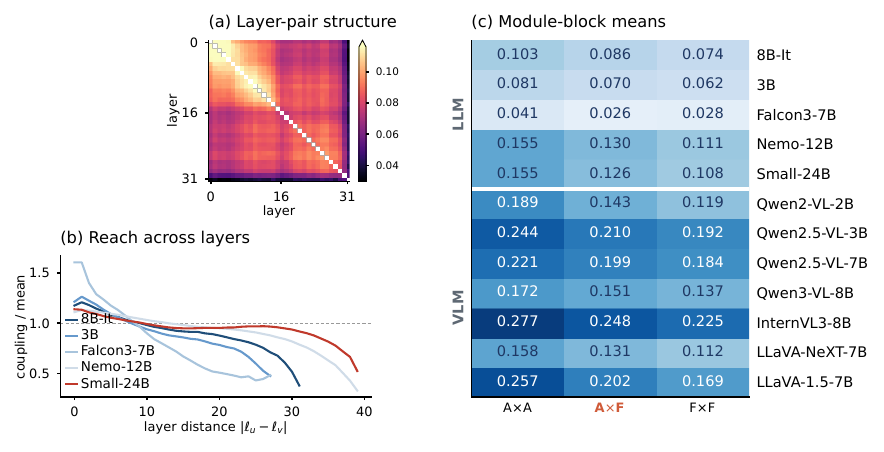}
\caption{\textbf{Extended anatomy of the estimated interaction matrix.}
(a) Layer-pair means on Llama-3.1-8B-Instruct; (b) correlation by layer distance relative to the all-pair mean; (c) module-block means across language and vision--language models.
Within-module means exclude self-correlations.
Near-layer coupling is strongest, yet the slower cross-depth decay and consistently large attention--FFN block (highlighted in orange) show that the estimated graph is neither local-only nor separable by module.}
\label{fig:anatomy-full}
\end{figure}

\textbf{Bridge units.}
With $d_u=\Hmat_{uu}$ and $b_u=\sum_{v\ne u}|\Hmat_{uv}|$, the fixed operational definition is $d_u<Q_{0.40}(d)$ and $b_u>Q_{0.60}(b)$ on every model.
Whether \method or the diagonal selector retains those units is an outcome, not part of the definition.
The intersection occupies only $3.5$--$14.5\%$ of the inventory across all ten language-model matrices (\cref{tab:depth-grading}), making it a selective relational population rather than a broad low-saliency partition.
The first test is consequently decision-level: at an identical greedy prefix, does accumulated coupling change which eligible unit is least costly to remove?

\begin{table}[htbp]
\centering\footnotesize
\caption{\textbf{A real edge-induced decision on Mistral-Nemo-12B.}
Three eligible candidates at step $186$ of the deployed $20\%$ selection.
Scores are cost-normalized and multiplied by $10^{12}$; layer indices are zero-based.
All three satisfy the same guard.}
\label{tab:edge-decision}
\setlength{\tabcolsep}{7pt}
\begin{tabular}{@{}llrrr@{}}\toprule
\textbf{Unit} & \textbf{Layer} & \textbf{Individual} & \textbf{Accumulated edge} & \textbf{Conditional total}\\
\midrule
Attention group 458 & 19 & \textbf{237.005} & 4.926 & 241.932 \\
FFN group 472 & 19 & 239.280 & 2.458 & 241.738 \\
\rowours FFN group 593 & 24 & 239.289 & 2.385 & \textbf{241.674} \\
\bottomrule\end{tabular}

\end{table}

\textbf{A concrete conditional decision.}
\Cref{tab:edge-decision} uses the current, cost-corrected Mistral-Nemo-12B mask.
All three displayed units are eligible at step $186$.
The selected path point reproduces the deployed mask exactly.
The diagonal and calibrated masks remove $188$ and $187$ units, respectively, because attention and FFN costs differ; both satisfy the realized-ratio check.
The calibrated mask retains attention unit $458$ and FFN unit $472$ in layer $19$ and removes FFN unit $593$ in layer $24$ instead.
Its selection-split calibration-risk advantage is $3.1$ paired standard errors.
Thus the trace identifies an actual conditional ordering change, rather than inferring one from marginal saliency or from the final mask alone.

\paragraph{Causal consequence of bridge structure.}
\textbf{Design.}
We test the bridge population on the five models in the full LLM comparison.
The forced population uses the full unit inventory, distinct from the guarded selector's eligible-candidate analysis.
On each, we force-prune the bridge units (low $d_u$, high $b_u$) and, as a control, an equal-size set drawn from the same low-saliency pool but with low connectivity, matched unit-for-unit on saliency $d_u$, unit type and per-unit cost.
The remainder of the budget is one cost-normalized greedy fill, computed once and shared by both arms.
The two masks that result hold the same number of units, the same attention/FFN split, the same removed cost and the same realized ratio to the last bit, and differ in which low-saliency population was forced out.
On Llama-3.1-8B-Instruct that is $169$ units each, $1.399\times10^{9}$ parameters, realized ratio $0.2005$, $82$ units shared and $87$ forced.

Removing the high-connectivity units is far more destructive on capability (\cref{tab:bridge-models}): nineteen of the twenty capability cells fall, by up to $23.7$ points, despite the two forced populations being matched on individual saliency, unit type, and cost.
The four groups show that the effect is distributed across behavior rather than carried by one benchmark family.

\begin{table}[htbp]
\centering\footnotesize
\caption{\textbf{Matched bridge intervention by task} across five language models ($\rho{=}20\%$).
Bridge/control sets match saliency, unit type, and cost and share one greedy fill; positive $\Delta$ means bridge removal is worse.
Avg$_{12}$ weights all tasks equally.}
\label{tab:bridge-models}
\adjustbox{max width=\textwidth}{
\scriptsize\setlength{\tabcolsep}{1.9pt}\renewcommand{\arraystretch}{.94}
\begin{tabular}{@{}llrrrrrrrrrrrr!{\color{black!35}\vrule width .5pt}r@{}}
\toprule
\textbf{Model} & \textbf{Forced set} & \textbf{PIQA} & \textbf{HS} & \textbf{WG} & \textbf{CSQA} & \textbf{ARC-e} & \textbf{ARC-c} & \textbf{OBQA} & \textbf{BoolQ} & \textbf{R-m} & \textbf{R-h} & \textbf{QuAIL} & \textbf{MMLU} & \textbf{Avg$_{12}$} \\
\midrule
Llama-3.2-3B & Control & 72.5 & 60.8 & 54.5 & 46.3 & 57.3 & 36.5 & 35.0 & 61.1 & 43.2 & 36.8 & 41.8 & 35.2 & 48.4 \\
\rowours & \textbf{Bridge} & 67.8 & 41.9 & 51.2 & 34.2 & 51.1 & 28.2 & 31.4 & 57.0 & 38.8 & 30.4 & 39.9 & 30.7 & 41.9 \\
 & $\Delta$ & $-4.7$ & $-18.9$ & $-3.2$ & $-12.0$ & $-6.3$ & $-8.3$ & $-3.6$ & $-4.1$ & $-4.4$ & $-6.4$ & $-1.9$ & $-4.5$ & \textbf{$-6.5$} \\
\addlinespace[2pt]
Falcon3-7B & Control & 75.1 & 63.8 & 57.5 & 44.6 & 69.0 & 43.1 & 34.2 & 69.9 & 51.6 & 39.3 & 48.7 & 38.4 & 52.9 \\
\rowours & \textbf{Bridge} & 71.1 & 60.8 & 55.4 & 35.9 & 62.8 & 38.1 & 29.2 & 70.7 & 53.1 & 38.7 & 50.0 & 37.1 & 50.2 \\
 & $\Delta$ & $-4.1$ & $-3.0$ & $-2.1$ & $-8.7$ & $-6.1$ & $-4.9$ & $-5.0$ & $+0.9$ & $+1.5$ & $-0.6$ & $+1.3$ & $-1.3$ & \textbf{$-2.7$} \\
\addlinespace[2pt]
Llama-3.1-8B-It & Control & 77.1 & 70.9 & 62.7 & 55.3 & 70.3 & 45.7 & 43.0 & 52.3 & 57.5 & 46.3 & 52.7 & 40.5 & 56.2 \\
\rowours & \textbf{Bridge} & 73.4 & 53.3 & 52.2 & 35.6 & 52.6 & 31.1 & 33.2 & 64.6 & 43.5 & 32.5 & 41.9 & 32.0 & 45.5 \\
 & $\Delta$ & $-3.7$ & $-17.6$ & $-10.5$ & $-19.7$ & $-17.7$ & $-14.7$ & $-9.8$ & $+12.3$ & $-14.0$ & $-13.7$ & $-10.8$ & $-8.5$ & \textbf{$-10.7$} \\
\addlinespace[2pt]
Mistral-Nemo-12B & Control & 77.9 & 72.5 & 60.1 & 47.7 & 69.3 & 47.6 & 38.4 & 72.1 & 58.8 & 45.4 & 52.1 & 40.1 & 56.8 \\
\rowours & \textbf{Bridge} & 73.6 & 51.0 & 52.6 & 35.2 & 52.4 & 32.2 & 33.4 & 50.7 & 40.9 & 32.5 & 41.8 & 31.7 & 44.0 \\
 & $\Delta$ & $-4.2$ & $-21.4$ & $-7.5$ & $-12.4$ & $-16.9$ & $-15.4$ & $-5.0$ & $-21.4$ & $-17.9$ & $-12.9$ & $-10.3$ & $-8.4$ & \textbf{$-12.8$} \\
\addlinespace[2pt]
Mistral-Small-24B & Control & 72.9 & 73.1 & 62.3 & 47.4 & 71.1 & 46.1 & 37.8 & 77.6 & 56.5 & 43.4 & 51.9 & 40.7 & 56.7 \\
\rowours & \textbf{Bridge} & 56.4 & 29.7 & 50.4 & 27.4 & 39.6 & 22.0 & 26.2 & 38.0 & 33.8 & 28.1 & 34.5 & 25.9 & 34.3 \\
 & $\Delta$ & $-16.5$ & $-43.5$ & $-11.9$ & $-20.1$ & $-31.5$ & $-24.1$ & $-11.6$ & $-39.6$ & $-22.6$ & $-15.2$ & $-17.4$ & $-14.8$ & \textbf{$-22.4$} \\
\bottomrule
\end{tabular}
}
\end{table}

The five-model intervention deliberately uses the full bridge population, but bridge prevalence varies with depth (\cref{tab:depth-grading}).
We therefore repeat the comparison with each forced bridge matched to its nearest-saliency low-coupling control in the same layer, unit type, and parameter-cost stratum.
Only strata containing both populations contribute a pair; the remainder of the budget is again one shared fill.
Across both models, bridge removal raises independent test KL and WikiText perplexity and lowers Avg$_{12}$ (\cref{tab:depth-matched-bridge}).
The effect is modest but consistent on 3B and substantially larger on 8B, showing that the bridge signal is not reducible to a coarse depth preference.

\begin{table}[htbp]
\centering\footnotesize
\caption{\textbf{Depth-matched bridge intervention} at $\rho{=}20\%$.
The matching above additionally fixes layer; both arms share one greedy fill.
Positive $\Delta$ uniformly indicates greater damage from bridge removal; Test-KL uncertainty is the paired sequence-level standard error over $128$ independent sequences.}
\label{tab:depth-matched-bridge}
\setlength{\tabcolsep}{5.0pt}\renewcommand{\arraystretch}{1.03}
\begin{tabular}{@{}lrrrrr@{}}\toprule
\textbf{Model} & \textbf{Pairs} & \textbf{$|\Delta d|\,(\times10^{-4})$} & \textbf{$\Delta$ Test KL $\uparrow$} & \textbf{$\Delta$ Wiki $\uparrow$} & \textbf{$\Delta$ Avg$_{12}$ loss $\uparrow$} \\
\midrule
Llama-3.2-3B & 35/77 & 4.81 & \textbf{0.023}\se{0.006} & \textbf{1.05} & \textbf{0.44} \\
Llama-3.1-8B-It & 42/81 & 2.69 & \textbf{0.090}\se{0.005} & \textbf{0.92} & \textbf{2.04} \\
\bottomrule\end{tabular}

\end{table}

\paragraph{Layer structure and allocation.}
\begin{table}[htbp]
\centering\footnotesize
\caption{\textbf{Depth dependence of individual saliency and aggregate coupling.}
Bridge is the percentage of units satisfying the fixed low-saliency/high-coupling definition; the remaining columns report Spearman rank correlation between layer index $\ell_u$ and individual saliency $\Hmat_{uu}$ or total coupling $b_u=\sum_{v\neq u}|\Hmat_{uv}|$.
Each row uses the same $\Hmat$ build evaluated in \cref{tab:stability}.
Coupling is negative on all ten; the saliency column changes sign across models.}
\label{tab:depth-grading}
\footnotesize\setlength{\tabcolsep}{4.5pt}\renewcommand{\arraystretch}{1.02}
\begin{tabular}{@{}l c rrr@{}}
\toprule
\textbf{Model} & \textbf{$L$} & \textbf{Bridge (\%)} & \textbf{$\rho_s(\ell_u, \Hmat_{uu})$} & \textbf{$\rho_s(\ell_u, b_u)$}\\
\midrule
Llama-3.2-3B & 28 & 12.8 & $-0.082$ & $\mathbf{-0.874}$\\
Falcon3-7B & 28 & 5.4 & $+0.143$ & $\mathbf{-0.543}$\\
Mistral-7B & 32 & 14.5 & $+0.107$ & $\mathbf{-0.880}$\\
Qwen2.5-7B & 28 & 5.0 & $-0.115$ & $\mathbf{-0.878}$\\
Llama-3.1-8B-It & 32 & 11.3 & $-0.229$ & $\mathbf{-0.867}$\\
Yi-1.5-9B & 48 & 10.3 & $+0.331$ & $\mathbf{-0.806}$\\
Falcon3-10B & 40 & 8.0 & $+0.031$ & $\mathbf{-0.745}$\\
Mistral-Nemo-12B & 40 & 12.0 & $-0.225$ & $\mathbf{-0.910}$\\
Llama-2-13B & 40 & 3.5 & $-0.384$ & $\mathbf{-0.884}$\\
Qwen2.5-14B & 48 & 13.7 & $+0.403$ & $\mathbf{-0.821}$\\
\midrule
\multicolumn{2}{@{}l}{\textbf{Range / median $|\rho_s|$}} & 3.5--14.5 & $0.184$ & $\mathbf{0.871}$\\
\bottomrule
\end{tabular}

\end{table}

\Cref{tab:depth-grading} verifies the bridge prevalence quoted in the main text and shows that aggregate coupling is systematically strongest in earlier layers across all ten models, whereas diagonal saliency has no consistent depth direction.
The edge matrix therefore contributes a relational depth signal that cannot be recovered by applying one monotone depth correction to individual scores.

\begin{figure}[H]
\centering
\includegraphics[width=.84\textwidth]{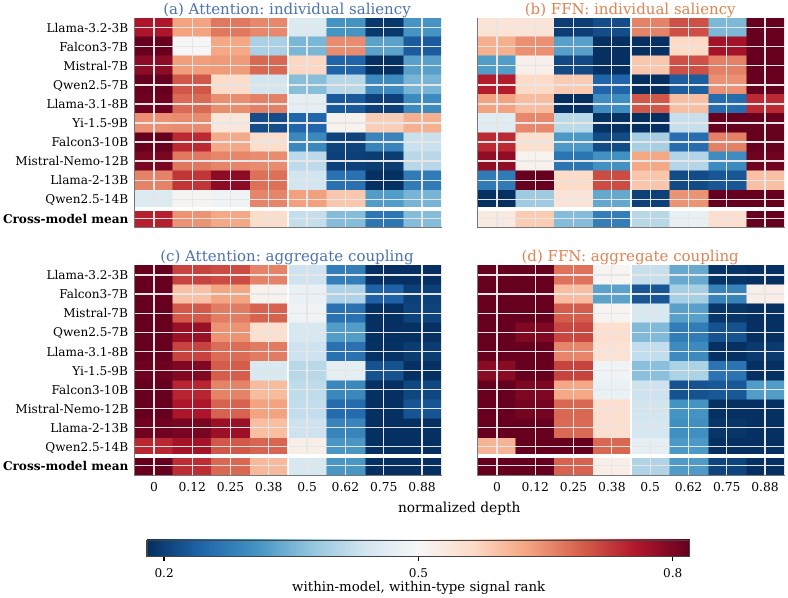}
\caption{\textbf{Where node and edge signals concentrate across depth.}
Each cell averages a within-model, within-unit-type percentile rank in one normalized-depth bin, so color compares spatial organization without conflating raw scales across architectures.
Individual saliency is architecture- and module-dependent in \textbf{(a,b)}; aggregate coupling instead concentrates toward earlier depth for both attention and FFN in \textbf{(c,d)}.
The separated bottom row is the cross-model mean; the matrices visualize the rank correlations in \cref{tab:depth-grading}.}
\label{fig:depth-signal-matrix}
\end{figure}

The contrast in \cref{fig:depth-signal-matrix} is not a generic instruction to protect all early units: the solver uses signed conditional costs, and individual selections remain heterogeneous within each depth band.
It instead identifies a stable source of information unavailable to a node-only selector---early units tend to participate in more consequential co-removal relations even when their isolated saliency is unexceptional.

\Needspace{11\baselineskip}
\textbf{Allocation.}
\Cref{fig:allocation} gives the per-layer, per-type profile on four models.
On Llama-3.1-8B-Instruct at $\rho{=}20\%$ the solution removes $46$ attention units and $105$ FFN groups over $26$ of $32$ layers, at most nine in any one layer, and splits evenly over the second half of the depth alone, $46$ attention against $46$ FFN.
The shared budget thus learns a heterogeneous allocation without collapsing onto one layer or module family; the depth pattern emerges from conditional costs rather than from a hand-set per-type schedule.

\begin{figure}[H]
\centering
\includegraphics[width=0.96\textwidth]{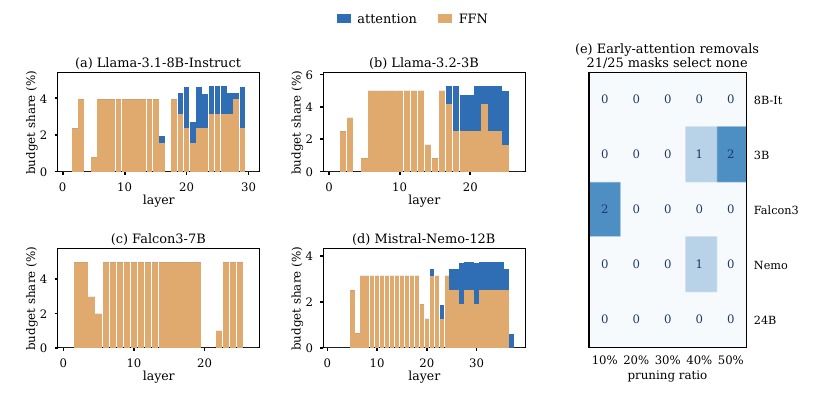}
\caption{\textbf{Per-layer allocation of the deployed masks.}
\textbf{(a)--(d)} Each layer's attention/FFN contribution to the shared removed-parameter budget at $\rho{=}20\%$; this cost-weighted view remains comparable when unit sizes differ across architectures.
\textbf{(e)} Attention units removed from the first half of depth over five models and five ratios: $21/25$ deployed solutions remove none, and the largest exception removes two.
The guard constrains per-layer totals but does not impose this allocation.}
\label{fig:allocation}
\end{figure}

\subsection{What the Edge Term Contributes}
\label{app:ablation-extra}
\begin{table}[htbp]
\centering\footnotesize
\caption{\textbf{Effective contribution of co-pruning edges across six language models.}
$\Delta$ is excess KL relative to the calibrated mask on the \emph{selection split}; $\sigma$ is the paired mean gap divided by its standard error on that same split.
Neither is an independent significance test.
The fixed-allocation block holds the diagonal mask's removed cost of each type fixed.}
\label{tab:ablation}
\adjustbox{max width=\textwidth}{
\footnotesize\setlength{\tabcolsep}{3pt}\renewcommand{\arraystretch}{1.02}
\begin{tabular}{@{}l c cc cc cc@{}}
\toprule
& & \multicolumn{2}{c}{\textbf{No edges}} & \multicolumn{2}{c}{\textbf{Undamped edges}} & \multicolumn{2}{c}{\textbf{Fixed split}}\\
\cmidrule(lr){3-4}\cmidrule(lr){5-6}\cmidrule(l){7-8}
\textbf{Model} & \textbf{$|\mathcal{F}|$} & \textbf{$\Delta$\%} & \textbf{$\sigma$} & \textbf{$\Delta$\%} & \textbf{$\sigma$} & \textbf{$\Delta$\%} & \textbf{$\sigma$}\\
\midrule
Llama-3.2-3B & 288 & +9.6 & 3.3 & +12.4 & 3.8 & +11.5 & 4.2\\
Falcon3-7B & 181 & +6.6 & 8.2 & +29.7 & 23.1 & +7.6 & 11.4\\
Mistral-7B & 318 & +5.8 & 5.9 & +32.0 & 11.0 & +5.4 & 4.6\\
Llama-3.1-8B-It & 303 & +6.7 & 4.6 & +13.0 & 6.0 & +8.5 & 5.4\\
Yi-1.5-9B & 400 & +3.8 & 2.1 & +746.9 & 39.1 & +4.4 & 3.1\\
Falcon3-10B & 402 & +2.5 & 3.9 & +30.1 & 8.3 & +3.6 & 3.5\\
\bottomrule
\end{tabular}
}
\end{table}

\begin{table}[htbp]
\centering\footnotesize
\caption{\textbf{Independent evaluation behind \cref{tab:ablation-main}.}
Test KL uses $128$ sequences disjoint from graph estimation and $\lambda$ selection; its small type is the sequence-level standard error.
The fixed-split arm preserves the diagonal selector's attention/FFN allocation, so its improvement isolates within-type reordering by edges.}
\label{tab:ablation-full}
\adjustbox{max width=.9\textwidth}{
\setlength{\tabcolsep}{4.2pt}\renewcommand{\arraystretch}{1.03}
\begin{tabular}{@{}ll l rrr@{}}\toprule
\textbf{Model} & \textbf{$\rho$} & \textbf{Selection rule} & \textbf{Test KL $\downarrow$} & \textbf{Wiki PPL $\downarrow$} & \textbf{Avg$_{12}$ $\uparrow$} \\
\midrule
\multirow{6}{*}{Llama-3.2-3B} & \multirow{3}{*}{20\%} & Diagonal & \cellcolor{CoCurveShade!7}0.791\se{0.020} & \cellcolor{CoCurveShade!7}17.1 & \cellcolor{CoCurveShade!7}45.8 \\
 &  & Edges, fixed split & \cellcolor{CoCurveShade!32}\textbf{0.734}\se{0.020} & \cellcolor{CoCurveShade!20}\underline{16.9} & \cellcolor{CoCurveShade!20}\underline{48.5} \\
 &  & \textbf{CoCurve} & \cellcolor{CoCurveShade!20}\underline{0.754}\se{0.020} & \cellcolor{CoCurveShade!32}\textbf{16.7} & \cellcolor{CoCurveShade!32}\textbf{48.9} \\
\addlinespace[2pt]
 & \multirow{3}{*}{30\%} & Diagonal & \cellcolor{CoCurveShade!7}1.709\se{0.032} & \cellcolor{CoCurveShade!7}49.0 & \cellcolor{CoCurveShade!7}39.8 \\
 &  & Edges, fixed split & \cellcolor{CoCurveShade!32}\textbf{1.382}\se{0.025} & \cellcolor{CoCurveShade!20}\underline{36.6} & \cellcolor{CoCurveShade!20}\underline{41.0} \\
 &  & \textbf{CoCurve} & \cellcolor{CoCurveShade!20}\underline{1.577}\se{0.029} & \cellcolor{CoCurveShade!32}\textbf{34.8} & \cellcolor{CoCurveShade!32}\textbf{42.1} \\
\midrule
\multirow{6}{*}{Llama-3.1-8B-It} & \multirow{3}{*}{20\%} & Diagonal & \cellcolor{CoCurveShade!7}0.611\se{0.014} & \cellcolor{CoCurveShade!7}13.1 & \cellcolor{CoCurveShade!20}\underline{55.1} \\
 &  & Edges, fixed split & \cellcolor{CoCurveShade!32}\textbf{0.571}\se{0.014} & \cellcolor{CoCurveShade!32}\textbf{12.5} & \cellcolor{CoCurveShade!7}54.9 \\
 &  & \textbf{CoCurve} & \cellcolor{CoCurveShade!20}\underline{0.584}\se{0.013} & \cellcolor{CoCurveShade!20}\underline{12.9} & \cellcolor{CoCurveShade!32}\textbf{56.0} \\
\addlinespace[2pt]
 & \multirow{3}{*}{30\%} & Diagonal & \cellcolor{CoCurveShade!7}1.169\se{0.017} & \cellcolor{CoCurveShade!7}26.0 & \cellcolor{CoCurveShade!7}48.7 \\
 &  & Edges, fixed split & \cellcolor{CoCurveShade!32}\textbf{1.016}\se{0.017} & \cellcolor{CoCurveShade!20}\underline{22.0} & \cellcolor{CoCurveShade!32}\textbf{49.7} \\
 &  & \textbf{CoCurve} & \cellcolor{CoCurveShade!20}\underline{1.085}\se{0.018} & \cellcolor{CoCurveShade!32}\textbf{21.6} & \cellcolor{CoCurveShade!20}\underline{49.0} \\
\bottomrule\end{tabular}
}
\end{table}

Across the six models in \cref{tab:ablation}, removing edges raises held-out risk by $2.5$--$9.6\%$ and changes the selected mask by Jaccard distance $0.13$--$0.53$ (median $0.22$), while fully undamping them is worse on every model.
On the independent split, adding edges while holding the diagonal selector's allocation fixed improves test KL and WikiText-2 perplexity in all four evaluated blocks (\cref{tab:ablation-full}).
The fixed-split and free-allocation arms need not order identically on next-token KL and a heterogeneous downstream mixture; neither uses downstream labels for selection.
Together, the controls isolate the robust source of the gain: relative to the diagonal selector, edges improve unit ordering beyond merely changing the attention/FFN split, while calibration prevents their local curvature signal from dominating the finite intervention.

The following controls isolate two structural design axes under the same budget and guard; estimator-budget sensitivity is treated separately in \appref{app:sensitivity}.

\textbf{(i) Module eligibility.}
\begin{table}[htbp]
\centering\footnotesize
\caption{\textbf{Restricting eligibility to one module type} at a $20\%$ target, with the deployed guard and $\hat\lambda$.
\emph{attn/FFN} gives the removed-unit split; Falcon3-7B attention reaches its guarded capacity at $9.1\%$, while its joint solution chooses only FFN and therefore equals the FFN-only arm.}
\label{tab:singlemodule}
\setlength{\tabcolsep}{4.5pt}\renewcommand{\arraystretch}{1.03}
\begin{tabular}{@{}l cc cc @{\hskip .8em} cc cc@{}}
\toprule
& \multicolumn{4}{c}{\textbf{Falcon3-7B}} & \multicolumn{4}{c}{\textbf{Llama-3.2-3B}}\\
\cmidrule(lr){2-5}\cmidrule(lr){6-9}
\textbf{Eligible units} & \textbf{$\rho$} & \textbf{attn/FFN} & \textbf{Wiki}\dn & \textbf{Cmn$_4$ (\%)}\up & \textbf{$\rho$} & \textbf{attn/FFN} & \textbf{Wiki}\dn & \textbf{Cmn$_4$ (\%)}\up\\
\midrule
Attention only & $9.1\%$ & 96/0 & 851 & 36.9 & $20.1\%$ & 180/0 & 465 & 37.7\\
FFN only & $20.2\%$ & 0/101 & \best{7.07} & \best{59.4} & $20.1\%$ & 0/120 & \best{16.42} & 55.0\\
Both (\textbf{CoCurve}) & $20.2\%$ & 0/101 & \best{7.07} & \best{59.4} & $20.1\%$ & 38/95 & 16.66 & \best{56.6}\\
\bottomrule
\end{tabular}

\end{table}
\Cref{tab:singlemodule} rules out a simpler interpretation in which the result follows merely from choosing the right module family.
On Llama-3.2-3B, where all arms reach the target, joint eligibility improves the capability trade-off over FFN-only pruning and avoids the severe perplexity loss of attention-only pruning.
On Falcon3-7B, the selector instead spends the entire feasible budget on FFN, adapting the module split rather than enforcing one in advance.

\textbf{(ii) Interaction granularity.}
The edges could act at a coarser level than unit pairs: keep the cross-component block but replace it with its scalar mean, and the interaction decides how much of each component type to remove and never which unit.
That is the level SAViT~\citep{zheng2022savit} interacts at, and \cref{tab:granularity} runs it.
Because averaging the block rescales the edge mass that enters the marginal, that arm is calibrated at its own $\hat\lambda$, isolating interaction granularity from calibration.
It still loses six of the seven columns, including all four capability groups; only PTB is better.
What the edges contribute is therefore which \emph{pairs} are costly together, not a per-component budget.

\begin{table}[htbp]
\centering\footnotesize
\caption{\textbf{Interaction granularity} (Llama-3.1-8B-Instruct, $\rho{=}20\%$).
Both rows use the same solver, guard and budget and differ only in whether the off-diagonal acts per unit pair or per component pair; each is independently calibrated and realizes its target ratio to within $0.2$ points.}
\label{tab:granularity}
\footnotesize\setlength{\tabcolsep}{4pt}\renewcommand{\arraystretch}{1.02}
\begin{tabular}{@{}l rrr cccc@{}}
\toprule
\textbf{Interaction granularity} & \textbf{WikiText}\dn & \textbf{PTB}\dn & \textbf{C4}\dn & \textbf{Cmn.}\up & \textbf{Sci.}\up & \textbf{Read.}\up & \textbf{Know.}\up \\
\midrule
\textbf{CoCurve} (unit pair) & \textbf{12.88} & 29.63 & \textbf{19.58} & \textbf{63.7} & \textbf{48.0} & \textbf{58.6} & \textbf{38.5} \\
component pair & 12.90 & \textbf{28.81} & 19.84 & 62.8 & 47.2 & 57.9 & 37.8 \\
\bottomrule
\end{tabular}

\end{table}

\subsection{Does the Surrogate Rank Masks?}

\label{app:surrogate}
\paragraph{Same-budget discrimination.}
Correlation across ratios can conflate ranking quality with the changing budget, so we evaluate the harder case of one fixed budget.
We fix $\rho{=}20\%$ on Llama-3.1-8B-Instruct and construct a pool of $26$ masks at the same target budget: the \method-greedy set, the diagonal-greedy set, $14$ random sets and $10$ local-swap perturbations of the \method set, each filled to within $0.75\%$ of the same removed cost.
We compare predicted $\tfrac12 s^\top\Hmat s$ against measured $\KL(p_0\|p_S)$ across the pool (\cref{fig:sameratio}).
The surrogate reaches full-pool Spearman $0.87$ and separates the groups completely: the twelve structured selections average $\KL\,0.62$ over $0.57$--$0.71$, whereas every random set is worse, averaging $1.26$.
We then remove this between-group signal and test the twelve structured masks alone, all within $0.15$ nats of one another.
Spearman correlation remains $0.55$, and reaches $0.58$ among the ten local swaps; this fine-grained discrimination most closely matches the comparisons made by the greedy solver.

\begin{figure}[H]
\centering
\includegraphics[width=0.88\textwidth]{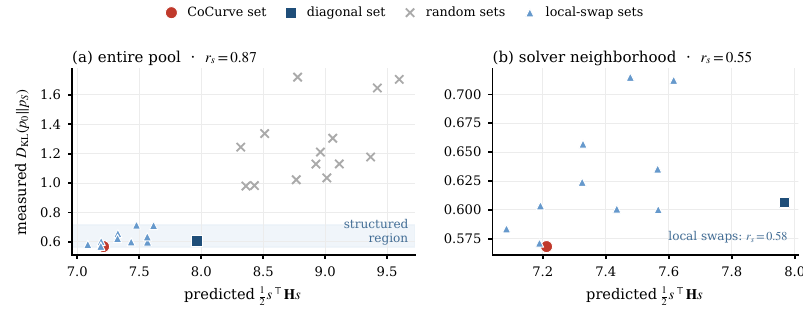}
\caption{\textbf{The surrogate discriminates at one fixed budget.}
Predicted $\tfrac12 s^\top\Hmat s$ against measured $\KL(p_0\|p_S)$ for masks matched to within $0.75\%$ of the same removed cost on Llama-3.1-8B-Instruct at $\rho{=}20\%$.
\textbf{(a)} The full pool reaches Spearman $0.87$ and places every structured set below every random set.
\textbf{(b)} The structured-only view isolates fine-grained ordering and retains correlation $0.55$ ($0.58$ among local swaps).}
\label{fig:sameratio}
\end{figure}

\subsection{Estimator and Calibration Robustness}

\label{app:robustness}

Three questions about stability, each answered by a different perturbation: resampling the held-out set, changing the selection protocol, and rebuilding the curvature from independent calibration draws.

\paragraph{Stability of the estimated matrix.}
\label{app:diagnostics}

We rebuild $\Hmat$ independently on random halves of the calibration set, resampling over \emph{sequences} (positions within a sequence are strongly dependent, so treating the stored positions as i.i.d.\ would understate the error).
Across ten models spanning six families and $3$B--$14$B the off-diagonal entries reproduce across halves at Pearson $0.885$--$0.996$ and the diagonal at $0.994$--$1.000$ (\cref{tab:stability}).
The agreement supports reproducible interaction structure across calibration halves.
The finite-ablation substitution and joint-additivity approximation behind the estimator are measured in \appref{app:th-est}.

\begin{table}[htbp]
\centering\footnotesize
\caption{\textbf{Split-half stability of $\Hmat$}, resampled over calibration sequences, mean of four random halvings.
Edges remain strongly reproducible across all ten models, while the diagonal terms are near-identical.}
\label{tab:stability}
\setlength{\tabcolsep}{9pt}\renewcommand{\arraystretch}{1.05}
\begin{tabular}{@{}l cc@{}}
\toprule
\textbf{Model} & \textbf{Individual terms (Pearson)} & \textbf{Co-pruning edges (Pearson)}\\
\midrule
Llama-3.2-3B & 0.996 & 0.885\\
Falcon3-7B & 0.998 & 0.979\\
Mistral-7B & 1.000 & 0.962\\
Qwen2.5-7B & 1.000 & 0.991\\
Llama-3.1-8B-It & 0.994 & 0.905\\
Yi-1.5-9B & 1.000 & 0.966\\
Falcon3-10B & 0.997 & 0.965\\
Mistral-Nemo-12B & 0.999 & 0.982\\
Llama-2-13B & 1.000 & \textbf{0.996}\\
Qwen2.5-14B & 0.999 & 0.950\\
\bottomrule
\end{tabular}

\end{table}

\paragraph{Sensitivity to the surrogate budget.}
\label{app:sensitivity}
The edge estimator has two finite-budget knobs: the calibration size (number of sampled positions) and the top-$r$ logit support of the Fisher-Gram features.
Neither choice drives the result.
We rebuild $\Hmat$ on the cheapest model (Llama-3.2-3B) at top-$r\in\{128,256,512\}$ and calibration size $\in\{64,128\}$, then run each build through the same selector---the guard of \appref{app:protocol} and $\lambda$ re-chosen on held-out KL for that build---and evaluate (\cref{tab:sensitivity}).
\Needspace{5\baselineskip}
Across the builds measured, WikiText-2 perplexity spans $3.1\%$ and MMLU $0.4$ points.
The two larger top-$r$ supports select the same $133$-unit mask despite using separate feature builds, while the remaining settings preserve downstream quality with only small mask changes.
This stability supports $(128,256)$ as a reliable default surrogate budget.
\Cref{fig:calibbudget} asks the same question of the selection directly rather than through a benchmark: how many calibration sequences are needed before the chosen units stop changing.
\begin{table}[htbp]
\centering
\caption{\textbf{Surrogate-budget sensitivity} (Llama-3.2-3B at $\rho{=}20\%$, under the deployed guard).
Upper block varies top-$r$ at fixed calibration $128$; lower block varies calibration at fixed top-$256$.
The default is $(128,256)$.
Each build re-estimates $\Hmat$, recalibrates $\hat\lambda$, and reselects its mask.}
\label{tab:sensitivity}\footnotesize
\setlength{\tabcolsep}{7pt}
\begin{tabular}{ll cc}
\toprule
\textbf{Axis} & \textbf{Setting} & \textbf{Wiki}\dn & \textbf{MMLU (\%)}\up\\
\midrule
Top-$r$ & $128$ & 17.17 & 34.0\\
Top-$r$ & $256$ (default) & 16.66 & 33.9\\
Top-$r$ & $512$ & 16.66 & 33.9\\
\midrule
Calibration sequences & $64$ & 16.82 & 33.6\\
Calibration sequences & $128$ (default) & 16.66 & 33.9\\
\bottomrule
\end{tabular}

\end{table}

\paragraph{Selection-split and protocol robustness.}
\label{app:lambda-robust}

The primary LLM masks use one calibration draw and $16$ held-out selection sequences.
Two questions follow, and they are different: whether the \emph{selection step} would survive a different draw of the held-out set, and whether the \emph{curvature it selects over} would survive a different draw of the calibration text.
We answer both.

\Needspace{11\baselineskip}
\paragraph{Held-out resampling concentrates on a narrow candidate set.}
The path files record per-sequence risk at every candidate, so the bootstrap reuses these risks without additional model forwards.
Over $26$ paths spanning ten models and five ratios, with $2000$ resamples each, bootstrap selections occupy a median of only $3.3\%$ of the enumerated family, and $2.5\%$ at the three main ratios.
A typical family contains $318$ candidate masks and the whole bootstrap moves among ten of them.

The criterion identifies a mask near the risk minimum rather than a uniquely meaningful numerical value of $\lambda$: evaluated back on the full held-out set, the mask chosen by a resample incurs a median of $0.00\%$ additional risk relative to the deployed mask, and every path puts at least $98.5\%$ of resamples within one standard error of the minimum.
The solution is piecewise constant and the path is flat near its optimum, so the members the bootstrap moves among are the ones that cannot be told apart.

\paragraph{Downstream stability around the selected risk minimum.}
The bootstrap judges $\hat\lambda$ by held-out risk, the quantity it is selected on.
The independent check is to take the selected point and its neighbors and evaluate all of them on the full suite.
On Llama-3.2-3B at $\rho{=}20\%$ we do this for five points: $\hat\lambda$ itself, the two members at the edges of its $1$-SE band, and the two nearest solutions just outside it.
The selected point is within $0.3$ points of the best of the five on three of the four capability groups and within $0.91$ on the fourth, and it is $2.7$ to $2.9$ points above the worst on the twelve-task mean.
Thus the selected mask preserves favorable downstream quality within the risk region that the held-out sample cannot distinguish.

\paragraph{Changing the protocol.}
A bootstrap only perturbs the sample.
We therefore varied the protocol itself on Llama-3.1-8B-Instruct at $\rho{=}20\%$: replacing KL by NLL, increasing the held-out set from $16$ to $48$ sequences, and using two further draws all select the bit-identical deployed solution.
Together with the split-half and budget tests above, this establishes stability at the graph, path, and deployed-decision levels.

\begin{figure}[H]
\centering
\includegraphics[width=0.68\textwidth]{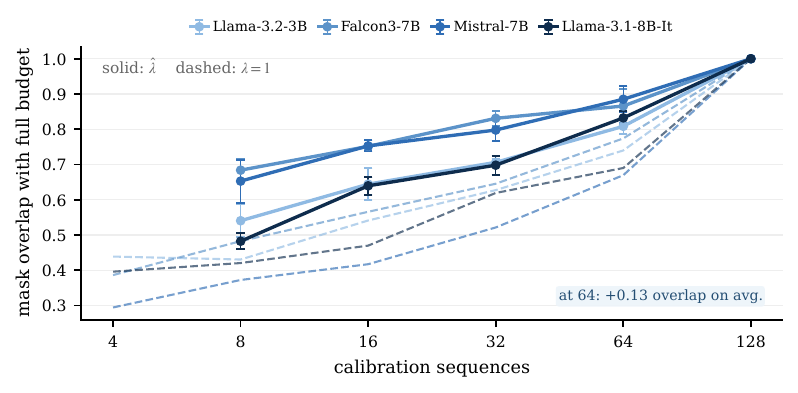}
\caption{\textbf{How much calibration data the selection needs.}
Overlap between the units chosen from $n$ sequences and those chosen from the full budget of $128$, mean over three draws.
Measuring the selection rather than downstream accuracy separates estimator convergence from benchmark resolution.
At $64$ sequences, calibrated $\hat\lambda$ raises overlap by $0.13$ on average relative to $\lambda{=}1$; damping the off-diagonal, the part of the matrix that needs the most data, also makes the solution easier to estimate.}
\label{fig:calibbudget}
\end{figure}

\FloatBarrier

\clearpage
\FloatBarrier
\section{Discussion and Reproducibility}
\label{app:discussion}
\subsection{Extended Related Work}
\label{app:connections}
\label{app:related-ext}
\textbf{Signals for structured post-training pruning.}
Weight magnitude provides a data-free reference, while Wanda and FLAP incorporate activation statistics into weight or structure scores~\citep{sun2024wanda,an2024flap}.
LLM-Pruner couples related structures with first-order Taylor importance~\citep{ma2023llmpruner}; SlimGPT uses layer-wise OBS reconstruction~\citep{ling2024slimgpt}; and SlimLLM combines head selection, FFN feature analysis, and regression repair~\citep{tang2025slimllm}.
AMP and 2SSP focus explicitly on allocating sparsity across modules or stages~\citep{sultan2025amp,sandri2025twossp}.
SubspacePath Pruner couples scenario representation subspaces to reusable head-wise pathways through lightweight probes, targeting inference-time specialization rather than task-agnostic joint attention--FFN co-pruning~\citep{gong2026subspacepath}.
Adjacent compression families answer different contracts: SparseGPT and the original Wanda remove individual weights without changing dense tensor shapes; ShortGPT removes entire Transformer blocks; SliceGPT rotates the retained basis before width slicing; and LoRAPrune interleaves structural decisions with adaptation~\citep{frantar2023sparsegpt,sun2024wanda,men2024shortgpt,ashkboos2024slicegpt,zhang2024loraprune}.
They are important context but not matched pruning-only selectors over the shared attention-head/FFN-unit inventory used by our primary grid.
These methods motivate broad signal and allocation coverage in \appref{app:baselines}, but none makes a global predictive-risk edge between two prospective removals the quantity accumulated during selection.

\textbf{Curvature, interactions, and forward perturbations.}
OBD scores isolated removals with a curvature diagonal~\citep{lecun1989obd}, whereas OBS and WoodFisher repair retained weights using curvature~\citep{hassibi1993obs,singh2020woodfisher}.
CBS studies cancellation among jointly removed weights~\citep{yu2022cbs}; SOSP-I constructs pairwise interactions between CNN structures~\citep{nonnenmacher2022sosp}; and Kwon et al.\ retain within-layer Fisher interactions after head and filter counts have been fixed~\citep{kwon2022fastpruning}.
SAViT uses cross-component Taylor importance for vision Transformers~\citep{zheng2022savit}, and LLM-Surgeon combines Kronecker-factored curvature with weight repair~\citep{vanderouderaa2024llmsurgeon}.
Bonsai instead learns an additive quality predictor from forward perturbations~\citep{dery2024bonsai}.
CoCurve joins the useful premises behind these lines: finite forward ablations define output-space features; their Gram matrix exposes pairwise attention--FFN interactions globally; and a conditional pruner accumulates the signed incident edges of the current removal set.

\textbf{Recovery, numerical compression, and VLMs.}
LoRA can recover a fixed pruned structure~\citep{ma2023llmpruner,ling2024slimgpt}, or be interleaved with structural pruning~\citep{zhang2024loraprune}.
Separating recovery from selection lets the primary grid identify structure quality, while the common-recipe study measures recoverability.
Likewise, NF4/INT8 composition asks whether structural and numerical compression coexist, not whether low-bit weights substitute for selecting a smaller computation graph.
Efficient-VLM work predominantly reduces visual-token computation~\citep{chen2024fastv,zhang2025sparsevlm,cao2024madtp,zou2026codecsight}; parameter-structural approaches use model-specific search or recovery~\citep{shi2023upop,liang2025efficientllava}.
Our VLM experiment instead holds the predictive-risk formulation fixed and changes only the architecture-facing unit registry.

This taxonomy defines the matched comparison contract: structural selectors over one attention--FFN inventory are compared separately from repair, token pruning, and routed execution, whose nominal sparsity has a different operational meaning.

\subsection{Scope, Limitations, and Extensions}
\label{app:limitations}
The evidence establishes three operating conditions and points to direct extensions.

\begin{itemize}[leftmargin=1.6em,itemsep=2pt,topsep=2pt]
    \item \textbf{Upfront, reusable selection.}
    Building the interaction graph is costlier than an isolated-score pass, but it is performed once and reused across pruning ratios.
    The measured cost is therefore best amortized when several operating points are compiled or a checkpoint is deployed repeatedly (\appref{app:selection-cost}).

    \item \textbf{Pairwise local model.}
    The graph is a second-order, pairwise approximation rather than an explicit model of all higher-order removals.
    Held-out coefficient selection, fixed-budget surrogate tests, and the measured ratio sweeps validate this approximation over the reported operating range (\appref{app:surrogate}, \appref{app:robustness}).

    \item \textbf{Architecture-facing deployment backend.}
    We instantiate CoCurve and physical slicing for attention and FFN units in dense language and vision--language Transformers.
    Routed experts or hybrid blocks require an appropriate unit registry and slicing backend, although the teacher-KL graph and conditional objective remain unchanged.
    They also expose a different loaded-versus-active parameter contract, so we do not present active-count matching as a pruning baseline.
\end{itemize}

These conditions concern the current implementation and measured regime; the co-pruning objective extends directly to broader structural inventories and deployment backends.

\subsection{Broader Impact}

\label{app:impact}
Structured, training-free compression can lower the hardware and energy threshold for serving capable models.
Physically smaller dense checkpoints are especially relevant where accelerator memory, cost, connectivity, or data-governance constraints prevent reliance on large hosted systems, widening the feasibility of local, institutional, and privacy-preserving deployment.
Because \method uses unlabeled calibration text without task-specific fine-tuning, one reusable graph can compile several deployment points, allowing practitioners to choose a quality--resource trade-off for their setting rather than inherit a single model size.
These benefits do not make compressed models universally safe or equitable: they inherit the source model's limitations and may shift capabilities unevenly, so safety-, bias-, and domain-critical behavior should be reevaluated at the chosen operating point.
\subsection{Assets, Licenses, and Ethics}
\label{app:assets}
All models are public pretrained checkpoints used under their respective licenses: Llama-2, Llama-3.1, and Llama-3.2 (Llama Community License); Falcon3 (Falcon LLM License); Mistral-7B, Mistral-Nemo, Mistral-Small, Yi-1.5, Qwen2.5-7B/14B, Qwen2.5-VL, and Qwen3-VL (Apache-2.0); and InternVL3 (MIT release with its Qwen2.5 component under the Qwen License).
Evaluation datasets are used under their published terms.
Baseline implementations follow the authors' released code where available and their published descriptions otherwise; \appref{app:baselines} records every deviation.
No model is trained or fine-tuned in the pruning-only setting, and the recovery stage of \appref{app:recovery} uses the cleaned Alpaca corpus under its stated license.

\subsection{Reproducibility Checklist}

\label{app:repro}
\textbf{Method and configurations.}
The supplement specifies the structured-unit registry, runtime-mask and physical-slicing semantics, single-unit ablations, blockwise Gram estimator (\cref{alg:estimate}), conditional solver and guard (\cref{alg:solve}), and breakpoint-guided interaction path.
\Cref{tab:hparam} records the model-specific dimensions and selected coefficients; \appref{app:protocol} fixes budgets, data splits, prompts, metrics, and seeds.
Selection uses no labels, gradients, or recovery, so a mask is determined by the tuple (checkpoint, calibration draw, ratio, configuration).

\textbf{Artifacts and verification.}
The release will include the estimator and solver code, pruning solutions, experiment configurations, length-normalized evaluation harness, and scripts that regenerate every reported table from machine-readable reports.
Physical slicing is checked against runtime masking before timing; FP32 equivalence checks isolate structural correctness, while bf16 drift and realized parameter counts are recorded.
Optional language and VLM recovery follow the frozen recipes in \appref{app:recipes} and \appref{app:vlm-recovery}; quantized conditions are marked explicitly rather than treated as the default precision.
Each systems table states its accelerator, precision, batching, attention backend, and whether the machine was otherwise idle, separating portable parameter counts from hardware-dependent memory and throughput.

\FloatBarrier

\end{document}